\documentclass{article}

\usepackage[utf8]{inputenc}
\usepackage[T1]{fontenc}
\usepackage[english]{babel}
\usepackage[hyphens]{url}
\usepackage{authblk}
\usepackage{booktabs}
\usepackage[colorlinks=true, linkcolor=red, urlcolor=blue, citecolor=blue, anchorcolor=blue,backref=page]{hyperref}
\usepackage{amsfonts}
\usepackage{nicefrac}
\usepackage{microtype}
\usepackage{amsmath, amsthm, amssymb}
\usepackage{fancybox}
\usepackage{graphicx}
\usepackage{float}
\usepackage[ruled, vlined, linesnumbered]{algorithm2e}
\usepackage{adjustbox}
\usepackage{algpseudocode}
\usepackage{color}
\usepackage{tikz}
\usepackage{pifont}
\usepackage{array}
\usepackage{subcaption}
\usepackage[textsize=scriptsize]{todonotes}
\usepackage{enumitem}
\usepackage{wrapfig}
\usepackage{enumitem}
\usepackage{etoolbox}
\usepackage{diagbox}
\usepackage{comment}
\usepackage{makecell}
\usepackage{multirow}
\usepackage{bbm}
\usepackage{multicol}
\usepackage{mathpazo}
\usepackage{xcolor,soul,colortbl}
\usepackage[compress,numbers,sort]{natbib}
\usepackage{cleveref}
\crefname{figure}{Fig.}{Figs.}
\crefname{equation}{Eq.}{Eqs.}
\crefname{definition}{Defn.}{Defns.}
\crefname{corollary}{Corollary}{Corollaries}
\crefname{proposition}{Prop.}{Props.}
\crefname{theorem}{Thm.}{Thms.}
\crefname{remark}{Remark}{Remarks}
\crefname{principle}{Principle}{Principles}
\crefname{lemma}{Lemma}{Lemmas}
\crefname{claim}{Claim}{Claims}
\crefname{table}{Table}{Tabs.}
\crefname{section}{\S}{\S\S}
\crefname{subsection}{\S}{\S\S}
\crefname{subsubsection}{\S}{\S\S}
\crefname{assumption}{Assumption}{Assumptions}
\crefname{appendix}{Appendix}{Appendices}
\crefname{example}{Ex.}{Exs.}
\crefname{algorithm}{Alg.}{Alg.}

\usepackage{enumitem}
\setlist[itemize]{leftmargin=*,itemsep=0.3em, parsep=0em, topsep=0.5em}
\setlist[enumerate]{leftmargin=*,itemsep=0.1em, parsep=0em, topsep=0.1em}

\usepackage[framemethod=TikZ]{mdframed}
\mdfsetup{skipabove=5pt,
innertopmargin=-5pt}

\usepackage[toc,page,header]{appendix}
\usepackage{minitoc}

\definecolor{Gray}{gray}{0.9}

\newtheoremstyle{slplain}%
  {.4\baselineskip\@plus.1\baselineskip\@minus.1\baselineskip}%
  {.3\baselineskip\@plus.1\baselineskip\@minus.1\baselineskip}%
  {\itshape}%
  {}%
  {\bfseries}%
  {.}%
  { }%
  {}%
\theoremstyle{slplain} %

\newtheorem*{definition*}{Definition}
\newtheorem*{theorem*}{Theorem}
\newtheorem{theorem}{Theorem}[section]
\newtheorem{lemma}[theorem]{Lemma}
\newtheorem{proposition}[theorem]{Proposition}

\newtheorem{corollary}[theorem]{Corollary}
\newtheorem{definition}[theorem]{Definition}

\makeatletter
\newtheorem*{rep@theorem}{\rep@title}
\newcommand{\newreptheorem}[2]{%
\newenvironment{rep#1}[1]{%
 \def\rep@title{#2 \ref{##1}}%
 \begin{rep@theorem}}%
 {\end{rep@theorem}}}
\makeatother

\newreptheorem{theorem}{Theorem}
\newreptheorem{lemma}{Lemma}
\newreptheorem{claim}{Claim}
\newreptheorem{corollary}{Corollary}
\newreptheorem{proposition}{Proposition}

\theoremstyle{definition}

\newtheorem{example}[theorem]{Example}
\newtheorem{remark}[theorem]{Remark}

\theoremstyle{plain} %

\numberwithin{equation}{section}

\newtheoremstyle{etplain}%
  {.0\baselineskip\@plus.1\baselineskip\@minus.1\baselineskip}%
  {.0\baselineskip\@plus.1\baselineskip\@minus.1\baselineskip}%
  {\itshape}%
  {}%
  {\bfseries}%
  {.}%
  { }%
  {}%

\newcommand{\R}{\mathbb{R}}

\newcommand{\norm}[1]{\left|\left| #1 \right|\right|}
\newcommand{\Prob}{\mathbb{P}}

\renewcommand\bar\overline

\newcommand\Etr{\mathcal{E}_\text{tr}}
\newcommand\Eall{{\mathcal{E}_\text{all}}}

\newcommand{\rfe}{R^e(f)}
   
\newcommand{\cutoff}{\alpha}            %
\newcommand{\thres}{t}                  %
\newcommand{\bm}{\mathbf} 

\newcommand\PA{\mathrm{Pa}}
\newcommand\DE{\mathrm{De}}

\DeclareMathOperator*{\argmin}{arg\,min}
\DeclareMathOperator*{\esssup}{ess\,sup}
\DeclareMathOperator{\SQ}{SQ}

\DeclareMathOperator{\VREx}{VREx}

\DeclareMathOperator*{\subjto}{subject\,\, to \qquad}

\newcolumntype{C}[1]{>{\centering\let\newline\\\arraybackslash\hspace{0pt}}m{#1}}

\DeclareMathOperator{\E}{\mathbb{E}}

\newcommand{\calA}{\ensuremath{\mathcal{A}}}
\newcommand{\calB}{\ensuremath{\mathcal{B}}}
\newcommand{\calC}{\ensuremath{\mathcal{C}}}

\newcommand{\calF}{\ensuremath{\mathcal{F}}}
\newcommand{\calG}{\ensuremath{\mathcal{G}}}

\newcommand{\calL}{\ensuremath{\mathcal{L}}}
\newcommand{\calM}{\ensuremath{\mathcal{M}}}
\newcommand{\calN}{\ensuremath{\mathcal{N}}}

\newcommand{\calP}{\ensuremath{\mathcal{P}}}

\newcommand{\calR}{\ensuremath{\mathcal{R}}}
\newcommand{\calS}{\ensuremath{\mathcal{S}}}

\newcommand{\calU}{\ensuremath{\mathcal{U}}}

\newcommand{\calX}{\ensuremath{\mathcal{X}}}
\newcommand{\calY}{\ensuremath{\mathcal{Y}}}

\newcommand{\bL}{\ensuremath{\bm{L}}}

\newcommand{\bbE}{\ensuremath{\mathbb{E}}}

\newcommand{\bbL}{\ensuremath{\mathbb{L}}}

\newcommand{\bbP}{\ensuremath{\mathbb{P}}}
\newcommand{\bbQ}{\ensuremath{\mathbb{Q}}}
\newcommand{\bbR}{\ensuremath{\mathbb{R}}}

\newcommand{\bbT}{\ensuremath{\mathbb{T}}}
\newcommand{\bbU}{\ensuremath{\mathbb{U}}}
\newcommand{\bbV}{\ensuremath{\mathbb{V}}}

\def\nd/{\textsuperscript{nd}}
\def\rd/{\textsuperscript{rd}}
\def\th/{\textsuperscript{th}}

\makeatletter
\def\nnil{\nil}
\newcounter{prob}

\newcounter{dual}

\makeatother

\newenvironment{prob*}{%
	\csname equation*\endcsname%
	\aligned%
}{%
	\endaligned%
	\csname endequation*\endcsname%
}

\AtBeginDocument{%
  
}

\SetKwInput{kwInit}{Input}
\SetKwInput{kwOutput}{Output}
\SetKwFunction{proc}{proc}
\SetKwProg{myproc}{Procedure}{}{}

\SetCommentSty{mycommfont}
\newcommand{\cmark}{\ding{51}}
\newcommand{\xmark}{\ding{55}}

\newcommand{\changelinkcolor}[1]{\hypersetup{linkcolor=#1}}  

\usepackage[final]{neurips_2022}

\usepackage{todonotes}

\title{Probable Domain Generalization \\ via Quantile Risk Minimization}

\author{%
  \textbf{Cian Eastwood}\thanks{Equal contribution. Correspondence to \texttt{c.eastwood@ed.ac.uk} or \texttt{arobey1@seas.upenn.edu}.}\, $^{1,2}$
  \quad \textbf{Alexander Robey}$^{* 3}$
  \quad \textbf{Shashank Singh}$^{1}$ \vspace{2mm}\\
  \textbf{Julius von K\"ugelgen}$^{1,4}$
  \quad \textbf{Hamed Hassani}$^{\,3}$
  \quad \textbf{George J. Pappas}$^{\,3}$
  \quad \textbf{Bernhard Sch\"olkopf}$^{\, 1}$ \vspace{4mm}\\ 
 $^{1}$ Max Planck Institute for Intelligent Systems, T\"ubingen \\
 $^{2}$ University of Edinburgh \quad $^{3}$ University of Pennsylvania \quad $^{4}$ University of Cambridge
}

\begin{document}%
\doparttoc%
\faketableofcontents%
\part{}\vspace{-10mm}%
\changelinkcolor{black}{}%
\maketitle
\changelinkcolor{red}{}%

\begin{abstract}
\let\thefootnote\relax\footnotetext{Code available at: \href{https://github.com/cianeastwood/qrm}{https://github.com/cianeastwood/qrm}}
Domain generalization (DG) seeks predictors which perform well on unseen test distributions by leveraging data drawn from multiple related training distributions or domains. To achieve this, DG is commonly formulated as an average- or worst-case problem over the set of possible domains. However, predictors that perform well on average lack robustness while predictors that perform well in the worst case tend to be overly-conservative. To address this, we propose a new probabilistic framework for DG where the goal is to learn predictors that perform well \emph{with high probability}. Our key idea is that distribution shifts seen during training should inform us of probable shifts at test time, which we realize by explicitly relating training and test domains as draws from the same underlying meta-distribution. To achieve probable DG, we propose a new optimization problem called \emph{Quantile Risk Minimization} (QRM). By minimizing the $\alpha$-quantile of predictor's risk distribution over domains, QRM seeks predictors that perform well with probability $\alpha$. To solve QRM in practice, we propose the \textit{Empirical QRM}~(EQRM) algorithm and provide: (i) a generalization bound for EQRM; and (ii) the conditions under which EQRM recovers the causal predictor as $\alpha\! \to\! 1$. In our experiments, we introduce a more holistic quantile-focused evaluation protocol for DG and demonstrate that EQRM outperforms state-of-the-art baselines on datasets from WILDS and DomainBed. 
\end{abstract}

\section{Introduction}\vspace{-1mm}
\label{sec:intro}
Despite remarkable successes in recent years~\citep{lecun2015deep, silver2016mastering, jumper2021highly}, machine learning systems often fail calamitously when presented with \textit{out-of-distribution} (OOD) data~\citep{torralba2011unbiased, beery2018recognition, hendrycks2019benchmarking, geirhos2020shorcut}. 
Evidence of state-of-the-art systems failing in the face of distribution shift is mounting rapidly---be it due to spurious correlations~\citep{zech2018variable, arjovsky2020invariant, niven2020probing}, changing sub-populations~\citep{santurkar2020breeds, wilds2021, borkan2019nuanced}, changes in location or time~\citep{hansen2013high, christie2018functional, shankar2021image}, or other naturally-occurring variations~\citep{karahan2016image, azulay2019deep, eastwood2021source, hendrycks2021natural, hendrycks2021many,robey2021modelbased,zhou2022deep}. These OOD failures are particularly concerning in safety-critical applications such as medical imaging~\citep{jovicich2009mri, albadawy2018deep, tellez2019quantifying, beede2020, wachinger2021detect} and autonomous driving~\citep{dai2018dark, volk2019towards, michaelis2019dragon}, where they represent one of the most significant barriers to the real-world deployment of machine learning systems~\citep{ribeiro2016should, biggio2018wild, maartensson2020reliability, castro2020causality}.

Domain generalization (DG) seeks to improve a system's
OOD performance by leveraging datasets from multiple
environments or domains at training time, each collected under different experimental conditions~\citep{blanchard2011generalizing, muandet2013domain, gulrajani2020search} (see \cref{fig:fig1:train-test}). 
The goal is to build a predictor which exploits invariances across the training domains in the hope that these invariances also hold in related but distinct test domains~\citep{gulrajani2020search, schoelkopf2012causal, li2018learning, krueger21rex}. 
\looseness-1 To realize this goal, DG is commonly formulated as an average-~\citep{blanchard2011generalizing,blanchard2021domain,zhang2020adaptive} or worst-case~~\citep{ben2009robust, sagawa2019distributionally, arjovsky2020invariant} optimization problem
over the set of possible domains.
However, optimizing for average performance provably lacks robustness to OOD data~\cite{nagarajan2021understanding}, while optimizing for worst-domain performance
tends to lead to overly-conservative solutions, with worst-case outcomes unlikely in practice~\citep{tsipras2019robustness, raghunathan2019adversarial}.

In this work, we argue that DG is neither an average-case nor a worst-case problem, but rather a probabilistic one.
To this end, we propose a probabilistic framework for DG, which we call \textit{Probable Domain Generalization} (\cref{sec:qrm}), wherein the key idea is that distribution shifts seen during training should inform us of \emph{probable} shifts at test time.
\looseness-1 To realize this, we explicitly relate training and test domains as draws from the same underlying meta-distribution~(\cref{fig:fig1:q-dist}), and then propose a new optimization problem called \emph{Quantile Risk Minimization} (QRM). By minimizing the $\alpha$-quantile of predictor's risk distribution over domains~(\cref{fig:fig1:risk}), QRM seeks predictors that perform well \emph{with high probability} rather than on average or in the worst case. In particular, QRM leverages the key insight that this $\alpha$-quantile is an upper bound on the test-domain risk which holds with probability $\alpha$, meaning that $\alpha$ is an interpretable conservativeness-hyperparameter with $\alpha\! =\! 1$ corresponding to the worst-case setting.

To solve QRM in practice, we introduce the \textit{Empirical QRM}~(EQRM) algorithm (\cref{sec:qrm_algs}). Given a predictor's empirical risks on the training domains, EQRM forms an estimated risk distribution using kernel density estimation (KDE, \cite{parzen1962estimation}). Importantly, KDE-smoothing ensures a right tail that extends beyond the largest training risk (see \cref{fig:fig1:risk}), with this risk ``extrapolation''~\citep{krueger21rex} unlocking \emph{invariant prediction} for EQRM (\cref{sec:qrm_algs:eqrm}). 
We then provide theory for EQRM (\cref{sec:qrm_algs:gen_bound}, \cref{sec:qrm_algs:causality}) and demonstrate empirically that EQRM outperforms state-of-the-art baselines on real and synthetic data (\cref{sec:exps}).

\textbf{Contributions.} To summarize our main contributions:\vspace{-2mm}
\begin{itemize}
    \item \looseness-1 \textit{A new probabilistic perspective and objective for DG:} We argue that predictors should be trained and tested based on their ability to perform well \emph{with high probability}. We then propose Quantile Risk Minimization for achieving this \emph{probable} form of domain generalization (\cref{sec:qrm}).
    \item \textit{A new algorithm:} \looseness-1 We propose the EQRM algorithm to solve QRM in practice and ultimately learn predictors that generalize with probability $\alpha$~(\cref{sec:qrm_algs}). We then provide several analyses of EQRM:\vspace{-0.75mm}
    \begin{itemize}
        \item \textit{Learning theory:} \looseness-1 We prove a uniform convergence bound, meaning the empirical $\alpha$-quantile risk tends to the population $\alpha$-quantile risk given sufficiently many domains and samples~(\cref{thm:simplified-gen-bound}).
        \item \textit{Causality.} We prove that EQRM learns predictors with invariant risk as $\alpha\! \to\! 1$ (\cref{prop:equalize_main}), then provide the conditions under which this is sufficient to 
        recover the causal predictor~(\cref{thm:causal_predictor}).
        \item \textit{Experiments:} We demonstrate that EQRM outperforms state-of-the-art baselines on several standard DG benchmarks, including \texttt{CMNIST}~\citep{arjovsky2020invariant} and datasets from WILDS~\citep{wilds2021} and DomainBed~\citep{gulrajani2020search}, and highlight the importance of assessing the tail or \emph{quantile performance} of DG algorithms~(\cref{sec:exps}).
    \end{itemize}
\end{itemize}

\begin{figure}[tb]\vspace{-4mm}
    \centering
    \begin{subfigure}[b]{0.24\linewidth}
        \centering
        \includegraphics[width=\linewidth]{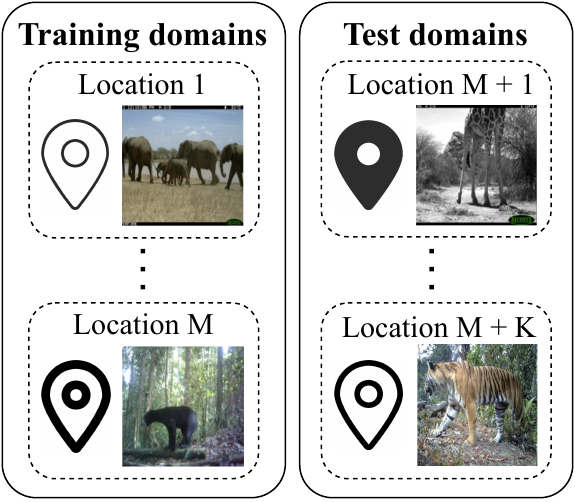}
        \vspace{0.1mm}
        \caption{}
        \label{fig:fig1:train-test}
    \end{subfigure}
    \hfill
    \begin{subfigure}[b]{0.365\linewidth}
        \centering
        \includegraphics[width=0.95\textwidth]{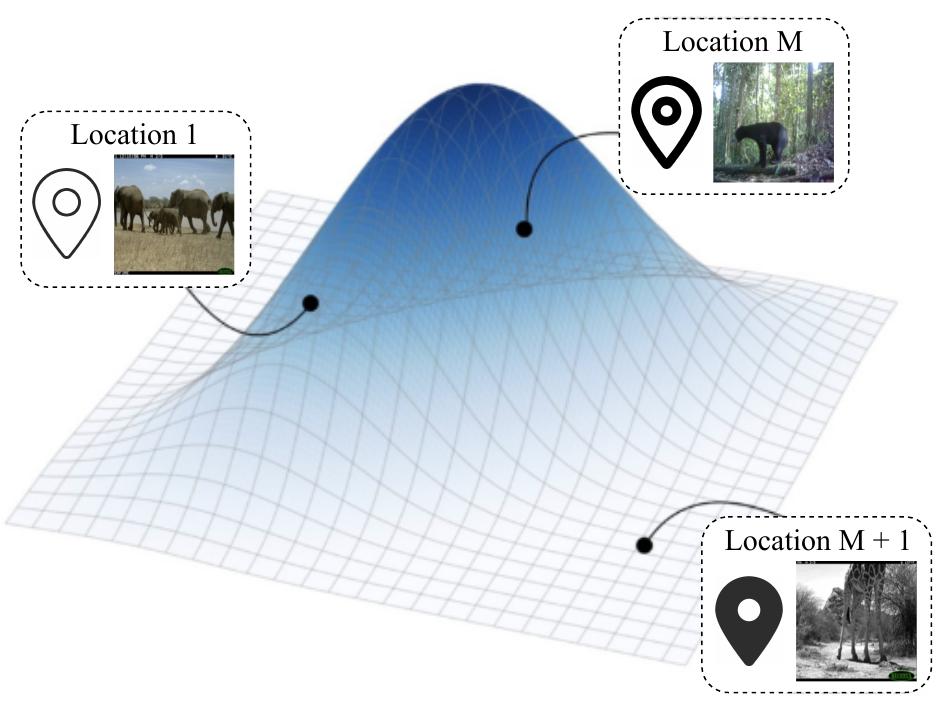}
        \caption{}
        \label{fig:fig1:q-dist}
    \end{subfigure}
    \hfill
    \begin{subfigure}[b]{0.36\linewidth}
        \centering
        \includegraphics[width=\linewidth]{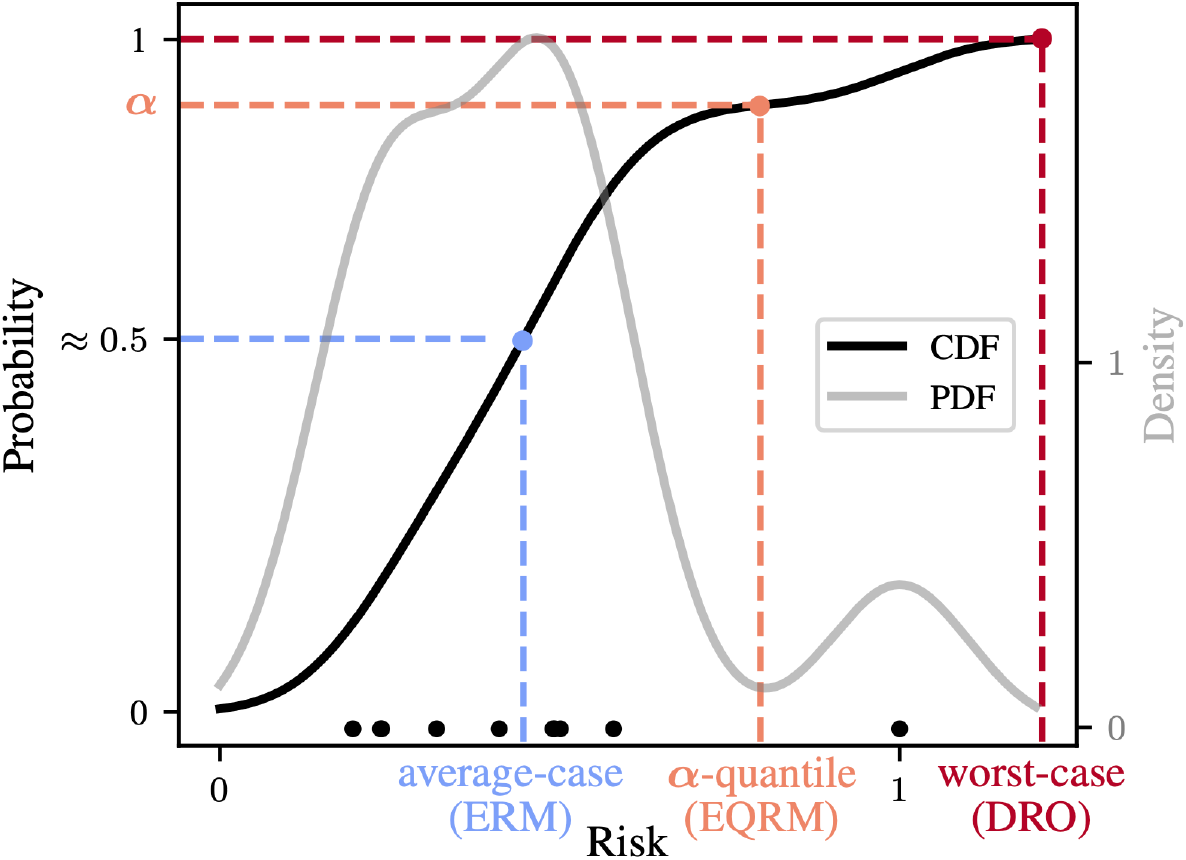}
        \caption{}
        \label{fig:fig1:risk}
    \end{subfigure}
    \vspace{-1.5mm}
    \caption{\small \textbf{Overview of Probable Domain Generalization and Quantile Risk Minimization.} (a) In domain generalization, training and test data are drawn from multiple related distributions or domains. For example, in the \texttt{iWildCam} dataset~\citep{beery2021iwildcam}, which contains camera-trap images of animal species, the domains correspond to the different camera-traps which captured the images. (b) We relate training and test domains as draws from the same underlying (and often unknown) meta-distribution over domains $\bbQ$. (c) We consider a predictor's estimated risk distribution over training domains, naturally-induced by $\bbQ$. \looseness-1 By minimizing the $\alpha$-quantile of this distribution, we learn predictors that perform well with high probability ($\approx \alpha$) rather than on average or in the worst case.
    }\label{fig:fig1}
\end{figure}

\section{Background: Domain generalization}\label{sec:backgr}
\textbf{Setup.}
In domain generalization~(DG), predictors are trained on data drawn from multiple related training distributions or \textit{domains} and then evaluated on
related but unseen 
test domains. 
For example, in the \texttt{iWildCam} dataset~\citep{beery2021iwildcam}, the task is to classify animal species in images, and the domains correspond to the different camera-traps which captured the images~(see \cref{fig:fig1:train-test}). 
More formally,
we consider datasets $D^e = \{(x^e_i, y^e_i)\}_{i=1}^{n_e}$ collected from $m$ different training domains or \textit{environments} $\Etr:= \{e_1, \dots, e_m\}$, with each dataset $D^e$ containing data pairs $(x^e_i, y^e_i)$ sampled i.i.d.\ from $\Prob(X^e,Y^e)$. \looseness-1 Then, given a suitable function class $\calF$ and loss function $\ell$, the goal of DG is to learn a predictor $f\in\calF$ that generalizes to data drawn from a larger set of all possible
domains $\Eall \supset \Etr$. 

\textbf{Average case.} Letting $\calR^e(f)$ denote the statistical risk of $f$ in domain $e$, and $\bbQ$ a distribution over the domains in $\Eall$, DG was first formulated~\citep{blanchard2011generalizing, muandet2013domain} as the following average-case problem:
\begin{equation}\label{eq:domain-gen-average-case}
    \min_{f\in\calF} \bbE_{e \sim \bbQ} \calR^e(f)
    \qquad \text{where} \qquad 
    \calR^e(f) := \mathbb{E}_{\bbP(X^e, Y^e)} [\ell(f(X^e), Y^e)].
\end{equation}

\textbf{Worst case.} \looseness-1 Since predictors that perform well \emph{on average} provably lack robustness~\citep{nagarajan2021understanding}, i.e.\ they can perform quite poorly on large subsets of $\Eall$, subsequent works~\citep{ben2009robust, sagawa2019distributionally, arjovsky2020invariant, krueger21rex, ahuja2021invariance, robey2021modelbased} have sought robustness by formulating DG as the following \emph{worst-case} problem:
\begin{equation}
\label{eq:domain-gen}
    \min_{f\in\calF} \max_{e \in \Eall} \calR^e(f).
\end{equation}
As we only have access to data from a finite subset of $\Eall$ during training, solving \eqref{eq:domain-gen} is not just challenging but in fact impossible~\citep{krueger21rex,ben2010theory,christiansen2021causal} without restrictions on how the domains may differ. %
\textbf{Causality and invariance in DG.} Causal works on DG~\citep{arjovsky2020invariant, krueger21rex, peters2016causal, christiansen2021causal, rojas2018invariant} describe domain differences using the language of causality and the notion of \textit{interventions}~\cite{pearl2009causality, peters2017elements}. In particular, they assume
all domains share the same underlying \textit{structural causal model}~(SCM)~\cite{pearl2009causality}, with different domains corresponding to different interventions (see \cref{app:causality:defs} for formal definitions and a simple example).
Assuming the mechanism of $Y$ remains fixed or invariant
but all $X$s may be intervened upon, recent works have shown that only the causal predictor has invariant: (i) predictive distributions~\citep{peters2016causal}, coefficients~\citep{arjovsky2020invariant} or risks~\citep{krueger21rex} across domains; and (ii) generalizes to arbitrary interventions on the $X$s~\citep{peters2016causal, arjovsky2020invariant,rojas2018invariant}. \looseness-1 These works then
leverage some form of invariance across domains to discover causal relationships which, through the invariant mechanism assumption, generalize to new domains.

\section{Quantile Risk Minimization}
\label{sec:qrm}
In this section we introduce 
\textit{Quantile Risk Minimization} (QRM) for achieving \textit{Probable Domain Generalization}. The core idea
is to replace the worst-case perspective
of~\eqref{eq:domain-gen} with a probabilistic one. This approach is founded on a great deal of work in classical fields such as control theory~\cite{campi2008exact,ramponi2018consistency} and smoothed analysis~\cite{spielman2004smoothed}, wherein approaches that yield high-probability guarantees are used in place of worst-case approaches in an effort to mitigate conservatism and computational limitations.
This mitigation is of particular interest in domain generalization
since generalizing to arbitrary domains is impossible~\cite{krueger21rex,ben2010theory,christiansen2021causal}.
Thus,
motivated by this classical literature, our goal is to obtain predictors that are robust \emph{with high probability} over domains drawn from $\Eall$, rather than in the worst case.  

\textbf{A distribution over environments.} We start by assuming the existence of
a probability distribution $\mathbb{Q}(e)$ over the set of all environments $\Eall$. 
For instance, in the context of medical imaging, $\bbQ$ could represent a distribution over potential changes to a hospital's setup or simply a distribution over candidate hospitals.
Given that such a distribution $\bbQ$ exists\footnote{As $\bbQ$ is often unknown, our analysis does not rely on using an explicit expression for $\bbQ$.}, we can think of the risk $\calR^e(f)$ as a \emph{random variable} for each $f\in\calF$, where the randomness is engendered by the draw of $e\sim\bbQ$.  This perspective gives rise to the following analogue of the optimization problem in~\eqref{eq:domain-gen}:
\begin{equation}
  \min_{f\in\calF} \: \esssup_{e\sim\bbQ} \calR^e(f) \quad\text{where}\quad \esssup_{e\sim\bbQ} \calR^e(f) = \inf\Big\{ t\geq 0 : \Pr_{e\sim\bbQ} \left\{\calR^e(f) \leq t\right\} = 1\Big\} \label{eq:domain-gen-rewritten}
\end{equation}
Here, $\esssup$ denotes the \emph{essential-supremum} operator from measure theory, meaning that for each $f\in\calF$, $\esssup_{\bbQ} \calR^e(f)$ is the least upper bound on $\calR^e(f)$ that holds for almost every $e\sim\bbQ$.  In this way, the $\esssup$ in~\eqref{eq:domain-gen-rewritten} is the measure-theoretic analogue of the $\max$ operator in~\eqref{eq:domain-gen}, with the subtle but critical difference being that the $\esssup$ in~\eqref{eq:domain-gen-rewritten} can neglect domains of measure zero under $\bbQ$. For example, for discrete $\bbQ$, \eqref{eq:domain-gen-rewritten} ignores domains which are impossible (i.e.\ have probability zero) while~\eqref{eq:domain-gen} does not, laying the foundation for ignoring domains which are \emph{improbable}.

\textbf{High-probability generalization.}  Although the minimax problem in~\eqref{eq:domain-gen-rewritten} explicitly incorporates the distribution $\bbQ$ over environments, this formulation is no less conservative than~\eqref{eq:domain-gen}.  Indeed, 
in many cases,~\eqref{eq:domain-gen-rewritten} is equivalent to~\eqref{eq:domain-gen}; see Appendix~\ref{app:sup-and-esssup} for details.  Therefore, rather than considering the worst-case problem in~\eqref{eq:domain-gen-rewritten}, we propose the following generalization of~\eqref{eq:domain-gen-rewritten} which requires that predictors generalize with probability $\alpha$ rather than in the worst-case:
\begin{equation}
\begin{alignedat}{2}
\label{eq:prob_gen}
    &\min_{f\in\calF,\, \thres \in \bbR} &&\thres  \qquad \subjto \Pr_{e\sim\bbQ} \left\{\calR^e(f) \leq t \right\} \geq \alpha
\end{alignedat}
\end{equation}
The optimization problem in~\eqref{eq:prob_gen} formally defines what we mean by
\textit{Probable Domain Generalization}.
In particular, we say that \textit{a predictor~$f$ generalizes with risk~$t$ at level~$\alpha$} if $f$ has risk at most~$t$ with probability at least~$\alpha$ over domains sampled from $\bbQ$.
In this way, the conservativeness parameter~$\alpha$ controls the strictness of generalizing to unseen domains.

\textbf{A distribution over risks.}  The optimization problem presented in~\eqref{eq:prob_gen} offers a principled formulation for generalizing to unseen distributional shifts governed by $\bbQ$.  However, $\bbQ$ is often unknown in practice and its support $\Eall$ may be high-dimensional or challenging to define~\cite{robey2021modelbased}. 
While 
many previous works have made progress by limiting the scope of possible shift types
over domains~\cite{eastwood2021source,robey2021modelbased,sagawa2019distributionally}, 
in practice, such structural assumptions are often difficult to justify and impossible to test. For this reason, we start our exposition of QRM by offering an alternative view of~\eqref{eq:prob_gen} which elucidates how a predictor's \emph{risk distribution} plays a central role in achieving probable domain generalization.

\looseness-1 To begin, note
that for each $f\in\calF$, the distribution over domains $\bbQ$ naturally induces\footnote{\looseness-1 $\bbT_f$ can be formally defined as the push-forward measure of $\bbQ$ through the risk functional $\calR^e(f)$; see App.~\ref{app:sup-and-esssup}.} a distribution~$\bbT_f$ over the risks in each domain $\calR^e(f)$.
In this way, rather than considering the randomness of~$\bbQ$ in
the often-unknown and (potentially) high-dimensional space of possible 
shifts~(\cref{fig:fig1:q-dist}), one can consider it
in the real-valued space of risks~(\cref{fig:fig1:risk}).  
This is analogous to statistical learning theory, where the analysis of convergence of empirical risk minimizers (i.e., of functions) is substituted by that of a weaker form of convergence, namely that of scalar risk functionals---a crucial step for VC theory \cite{vapnik1999nature}.
From this perspective, the statistics of $\bbT_f$ can be thought of as capturing the sensitivity of $f$ to different environmental shifts, summarizing the effect of different intervention types, strengths, and frequencies.  To this end,~\eqref{eq:prob_gen} can be equivalently rewritten in terms of the risk distribution $\bbT_f$ as follows:
\vspace{-0.2em}
\begin{mdframed}[roundcorner=5pt, backgroundcolor=yellow!8,innertopmargin=8pt]
\begin{equation} \tag{QRM}
    \min_{f\in\calF} \: F^{-1}_{\bbT_f}(\alpha) \quad\text{where}\quad F^{-1}_{\bbT_f}(\alpha) := \inf \Big\{t\in\R : \Pr_{R\sim\bbT_f} \left\{ R \leq t\right\} \geq \alpha \Big\}. \label{eq:qrm}
\end{equation}
\end{mdframed}
\vspace{-0.8em}
Here, $F^{-1}_{\bbT_f}(\alpha)$ denotes the inverse CDF (or quantile\footnote{In financial optimization, when concerned with a distribution over potential 
losses, the $\alpha$-quantile value is known as the \textit{value at risk}
(VaR) at level $\cutoff$~\citep{duffie1997overview}.}) function of the risk distribution $\bbT_f$. By means of this reformulation, 
we elucidate how solving~\eqref{eq:qrm} amounts to finding a predictor with minimal $\alpha$-quantile risk.  That is,~\eqref{eq:qrm} requires that a predictor $f$ satisfy the probabilistic constraint for at least an $\alpha$-fraction of the risks $R\sim\mathbb{T}_f$, or, equivalently, for an $\alpha$-fraction of the environments $e\sim\mathbb{Q}$.
In this way, $\alpha$ can be used to interpolate between typical ($\alpha\!=\!0.5$, median) and worst-case ($\alpha\!=\!1$) problems in an interpretable manner. Moreover, if the mean and median of $\bbT_f$ coincide, $\alpha\!=\!0.5$ gives an average-case problem, with \eqref{eq:qrm} recovering several notable objectives for DG as special cases.
\begin{proposition}\label{prop:average-case-equiv}
For $\alpha\!\! =\!\! 1$,~\eqref{eq:qrm} is equivalent to the worst-case problem of~\eqref{eq:domain-gen-rewritten}.
\looseness-1 For $\alpha\! =\! 0.5$, it
is equivalent to the average-case problem of~\eqref{eq:domain-gen-average-case} if the mean and median of $\bbT_f$ coincide $\forall f\! \in\! \calF$:
\begin{align}\label{eq:dg-average-case}
\textstyle
    \min_{f\in\calF} \: \E_{R\sim\bbT_f} R =
    \min_{f\in\calF} \: \E_{e\sim\bbQ} \calR^e(f)
\end{align}%
\end{proposition}%

\textbf{Connection to DRO.} \looseness-1 While fundamentally different in terms of objective and generalization capabilities (see \cref{sec:qrm_algs}), we draw connections between QRM and distributionally robust optimization (DRO) in \cref{app:dro} by considering an alternative problem which optimizes the \emph{superquantile}.

\section{Algorithms for Quantile Risk Minimization}\label{sec:qrm_algs}\vspace{-1mm}
We now introduce the \emph{Empirical QRM}~(EQRM) algorithm for solving \eqref{eq:qrm} in practice, akin to Empirical Risk Minimization~(ERM) solving the Risk Minimization~(RM) problem~\citep{Vapnik98}. 

\subsection{From QRM to Empirical QRM}\label{sec:qrm_algs:eqrm}\vspace{-1.5mm}
In practice, given a predictor $f$ and its empirical risks $\hat{\calR}^{e_1}(f), \dots, \hat{\calR}^{e_m}(f)$ on the $m$ training domains, we must form
an \emph{estimated} risk distribution $\widehat{\bbT}_f$. 
In general, given no prior knowledge about the form of $\bbT_f$ (e.g.\ Gaussian), we 
use 
\textit{kernel density estimation} (KDE, \cite{rosenblatt1956remarks, parzen1962estimation}) with Gaussian kernels and either the Gaussian-optimal rule~\citep{silverman1986density} or Silverman's rule-of-thumb~\citep{silverman1986density} for bandwidth selection. \Cref{fig:fig1:risk} depicts the PDF and CDF for 10 training risks when using Silverman's rule-of-thumb.
Armed with a predictor's estimated risk distribution $\widehat{\bbT}_f$, we can approximately solve \eqref{eq:qrm} using the following empirical analogue:
\begin{equation} %
\begin{alignedat}{2}%
\label{eq:qrm2}
    \min_{f\in\calF}\ F^{-1}_{\widehat{\bbT}_f}(\alpha)
\end{alignedat}
\end{equation}
Note that~\eqref{eq:qrm2} depends only on known quantities so we can compute and minimize it in practice, as detailed in \cref{alg:eqrm} of \cref{sec:impl_details:algs}.

\begin{wrapfigure}{r}{0.29\textwidth}
\centering
    \centering
    \includegraphics[width=\linewidth]{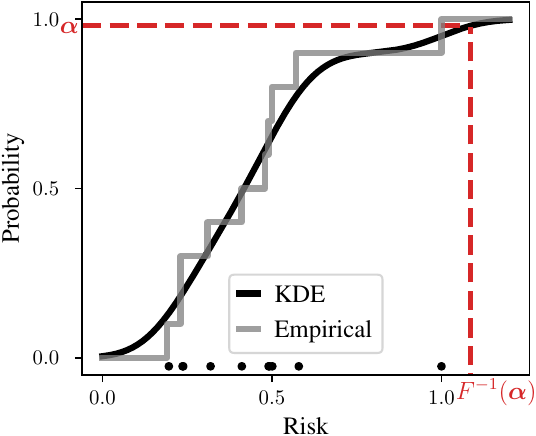}\vspace{-2mm}
    \caption{\small Risk CDFs.}
    \label{fig:kde-smoothing}
\end{wrapfigure}
\textbf{Smoothing permits risk extrapolation.}
\Cref{fig:kde-smoothing} compares the KDE-smoothed  CDF (black) to the unsmoothed empirical CDF (gray). As shown, the latter places zero probability mass on risks greater than our largest training risk, thus implicitly assuming that test risks cannot be larger than training risks. In contrast, the KDE-smoothed  CDF permits ``risk extrapolation''~\citep{krueger21rex} since its right tail extends beyond our largest training risk, with the estimated $\alpha$-quantile risk going to infinity as $\alpha\! \to\! 1$ (when kernels have full support). Note that different bandwidth-selection methods encode different assumptions about right-tail heaviness and thus about projected OOD risk. In \cref{sec:qrm_algs:causality}, we discuss how, as $\alpha\! \to\! 1$, this KDE-smoothing allows EQRM to learn predictors with invariant risk over domains.
In~\cref{app:kde}, we discuss different bandwidth-selection methods for EQRM.

\subsection{Theory: Generalization bound}\label{sec:qrm_algs:gen_bound}
We now give a simplified version of our main generalization bound---\cref{thm:generalization}---which states that, given sufficiently many domains and samples, the empirical $\alpha$-quantile risk
is a good estimate of the population $\alpha$-quantile risk. In contrast to previous results for DG, we bound the \emph{proportion of test domains} for which a predictor performs well, rather than the average error~\citep{blanchard2011generalizing, blanchard2021domain}, and make no assumptions about the shift type, e.g.\ covariate shift~\citep{muandet2013domain}. \looseness-1 The full version, stated and proved in Appendix~\ref{app:gen_bounds}, provides specific finite-sample bounds on $\epsilon_1$ and $\epsilon_2$ below, depending on the hypothesis class $\calF$, the empirical estimator
$F^{-1}_{\widehat{\bbT}_f}(\alpha)$,
and the assumptions on the possible risk profiles of hypotheses $f \in \calF$. 
\begin{theorem}[Simplified form of~\cref{thm:generalization}, uniform convergence]\label{thm:simplified-gen-bound}
Given $m$ domains and $n$ samples in each, then with high probability over the training data,
\begin{equation}
    \sup_{f \in \calF} \left|
    F^{-1}_{\bbT_f}(\alpha - \epsilon_2) - 
    F^{-1}_{\widehat{\bbT}_f}(\alpha)
    \right| \leq \epsilon_1,
    \label{ineq:simplified_generalization_bound}
\end{equation}
where $\epsilon_1 \to 0$ as $n \to \infty$ and $\epsilon_2 \to 0$ as $m \to \infty$.
\end{theorem}
While many domains are required for this to bound be tight, i.e.\ for $\alpha$ to \emph{precisely} estimate the true quantile, our empirical results in \cref{sec:exps} demonstrate that EQRM performs well in practice given only a few domains. In such settings, $\alpha$ still controls conservativeness, but with a less precise interpretation.

\subsection{Theory: Causal recovery}\label{sec:qrm_algs:causality}
We now prove that EQRM can recover the causal predictor in two parts. First, we show that, as $\alpha \to 1$, EQRM learns a predictor with minimal, invariant risk over domains. For Gaussian estimators of the risk distribution $\bbT_f$, some intuition can be gained from \cref{eq:qrm-gaussian} of \cref{app:causality:discovery:gaussian}, noting that $\alpha \to 1$ puts increasing weight on the sample standard deviation of risks over domains $\hat{\sigma}_f$, eventually forcing it to zero.
For kernel density estimators, a similar intuition applies so long as the bandwidth has a certain dependence on $\hat{\sigma}_f$, as detailed in \cref{app:causality:discovery:kde}. Second, we show that learning such a \emph{minimal invariant-risk predictor} is sufficient to recover the causal predictor under weaker assumptions than prior work, namely~\citet{peters2016causal} and \citet{krueger21rex}. Together, these two parts provide the conditions under which EQRM successfully performs ``causal recovery'', i.e., correctly recovers the true causal coefficients in a linear causal model of the data.
\begin{definition}\upshape
A predictor~$f$ is said to be an \emph{invariant-risk predictor} if its risk is equal almost surely across domains (i.e., $\operatorname{Var}_{e \sim \bbQ}[\calR^e(f)] = 0$). \looseness-1 A predictor is said to be a \emph{minimal invariant-risk predictor} if it achieves the minimal possible risk across all possible invariant-risk predictors.
\end{definition}
\begin{proposition}[EQRM learns a minimal invariant-risk predictor as $\alpha\! \to\! 1$, informal version of \cref{prop:Gaussian_QRM_invariant,prop:KDE_QRM_invariant}]
    Assume: (i) $\calF$ contains an invariant-risk predictor 
    with finite training risks; and (ii) no arbitrarily-negative training risks. \looseness-1 Then, as $\alpha\! \to\! 1$, Gaussian and kernel EQRM predictors (the latter with certain bandwidth-selection methods) converge to minimal invariant-risk predictors.
    \label{prop:equalize_main}
\end{proposition}%
\cref{prop:Gaussian_QRM_invariant,prop:KDE_QRM_invariant} are stated and proved in~\cref{app:causality:discovery:gaussian,app:causality:discovery:kde} respectively. In addition, for the special case of Gaussian estimators of $\bbT_f$, \cref{app:causality:discovery:gaussian} relates our $\alpha$ parameter to the $\beta$ parameter of VREx~\cite[Eq.~8]{krueger21rex}. We next specify the conditions under which learning such a minimal invariant-risk predictor is sufficient to recover the causal predictor.
\begin{theorem}[The causal predictor is the only minimal invariant-risk predictor]
    \label{thm:causal_predictor}
    Assume that: (i)~$Y$ is generated from a linear SEM, $Y = \beta^\intercal X + N$, with $X$ observed and coefficients $\beta \in \bbR^d$; (ii) $\calF$ is the class of linear predictors, indexed by $\hat\beta \in \bbR^d$; (iii) the loss $\ell$ is squared-error; (iv) the risk $\bbE[(Y - \beta^TX)^2]$ of the causal predictor $\beta$ is invariant
    across domains; and (v) \looseness-1 the system of equations
    \begin{align}
        \notag
        0 \geq
        & x^\intercal \text{\emph{Cov}}_{X \sim e_1}(X, X) x
            + 2 x^\intercal \text{\emph{Cov}}_{N,X \sim e_1} (X, N) \\
        \notag
        = & \cdots \\
        \label{eq:causal_recovery_equations_main}
        = & x^\intercal \text{\emph{Cov}}_{X \sim e_m}(X, X) x
            + 2 x^\intercal \text{\emph{Cov}}_{N,X \sim e_m} (X, N)
    \end{align}
    has the unique solution $x = 0$. If $\hat\beta$ is a minimal invariant-risk predictor, then $\hat\beta=\beta$.
\end{theorem}%

\textbf{Assumptions (i--iii).} \looseness-1 The assumptions that $Y$ is drawn from a linear structural equation model~(SEM) and that the loss is squared-error, while restrictive, are needed for all comparable causal recovery results~\citep{peters2016causal, krueger21rex}. In fact, these assumptions are weaker than both \citet[Thm.~2]{peters2016causal} (assume a linear \emph{Gaussian} SEM for $X$ \emph{and} $Y$) and  \citet[Thm.~1]{krueger21rex} (assume a linear SEM for $X$ \emph{and} $Y$). 

\textbf{Assumption (iv).} The assumption that the risk of the causal predictor is invariant across domains, often called \emph{domain homoskedasticity}~\citep{krueger21rex}, is necessary for any method inferring causality from the \emph{invariance of risks} across domains.
For methods based on the \emph{invariance of functions}, namely the conditional mean $\E[Y|\PA(Y)]$~\citep{arjovsky2020invariant, yin2021optimization}, this assumption is not required. \cref{sec:additional_exps:linear_regr:risks_vs_functions} compares methods based on invariant risks and to those based on invariant functions.

\textbf{Assumption (v).} In contrast to both \citeauthor{peters2016causal} and \citeauthor{krueger21rex}, we do not require specific types of interventions on the covariates. Instead, we require that a more general condition be satisfied, namely that the system of $d$-variate quadratic equations in~\eqref{eq:causal_recovery_equations_main} has a unique solution. Intuitively, $\text{Cov}(X, X)$ captures how correlated the covariates are and ensures they are sufficiently uncorrelated to distinguish each of their influences on $Y$, while $\text{Cov}(X, N)$ captures how correlated descendant covariates are with $Y$ (via $N$). Together, these terms capture the idea that \emph{predicting $Y$ from the causal covariates must result in the minimal invariant-risk}:
the first inequality ensures the risk is \emph{minimal} and the subsequent $m - 1$ equalities that it is \emph{invariant}. While this generality comes at the cost of abstraction, \cref{app:causal_recovery} provides several concrete examples with different types of interventions to aid understanding and illustrate how this condition generalizes existing causal-recovery results based on invariant risks~\citep{peters2016causal, krueger21rex}. \looseness-1 \Cref{app:causal_recovery} also provides a proof of~\cref{thm:causal_predictor} and further discussion.

\section{Related work}\label{sec:related}
\textbf{Robust optimization in DG.}  Throughout this paper, we follow an established line of work (see e.g.,~\cite{arjovsky2020invariant,krueger21rex,ahuja2021invariance}) which formulates the DG problem through the lens of robust optimization~\cite{ben2009robust}.  To this end, various algorithms have been proposed for solving constrained~\cite{robey2021modelbased} and distributionally robust~\cite{sagawa2019distributionally} variants of the worst-case problem in~\eqref{eq:domain-gen}.  
Indeed, this robust formulation has a firm foundation in the broader machine learning literature, with notable works in adversarial robustness~\cite{goodfellow2014explaining,madry2017towards,zhang2019theoretically,robey2021adversarial,zhu2021adversarially} and fair learning~\cite{martinez2021blind,diana2021minimax} employing similar formulations. Unlike these past works, we consider a robust but non-adversarial formulation for DG, where predictors are trained to generalize with high probability rather than in the worst case. Moreover, the majority of this literature---both within and outside of DG---relies on specific structural assumptions (e.g.\ covariate shift) on the types of possible interventions or perturbations.
In contrast, we make the weaker and more flexible assumption of i.i.d.-sampled domains, which ultimately makes use of the observed domain-data to determine the types of shifts that are \emph{probable}. We further discuss this important difference in~\cref{sec:discussion}. 

\textbf{Other approaches to DG.}  Outside of robust optimization, many algorithms have been proposed for the DG setting which draw on insights from a diverse array of fields, including approaches based on tools from meta-learning~\cite{li2018learning,balaji2018metareg,dou2019domain,shu2021open,zhang2020adaptive}, kernel methods~\cite{dubey2021adaptive,deshmukh2019generalization}, and information theory~\cite{ahuja2021invariance}.
Also prominent are works that design regularizers to generalize OOD~\cite{zhao2020domain,li2020domain,kim2021selfreg} and works that seek domain-invariant representations~\cite{ganin2016domain,li2018deep,huang2020self}.
Many of these works employ hyperparameters which are difficult to
interpret, which has no doubt contributed to the well-established model-selection problem in DG~\cite{gulrajani2020search}.  In contrast, in our framework, $\alpha$ can be easily interpreted in terms of quantiles of the risk distribution. 
In addition, many of these works do not explicitly relate the training and test domains, meaning they lack theoretical results in the non-linear setting (e.g.~\cite{su2019one, arjovsky2020invariant, zhang2020adaptive, krueger21rex}). For those which do, they bound either average error over test domains~\citep{blanchard2011generalizing,blanchard2021domain,garg2021learn} or worst-case error under specific shift types (e.g.\ covariate~\citep{robey2021modelbased}).
\looseness-1 As argued above, the former lacks robustness while the latter can be both overly-conservative and difficult to justify in practice, where shift types are often unknown.

\textbf{High-probability generalization.}  As noted in \S~\ref{sec:qrm}, relaxing worst-case problems in favor of probabilistic ones has a long history in
control theory~\cite{campi2008exact,ramponi2018consistency,tempo2013randomized,lindemann2021stl,lindemann2022temporal}, operations research~\cite{shapiro2021lectures}, and smoothed analysis~\cite{spielman2004smoothed}. Recently, this paradigm 
has been applied to several
areas of machine learning, including perturbation-based robustness~\cite{robey2022probabilistically,rice2021robustness}, fairness~\cite{li2020tilted}, active learning~\cite{curi2020adaptive}, and reinforcement learning~\cite{paternain2022safe,chow2017risk}. However, 
it has not yet been applied to domain generalization.

\textbf{Quantile minimization.} In financial optimization, the quantile and superquantile functions~\citep{duffie1997overview, rockafellar2000CVaR,krokhmal2002portfolio} are central to the literature surrounding portfolio risk management, with numerous applications spanning banking regulations and insurance policies~\cite{wozabal2012value,jorion1997value}.
In statistical learning theory, several recent papers have derived uniform convergence guarantees in terms of alternative risk functionals besides expected risk~\citep{lee2020learning,khim2020uniform,duchi2021learning,curi2020adaptive}. These results focus on functionals that can be written in terms of expectations over the loss distribution (e.g., the superquantile). In contrast, our uniform convergence guarantee (Theorem~\ref{thm:generalization}) shows uniform convergence of the quantile function, which \emph{cannot} be written as such an expectation; this necessitates stronger conditions to obtain uniform convergence, which ultimately suggest regularizing the estimated risk distribution (e.g.\ by kernel smoothing).

\textbf{Invariant prediction and causality.}
Early work studied the problem of learning from multiple cause-effect datasets that share a functional mechanism but differ in noise distributions \citep{schoelkopf2012causal}. More generally, given (data from) multiple distributions, one can try to identify components which are stable, robust, or \emph{invariant}, and find means to transfer them across problems~\cite{zhang2013domain,Bareinboim2014,zhang2015multi,gong2016domain,HuaZhaZhaSanGlySch17}. 
As discussed in~\cref{sec:backgr}, recent works have leveraged different forms of invariance across domains to discover causal relationships which, under the invariant mechanism assumption~\cite{peters2017elements}, generalize to new domains~\cite{peters2016causal, rojas2018invariant, arjovsky2020invariant, krueger21rex, heinze2018invariant, pfister2019invariant, gamella2020active}.
In particular, VREx~\citep{krueger21rex} leveraged \textit{invariant risks} (like EQRM) while IRM~\citep{arjovsky2020invariant} leveraged \textit{invariant functions} or coefficients---see \cref{sec:additional_exps:linear_regr:risks_vs_functions} for a detailed comparison of these approaches.

\section{Experiments}%
\label{sec:exps}%
We now evaluate our EQRM algorithm on synthetic datasets~(\cref{sec:exps:synthetic}), real-world datasets from WILDS~(\cref{sec:exps:real}), and few-domain datasets from DomainBed~(\cref{sec:exps:domainbed}). \Cref{sec:additional_exps} reports further results, while \cref{sec:impl_details} reports further experimental details.

\subsection{Synthetic datasets}%
\label{sec:exps:synthetic}%
\begin{figure}[tb]\vspace{-4mm}
    \centering
    \includegraphics[width=\linewidth]{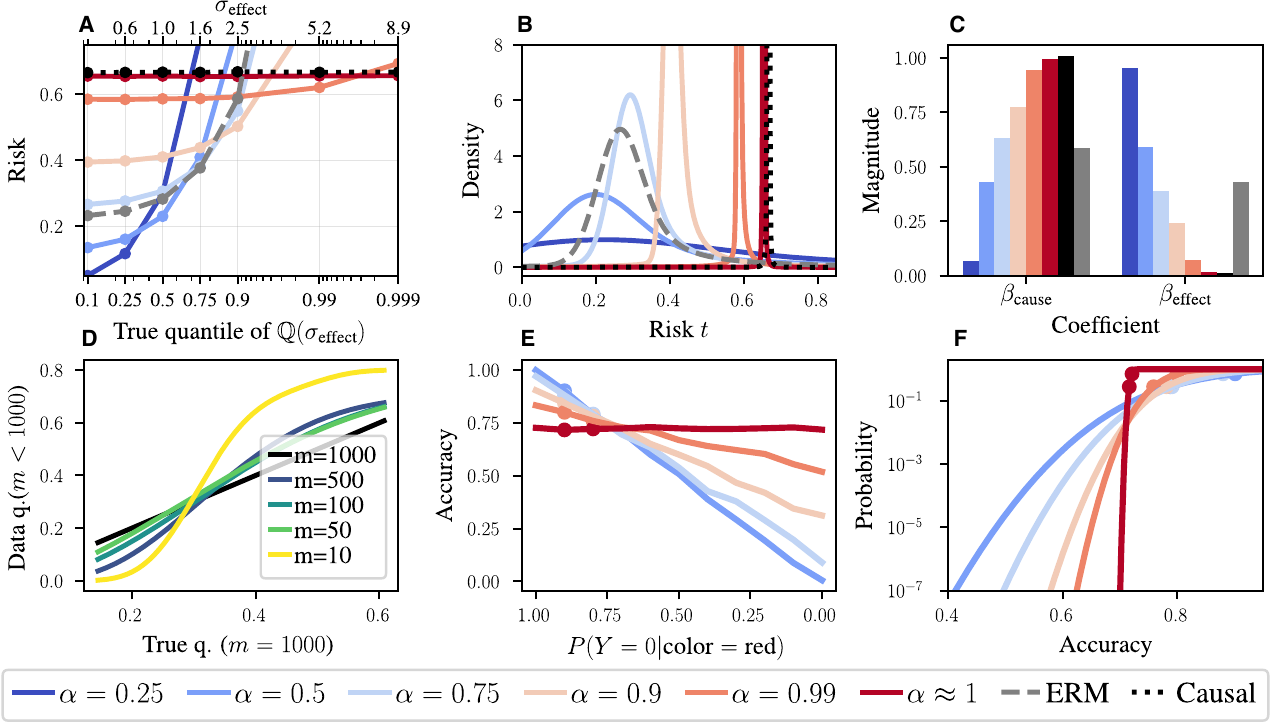}
    \caption{\small
    \textbf{EQRM on a toy linear regression dataset (A--D) and on ColoredMNIST (E--F).} 
    \textbf{A:} Test risk at different quantiles or degrees of ``OODness''. For each quantile (x-axis), the corresponding $\alpha$ has the lowest risk (y-axis).
    \textbf{B:} Estimated risk distributions (corresponding CDFs 
    in~\cref{sec:additional_exps:linear_reg:cdf-curves}).
    \textbf{C:} Regression coefficients approach those of the causal predictor ($\beta_{\text{cause}}\! =\! 1, \beta_{\text{effect}}\! =\! 0$) as $\cutoff\! \to\! 1$. \textbf{D:} Q-Q plot comparing the ``true'' risk quantiles (estimated with $m\! =\! 1000$) against estimated ones ($m\! <\! 1000$), for $\alpha\! =\! 0.9$. \textbf{E:} Performance of different $\alpha$'s over increasingly OOD test domains, with dots showing training-domain accuracies. \textbf{F:} KDE-estimated accuracy-CDFs depicting accuracy-robustness curves. \looseness-1 Larger $\alpha$'s make lower accuracies less likely.}
    \label{fig:exps:linear-regr}\vspace{-1.5mm}
\end{figure}
\looseness-1 \textbf{Linear regression.} We first consider a linear regression dataset based on
the following linear
SCM: 
\begin{align*}\vspace{-2mm}
\textstyle
    X_1     \gets              N_1, \qquad\qquad\qquad
    Y       \gets X_1        + N_Y, \qquad\qquad\qquad
    X_2     \gets Y        + N_2,
\end{align*}\vspace{-1mm}%
with $N_j \sim \calN(0, \sigma^2_j)$. Here we have two features: one cause $X_1\! =\! X_{\text{cause}}$ and\vspace{0.25mm}
one effect $X_2\! =\! X_{\text{effect}}$ of $Y$. 
By fixing $\sigma_{1}^2\! =\! 1$ and $\sigma_{Y}^2\! =\! 2$ across domains but sampling $\sigma_{2}\! \sim\! \text{LogNormal}(0, 0.5)$, \vspace{0.5mm}
we create a dataset in which $X_2$ is more predictive of $Y$ than $X_1$ but less stable. 
Importantly, as we know the true distribution over domains $\bbQ(e)\! =\! \text{LogNormal}(\sigma_2^e ;0, 0.5)$, we know the true risk quantiles. \Cref{fig:exps:linear-regr} depicts results for different $\alpha$'s with $m\! =\! 1000$ domains and $n\! =\! 200000$ samples in each, using the mean-squared-error (MSE) loss. Here we see that: \textbf{A:} for each true quantile (x-axis), the corresponding $\alpha$ has the lowest risk (y-axis), confirming that the empirical $\alpha$-quantile risk is a good estimate of the population $\alpha$-quantile risk; \textbf{B:} As $\cutoff\! \to\! 1$, the estimated risk distribution of $f_{\alpha}$ approaches an invariant (or Dirac delta) distribution centered on the risk of the causal predictor; \textbf{C:} the regression coefficients approach those of the causal predictor as $\cutoff\! \to\! 1$, trading predictive performance for robustness; and \textbf{D:} reducing the number of domains $m$ reduces the accuracy of the estimated $\alpha$-quantile risks. In \cref{sec:additional_exps:linear_regr}, we additionally: (i) depict the risk CDFs corresponding to plot \textbf{B} above, and discuss how they depict the predictors' risk-robustness curves~(\ref{sec:additional_exps:linear_reg:cdf-curves}); and (ii) discuss the solutions of EQRM on datasets in which $\sigma_1^2$, $\sigma_2^2$ and/or $\sigma_Y^2$ change over domains, compared to existing invariance-seeking algorithms like IRM~\citep{arjovsky2020invariant} and VREx~\citep{krueger21rex}~(\ref{sec:additional_exps:linear_regr:risks_vs_functions}).

\textbf{ColoredMNIST.} We next consider the \texttt{ColoredMNIST} or \texttt{CMNIST} dataset~\citep{arjovsky2020invariant}. Here, the \texttt{MNIST} dataset is used to construct a binary classification task (0--4 or 5--9) in which digit color (red or green) is a highly-informative but spurious feature. In particular, the two training domains are constructed such that red digits have an 80\% and 90\% chance of belonging to class 0, while the single test domain is constructed such that they only have a 10\% chance.
The goal is to learn an invariant predictor which uses only digit shape---a stable feature having a 75\% chance of correctly determining the class in all 3 domains. We compare with IRM~\citep{arjovsky2020invariant}, GroupDRO~\citep{sagawa2019distributionally}, SD~\citep{pezeshki2021gradient}, IGA~\citep{koyama2020out} and VREx~\citep{krueger21rex} using: (i) random initialization (Xavier method~\cite{glorot2010understanding}); and (ii) random initialization followed by several iterations of ERM. The ERM initialization or pretraining directly corresponds to the delicate penalty ``annealing'' or warm-up periods used by most penalty-based methods~\citep{arjovsky2020invariant, krueger21rex, pezeshki2021gradient, koyama2020out}. For all methods, we use a 2-hidden-layer MLP with 390 hidden units, the Adam optimizer, a learning rate of $0.0001$, and dropout with $p\! =\! 0.2$. We sweep over five penalty weights
for the baselines and five $\alpha$'s 
for EQRM. See \cref{sec:impl_details:cmnist} for more experimental details. \cref{tab:cmnist-results} shows that: (i) all methods struggle without ERM pretraining, explaining the need for penalty-annealing strategies in previous works and corroborating the results of \cite[Table 1]{zhang2022rich}; (ii) with ERM pretraining, EQRM matches or outperforms baseline methods, even approaching oracle performance (that of ERM trained on grayscale digits). \looseness-1 These results suggest ERM pretraining as an effective strategy for DG methods.

In addition, \Cref{fig:exps:linear-regr} depicts the behavior of EQRM with different $\alpha$s. Here we see that: \textbf{E:} increasing $\alpha$ leads to more consistent performance across domains, eventually forcing the model to ignore color and focus on shape for invariant-risk prediction; and \textbf{F:} a predictor's (estimated) accuracy-CDF depicts its accuracy-robustness curve, just as its risk-CDF depicts its risk-robustness curve.\looseness-1\ Note that $\alpha\! =\! 0.5$ gives the best worst-case (i.e.\ worst-domain) risk over the two training domains---the preferred solution of DRO~\citep{sagawa2019distributionally}---while $\alpha\! \to\! 1$ sacrifices risk for increased invariance or robustness.

\subsection{Real-world datasets}\label{sec:exps:real}\vspace{-2mm}
We now evaluate our methods on the real-world or \textit{in-the-wild} distribution shifts of WILDS~\citep{wilds2021}.  We focus our evaluation on \texttt{iWildCam}~\cite{beery2021iwildcam} and \texttt{OGB-MolPCBA}~\cite{hu2020open,wu2018moleculenet}---two large-scale classification datasets which have numerous test domains and thus facilitate a comparison of the test-domain risk distributions and their quantiles.
Additional comparisons
(e.g.\ using average accuracy) can be found in Appendix~\ref{sec:additional_exps:wilds}.  Our results demonstrate that, across two distinct data types (images and molecular graphs), EQRM offers superior tail or quantile performance.

\begin{figure}\vspace{-2mm}
\centering
\begin{minipage}{0.5\textwidth}
\centering
\captionsetup{type=table} %
\caption{\small \texttt{CMNIST} test accuracy.}\label{tab:cmnist-results}\vspace{-2mm}
        \resizebox{0.55\linewidth}{!}{%
        \begin{tabular}{@{}lcc@{}} \toprule
             \multirow{2}{*}{\textbf{Algorithm}} & \multicolumn{2}{c}{\textbf{Initialization}} \\ 
             \cmidrule(lr){2-3} & Rand. & ERM \\ \midrule
             ERM & $27.9 \pm 1.5$ & $27.9 \pm 1.5$ \\
             IRM & $\bm{52.5 \pm 2.4}$ & $69.7 \pm 0.9$ \\
             GrpDRO & $27.3 \pm 0.9$ & $29.0 \pm 1.1$ \\
             SD & $49.4 \pm 1.5$ & $70.3 \pm 0.6$ \\
             IGA & $50.7 \pm 1.4$ & $57.7 \pm 3.3$ \\
             V-REx & $\bm{55.2\pm 4.0}$ & $\bm{71.6 \pm 0.5}$ \\
             EQRM & $\bm{53.4 \pm 1.7}$ & $\bm{71.4 \pm 0.4}$ \\ \midrule
             Oracle & \multicolumn{2}{c}{$72.1 \pm 0.7$} \\ \bottomrule
        \end{tabular}}
\end{minipage}%
\begin{minipage}{0.5\textwidth}
\centering
\captionsetup{type=figure} %
\includegraphics[width=0.925\linewidth]{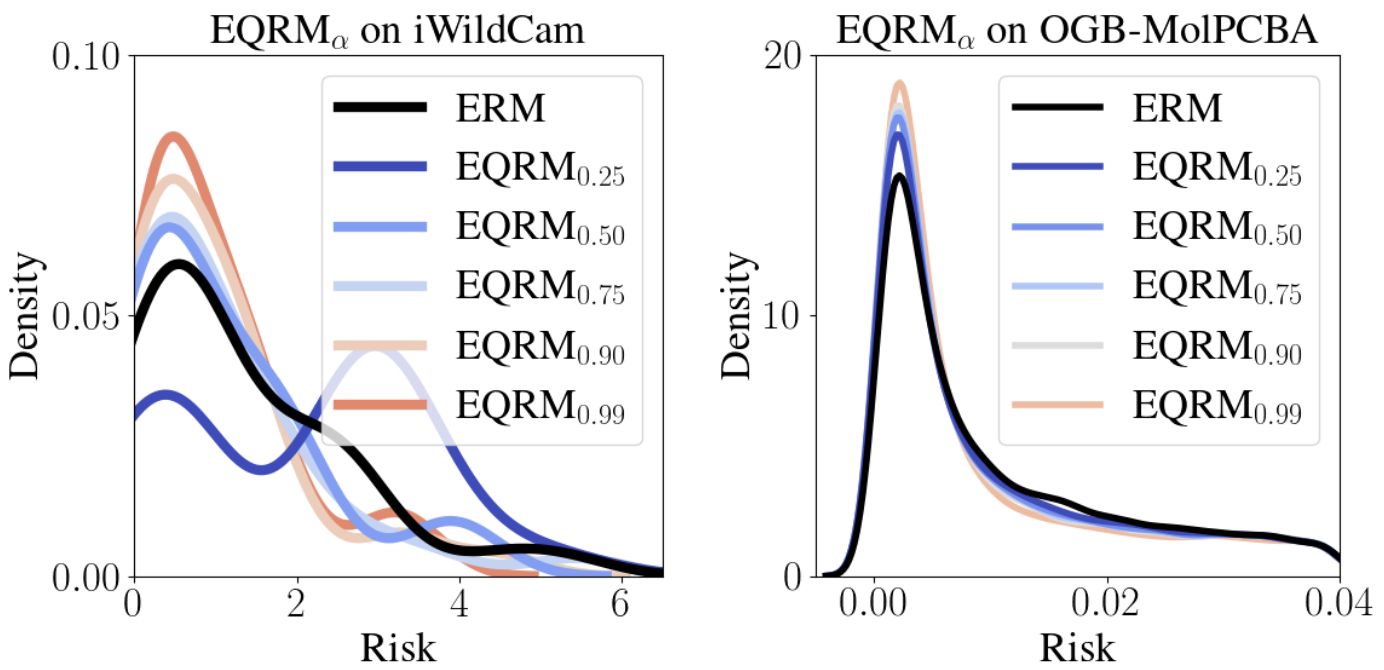}\vspace{-1.5mm}
    \caption{\small Test-domain risk distributions.
    }
    \label{fig:real-world-pdfs}
\end{minipage}%
\end{figure}

\begin{table}[tb]\vspace{-4mm}
    \begin{minipage}{0.475\textwidth}
    \centering
    \caption{\small EQRM test risks on \texttt{iWildCam}.}\label{tab:quantiles-iwildcam}
    \resizebox{\columnwidth}{!}{
    \begin{tabular}{@{}ccccccccc@{}} \toprule
         \multirow{2}{*}{Alg.} & \multirow{2}{*}{\makecell{Mean \\ risk}} & \multicolumn{7}{c}{Quantile risk} \\ \cmidrule(lr){3-9}
         & & 0.0 & 0.25 & 0.50 & 0.75 & 0.90 & 0.99 & 1.0 \\ \midrule
         ERM & 1.31 & 0.015 & 0.42 & 0.76 & 2.25 & 2.73 & 4.99 & 5.25 \\
         IRM & 1.53 & 0.098 & 0.52 & 1.24 & 1.86 & 2.36 & 6.95 & 7.46 \\ 
         GroupDRO & 1.73 & 0.091 & 0.68 & 1.65 & 2.18 & 3.36 & 5.29 & 5.54 \\
         CORAL & 1.27 & 0.024 & 0.45 & 0.73 & 2.12 & 2.66 & 4.50 & 4.98\\ \midrule
         EQRM$_{0.25}$ & 2.03 & 0.024 & 0.46 & 2.70 & 3.01 & 3.48 & 5.03 & 5.26 \\
         EQRM$_{0.50}$ & 1.11 & \textbf{0.004} & 0.24 & 0.68 & 1.71 & 2.15 & 4.04 & 4.11 \\
         EQRM$_{0.75}$ & 1.05 & 0.009 & \textbf{0.21} & 0.68 & 1.50 & 2.35 & 4.88 & 5.45 \\
         EQRM$_{0.90}$ & \textbf{0.98} & 0.047 & 0.28 & \textbf{0.63} & \textbf{1.26} & \textbf{1.81} & 4.11 & 4.48 \\
         EQRM$_{0.99}$ & 0.99 & 0.12 & 0.35 & 0.64 & 1.30 & 2.00 & \textbf{3.44} & \textbf{3.55} \\ \bottomrule
    \end{tabular}}
    
    \end{minipage}\hfill
    \begin{minipage}{0.505\textwidth}
    \centering
    \caption{\small EQRM test risks on \texttt{OGB-MolPCBA}.}\label{tab:quantiles-ogb}
    \resizebox{\columnwidth}{!}{
    \begin{tabular}{@{}ccccccccc@{}} \toprule
         \multirow{2}{*}{Alg.} & \multirow{2}{*}{\makecell{Mean \\ risk}} & \multicolumn{7}{c}{Quantile risk} \\ \cmidrule(lr){3-9}
         & & 0.0 & 0.25 & 0.50 & 0.75 & 0.90 & 0.99 & 1.0 \\ \midrule
         ERM & \textbf{0.051} & 0.0 & 0.004 & 0.017 & 0.060 & 0.13 & 0.49 & 16.04 \\
         IRM & 0.073 & 0.098 & 0.52 & 1.24 & 1.86 & 2.36 & 6.95 & 7.46 \\ 
         GroupDRO & 0.21 & 0.091 & 0.68 & 1.65 & 2.18 & 3.36 & 5.29 & \textbf{5.54} \\ 
         CORAL & 0.055 & 0.0 & 0.12 & 0.32 & 1.23 & 2.01 & 5.76 & 7.44 \\
         \midrule
         EQRM$_{0.25}$ & 0.054 & 0.0 & 0.003 & 0.016 & 0.059 & 0.13 & 0.48 & 15.46 \\ 
         EQRM$_{0.50}$ & 0.052 & 0.0 & 0.003 & 0.015 & 0.059 & 0.13 & 0.48 & 11.33 \\
         EQRM$_{0.75}$ & 0.052 & 0.0 & 0.003 & 0.015 & 0.059 & 0.13 & 0.47 & 12.15 \\
         EQRM$_{0.90}$ & 0.052 & 0.0 & 0.003 & 0.015 & 0.059 & 0.12 & 0.47 & 10.81 \\
         EQRM$_{0.99}$ & 0.053 & 0.0 & 0.003 & \textbf{0.014} & \textbf{0.055} & \textbf{0.11} & \textbf{0.46} & 7.16 \\ \bottomrule
    \end{tabular}}
    \end{minipage}
\end{table}%

\textbf{iWildCam.} We first consider the \texttt{iWildCam} image-classification dataset, which has 243 training domains and 48 test domains. Here, the label $Y$ is one of 182 different animal species and the domain~$e$ is the camera trap which captured the image.
In Table~\ref{tab:quantiles-iwildcam}, we observe that EQRM$_\alpha$ does indeed tend to optimize the $\alpha$-risk quantile, 
with larger $\alpha$s during training resulting in lower test-domain risks at the corresponding quantiles.
In the left pane of \Cref{fig:real-world-pdfs}, we plot the (KDE-smoothed) test-domain risk distribution for ERM and EQRM.  Here we see a clear trend: as $\alpha$ increases, the tails of the risk distribution tend to drop below ERM, which corroborates the superior quantile performance reported in Table~\ref{tab:quantiles-iwildcam}.
Note that, in Table~\ref{tab:quantiles-iwildcam}, EQRM tends to record lower \textit{average} risks than ERM.  This
has several plausible explanations. First, the number of testing domains (48) is relatively small, which could result in a biased sample with respect to the training domains.  Second, the test domains may not represent i.i.d.\ draws from $\bbQ$, as WILDS~\citep{wilds2021} test domains tend to be more challenging.

\textbf{OGB-MolPCBA.} We next consider the \texttt{OGB-MolPCBA} (or \texttt{OGB}) dataset, which is a molecular graph-classification benchmark containing 44,930 training domains and 43,793 test domains with an average of $3.6$ samples per domain.  
Table~\ref{tab:quantiles-ogb} shows that ERM achieves the lowest \emph{average} test risk on \texttt{OGB},
in contrast to the \texttt{iWildCam} results,
while EQRM$_\alpha$ still achieves stronger quantile performance.
Of particular note is the fact that our methods significantly outperform ERM with respect to worst-case performance (columns/quantiles labeled 1.0); when QRM$_\alpha$ is run with large values of $\alpha$, we reduce the worst-case risk by more than a factor of two.  \looseness-1 In~\Cref{fig:real-world-pdfs}, we again see that the risk distributions of EQRM$_\alpha$ have lighter tails than that of ERM.

\textbf{A new evaluation protocol for DG.}
The analysis provided in Tables~\ref{tab:quantiles-iwildcam}-\ref{tab:quantiles-ogb} and \Cref{fig:real-world-pdfs} diverges from the standard evaluation protocol in DG~\cite{gulrajani2020search,wilds2021}. 
Rather than evaluating an algorithm's performance \emph{on average} across test domains, we seek to understand \emph{the distribution of its performance}---particularly in the tails by means of the quantile function.  This new evaluation protocol lays bare the importance of multiple test domains in DG benchmarks, allowing predictors' risk distributions to be analyzed and compared.  Indeed, as shown in Tables~\ref{tab:quantiles-iwildcam}-\ref{tab:quantiles-ogb},
solely reporting a predictor's average or worst risk over test domains can be misleading when assessing its ability to generalize OOD, indicating that the performance of DG algorithms was likely never ``lost'', as reported in~\cite{gulrajani2020search}, but rather invisible through the lens of average performance. This underscores the necessity of incorporating tail- or quantile-risk measures into a more holistic evaluation protocol for DG, ultimately providing a more nuanced and complete picture. In practice, which measure is preferred will depend on the application. For example, medical applications could have a human-specified robustness-level or quantile-of-interest.

\subsection{DomainBed datasets}\label{sec:exps:domainbed}\vspace{1mm}%
\begin{table}[tb]\vspace{-4mm}
    \centering
    \caption{\small DomainBed results. Model selection: training-domain validation set.}
    \label{tab:domainbed-results}
    \adjustbox{max width=0.8\linewidth}{%
    \begin{tabular}{@{}lcccccc@{}}
    \toprule
    \textbf{Algorithm}        & \textbf{VLCS}             & \textbf{PACS}             & \textbf{OfficeHome}       & \textbf{TerraIncognita}   & \textbf{DomainNet}        & \textbf{Avg}              \\
    \midrule
    ERM                      & 77.5 $\pm$ 0.4            & 85.5 $\pm$ 0.2            & 66.5 $\pm$ 0.3            & 46.1 $\pm$ 1.8            & 40.9 $\pm$ 0.1            & 63.3                      \\
    IRM                       & 78.5 $\pm$ 0.5            & 83.5 $\pm$ 0.8            & 64.3 $\pm$ 2.2            & 47.6 $\pm$ 0.8            & 33.9 $\pm$ 2.8            & 61.6                      \\
    GroupDRO                 & 76.7 $\pm$ 0.6            & 84.4 $\pm$ 0.8            & 66.0 $\pm$ 0.7            & 43.2 $\pm$ 1.1            & 33.3 $\pm$ 0.2            & 60.9                      \\
    Mixup                    & 77.4 $\pm$ 0.6            & 84.6 $\pm$ 0.6            & 68.1 $\pm$ 0.3            & 47.9 $\pm$ 0.8            & 39.2 $\pm$ 0.1            & 63.4                      \\
    MLDG                     & 77.2 $\pm$ 0.4            & 84.9 $\pm$ 1.0            & 66.8 $\pm$ 0.6            & 47.7 $\pm$ 0.9            & 41.2 $\pm$ 0.1            & 63.6                      \\
    CORAL                     & 78.8 $\pm$ 0.6            & 86.2 $\pm$ 0.3            & 68.7 $\pm$ 0.3            & 47.6 $\pm$ 1.0            & 41.5 $\pm$ 0.1            & \textbf{64.6}                      \\
    ARM                       & 77.6 $\pm$ 0.3            & 85.1 $\pm$ 0.4            & 64.8 $\pm$ 0.3            & 45.5 $\pm$ 0.3            & 35.5 $\pm$ 0.2            & 61.7                      \\
    VREx                      & 78.3 $\pm$ 0.2            & 84.9 $\pm$ 0.6            & 66.4 $\pm$ 0.6            & 46.4 $\pm$ 0.6            & 33.6 $\pm$ 2.9            & 61.9                      \\ \midrule
    EQRM                   & 77.8 $\pm$ 0.6            & 86.5 $\pm$ 0.2            & 67.5 $\pm$ 0.1            & 47.8 $\pm$ 0.6            & 41.0 $\pm$ 0.3            & 64.1                      \\
    \bottomrule
    \end{tabular}}
\end{table}\vspace{-2mm}

Finally, we consider the benchmark datasets of DomainBed~\cite{gulrajani2020search}, in particular \texttt{VLCS}~\citep{fang2013}, \texttt{PACS}~\citep{li2017deeper}, \texttt{OfficeHome}~\citep{venkateswara2017}, \texttt{TerraIncognita}~\citep{beery2018recognition} and \texttt{DomainNet}~\citep{peng2019moment}. As each of these datasets contain just 4 or 6 domains, it is not possible to meaningfully compare tail or quantile performance. Nonetheless, in line with much recent work, and to compare EQRM to a range of standard baselines on few-domain datasets, Table~\ref{tab:domainbed-results} reports DomainBed results in terms of the average performance across each choice of test domain. While EQRM outperforms most baselines, including ERM, we reiterate that comparing algorithms solely in terms of average performance can be misleading (see final paragraph of~\cref{sec:exps:real}). Full implementation details are given in \cref{sec:impl_details:domainbed}, with further results in \cref{sec:additional_exps:domainbed} (additional baselines, per-dataset results, and test-domain model selection).

\section{Discussion}\label{sec:discussion}\vspace{-1mm}
\textbf{Interpretable model selection.} $\alpha$ approximates the probability with which our predictor will generalize with risk below the associated $\alpha$-quantile value. Thus, $\alpha$ represents an interpretable parameterization of the risk-robustness trade-off. Such interpretability is critical for model selection in DG, and for practitioners with application-specific requirements on performance and/or robustness.

\textbf{The assumption of i.i.d.\ domains.} For $\alpha$ to approximate the probability of generalizing, training and test domains must be i.i.d.-sampled. While this is rarely true in practice---e.g.\ hospitals have shared funders, service providers, etc.---we can better satisfy this assumption by subscribing to a new data collection process in which we collect training-domain data which is representative of how the underlying system tends to change. For example: (i) randomly select 100 US hospitals; (ii) gather and label data from these hospitals; (iii) train our system with the desired $\alpha$; (iv) deploy our system to all US hospitals, where it will be successful with probability $\approx \alpha$. While this process may seem expensive, time-consuming and vulnerable (e.g.\ to new hospitals), it offers a promising path to machine learning systems which \textit{generalize with high probability}. Moreover, it is worth noting the alternative: prior works achieve generalization by assuming that only particular types of shifts can occur, e.g.\ covariate shifts~\citep{quinonero2008dataset,Storkey09,robey2021modelbased}, label shifts~\citep{Storkey09, lipton2018detecting}, concept shifts~\citep{moreno2012unifying}, measurement shifts~\citep{eastwood2021source}, mean shifts~\citep{rothenhausler2021anchor}, shifts which leave the mechanism of $Y$ invariant~\cite{schoelkopf2012causal,peters2016causal,arjovsky2020invariant,krueger21rex}, etc. In real-world settings, where the underlying shift mechanisms are often unknown, such assumptions are both difficult to justify and impossible to test. Future work could look to relax the i.i.d.-domains assumption by leveraging knowledge of domain dependencies (e.g.\ time).

\textbf{The wider value of risk distributions.} As demonstrated in \cref{sec:exps}, a predictor's risk distribution has value beyond quantile-minimization---it estimates the probability associated with each level of risk. 
Thus, regardless of the algorithm used, risk distributions can be used to analyze trained predictors. 

\section{Conclusion}\vspace{-1mm}
We have presented Quantile Risk Minimization for achieving \textit{Probable} Domain Generalization, i.e., learning predictors that perform well \emph{with high probability} rather than \emph{on-average} or \emph{in the worst case}. After explicitly relating training and test domains as draws from the same underlying meta-distribution, we proposed to learn predictors with minimal $\alpha$-quantile risk. We then introduced the EQRM algorithm, for which we proved a generalization bound and recovery of the causal predictor
as $\alpha\! \to\! 1$.
\looseness-1 Finally, in our experiments, we introduced a more holistic quantile-focused evaluation protocol for DG, and demonstrated that EQRM outperforms state-of-the-art baselines on several DG benchmarks.

\acksection\vspace{-2.5mm}
We thank Chris Williams, Ian Mason and Krikamol Muandet for providing feedback on an earlier draft, as well as Lars Lorch, David Krueger, Francesco Locatello and members of the MPI T\"ubingen causality group for helpful discussions and comments. We also thank Minsu Kim for catching an error in an earlier proof of Theorem~\ref{thm:generalization}.
This work was supported by the German Federal Ministry of Education and Research (BMBF): Tübingen AI Center, FKZ: 01IS18039A, 01IS18039B; and by the Machine Learning Cluster of Excellence, EXC number 2064/1 – Project number 390727645.

\bibliography{bib}
\bibliographystyle{unsrtnat}

\newpage
\appendix

\addcontentsline{toc}{section}{Appendices}%
\part{Appendices} %
\changelinkcolor{black}{}
\parttoc%
\newpage
\changelinkcolor{red}{}

\section{Causality}
\label{app:causality}

\subsection{Definitions and example}
\label{app:causality:defs}
As in previous causal works on DG~\citep{arjovsky2020invariant, krueger21rex, peters2016causal, christiansen2021causal, rojas2018invariant}, our causality results assume
all domains share the same underlying \textit{structural causal model}~(SCM)~\cite{pearl2009causality}, with different domains corresponding to different interventions. \looseness-1 For example, the different camera-trap deployments depicted in~\cref{fig:fig1:train-test} may induce changes in (or interventions on) equipment, lighting, and animal-species prevalence rates. 
\begin{definition}\upshape
An \emph{SCM}\footnote{A Non-parametric Structural Equation Model with Independent Errors (NP-SEM-IE) to be precise.
} $\calM =(\calS, \bbP_{N})$ consists of a collection of $d$ \emph{structural assignments}
\begin{equation}
\label{eq:scm-def}
    \calS=\{ X_j \gets g_j(\PA(X_j), N_j)\}_{j = 1}^d,
\end{equation}
where $\PA(X_j) \subseteq \{X_1, \dots, X_d\} \setminus \{X_j\}$ are the \emph{parents} or \emph{direct causes} of $X_j$, and $\bbP_{N}=\prod_{j=1}^d \bbP_{N_j}$, a joint distribution over the (jointly) independent noise variables $N_1, \dots, N_d$. An SCM $\calM$ induces a (``causal'') graph $\calG$ which is obtained by creating a node for each $X_j$ and then drawing a directed edge from each parent in $\PA(X_j)$ to $X_j$. We assume this graph to be acyclic.
\end{definition}

We can draw samples from the \emph{observational distribution} $\bbP_{\calM}(X)$ by first sampling a noise vector $n
\sim \bbP_{N}$, and then using the structural assignments to generate a data point $x
\sim \bbP_{\calM}(X)$, recursively computing the value of every node $X_j$ whose parents' values are known. We can also manipulate or \emph{intervene} upon the structural assignments of $\calM$ to obtain a related SCM $\calM^e$.

\begin{definition}\upshape
\looseness-1 An \emph{intervention} $e$ is a modification to one or more of the structural assignments of $\calM$, resulting in a new SCM $\calM^e=(\calS^e, \bbP^e_N)$ and (potentially) new graph $\calG^e$, with structural assignments
\begin{equation}
\label{eq:interv-def}
    \calS^e=\{ X^e_j \gets g^e_j(\PA^e(X^e_j), N^e_j)\}_{j = 1}^d.
\end{equation}%
\end{definition}%
We can draw samples from the \textit{intervention distribution} $\bbP_{\calM^e}(X^e)$ in a similar manner to before, now using the modified structural assignments. We can connect these ideas to DG by noting that each intervention $e$ creates a new domain or \emph{environment} $e$ with interventional distribution $\bbP(X^e, Y^e)$.

\begin{example}\label{ex:example}
Consider the following linear
SCM, with $N_j \sim \calN(0, \sigma^2_j)$:
\begin{align*}
\textstyle
    X_1     \gets              N_1, \qquad\qquad\qquad
    Y       \gets X_1        + N_Y, \qquad\qquad\qquad
    X_2     \gets Y        + N_2.
\end{align*}%
\end{example}%
Here, interventions could
replace the structural assignment of $X_1$ with $X^e_1 \gets 10$ and
change the noise variance of $X_2$, resulting in a set of training environments $\Etr = \{\text{fix}\ X_1\ \text{to}\ 10,\ \text{replace}\ \sigma_2\ \text{with}\ 10\}$.

\subsection{EQRM recovers the causal predictor}
\label{app:causality:discovery}
\paragraph{Overview.} We now prove that EQRM recovers the causal predictor in two stages. First, we prove the formal versions of \cref{prop:equalize_main}, i.e.\ that EQRM learns a minimal invariant-risk predictor as $\alpha \to 1$ when using the following estimators of $\bbT_{f}$: (i) a Gaussian estimator (\cref{prop:Gaussian_QRM_invariant} of \cref{app:causality:discovery:gaussian}); and (ii) kernel-density estimators with certain bandwidth-selection methods (\cref{prop:KDE_QRM_invariant} of \cref{app:causality:discovery:kde}). Second, we prove \cref{thm:causal_predictor}, i.e.\ that learning a minimal invariant-risk predictor is sufficient to recover the causal predictor under weaker assumptions than those of~\citet[Thm~2]{peters2016causal} and~\citet[Thm~1]{krueger21rex} (\cref{app:causal_recovery}). Throughout this section, we consider the ``population'' setting within each domain (i.e., $n \to \infty$); in general, with only finitely-many observations from each domain, only approximate versions of these results are possible.

\paragraph{Notation.} Given $m$ training risks $\{\calR^{e_1}(f), \dots, \calR^{e_m}(f) \}$ corresponding to the risks of a fixed predictor $f$ on $m$ training domains, let
\[\hat{\mu}_f=\frac{1}{m} \sum_{i=1}^m \calR^{e_i}(f)\]
denote the sample mean and
\[\hat{\sigma}^2_f=\frac{1}{m-1} \sum_{i=1}^m (\calR^{e_i}(f) - \hat{\mu}_f)^2\]
the sample variance of the risks of $f$.

\subsubsection{Gaussian estimator}
\label{app:causality:discovery:gaussian}
When using a Gaussian estimator for $\widehat{\bbT}_f$, we can rewrite the EQRM objective of~\eqref{eq:qrm2}
in terms of the standard-Normal inverse CDF $\Phi^{-1}$ as
\begin{align}\label{eq:qrm-gaussian}
    \hat f_\alpha := \argmin_{f\in\calF}\ \hat{\mu}_f + \Phi^{-1}(\cutoff) \cdot \hat{\sigma}_f.
\end{align}
Informally, we see that $\alpha\! \to\! 1\! \implies\! \Phi^{-1}(\alpha)\! \to\! \infty\! \implies\! \hat{\sigma}_f\! \to\! 0$. More formally, we now show that, as $\alpha \to 1$, minimizing~\eqref{eq:qrm-gaussian} leads to a predictor with minimal invariant-risk:
\begin{proposition}[Gaussian QRM learns a minimal invariant-risk predictor as $\alpha \to 1$]
    \label{prop:Gaussian_QRM_invariant}
    Assume
    \begin{enumerate}
        \item $\calF$ contains an invariant-risk predictor $f_0 \in \calF$ with finite mean risk (i.e., $\hat\sigma_{f_0} = 0$ and $\hat\mu_{f_0} < \infty$), and
        \item there are no arbitrarily negative mean risks (i.e., $\mu_* := \inf_{f \in \calF} \mu_f > -\infty$).
    \end{enumerate}
    Then, for the Gaussian QRM predictor $\hat f_\alpha$ given in Eq.~\eqref{eq:qrm-gaussian},
    \[\lim_{\alpha \to 1} \hat\sigma_{\hat f_\alpha} = 0 \quad \text{ and } \quad \limsup_{\alpha \to 1} \hat\mu_{\hat f_\alpha} \leq \hat\mu_{f_0}.\]
\end{proposition}
\Cref{prop:Gaussian_QRM_invariant} essentially states that, if an invariant-risk predictor exists, then Gaussian EQRM equalizes risks across the $m$ domains, to a value at most the risk of the invariant-risk predictor.
As we discuss in~\cref{app:causal_recovery},
an invariant-risk predictor $f_0$ (Assumption 1.~of ~\cref{prop:Gaussian_QRM_invariant} above) exists under the assumption that the mechanism generating the labels $Y$ does not change between domains and is contained in the hypothesis class $\calF$, together with a homoscedasticity assumption (see \cref{sec:additional_exps:linear_regr:risks_vs_functions}).
Meanwhile, Assumption 2.~of \cref{prop:Gaussian_QRM_invariant} above is quite mild and holds automatically for most loss functions used in supervised learning (e.g., squared loss, cross-entropy, hinge loss, etc.). We now prove \cref{prop:Gaussian_QRM_invariant}.
\begin{proof}
    By definitions of $\hat f_\alpha$ and $f_0$,
    \begin{equation}
        \hat{\mu}_{\hat f_\alpha} + \Phi^{-1}(\cutoff) \cdot \hat{\sigma}_{\hat f_\alpha}
        \leq \hat{\mu}_{f_0} + \Phi^{-1}(\cutoff) \cdot \hat{\sigma}_{f_0}
        = \hat{\mu}_{f_0}.
        \label{ineq:gaussian_basic_ineq}
    \end{equation}
    Since for $\alpha \geq 0.5$ we have that $\Phi^{-1}(\cutoff) \hat{\sigma}_{\hat f_\alpha} \geq 0$, it follows that $\hat{\mu}_{\hat f_\alpha} \leq \hat{\mu}_{f_0}$. Moreover, rearranging and using the definition of $\mu_*$, we obtain
    \[\hat{\sigma}_{\hat f_\alpha}
      \leq \frac{\hat{\mu}_{f_0} - \hat{\mu}_{\hat f_\alpha}}{\Phi^{-1}(\cutoff)}
      \leq \frac{\hat{\mu}_{f_0} - \mu_*}{\Phi^{-1}(\cutoff)}
      \to 0
      \quad \text{ as } \quad \alpha \to 1.\]
\end{proof}

\paragraph{Connection to VREx.} For the special case of using a Gaussian estimator for $\widehat{\bbT}_f$, we can equate the EQRM objective of \eqref{eq:qrm-gaussian} with the $\calR_{\VREx}$ objective of \cite[Eq.~8]{krueger21rex}. To do so, we rewrite $\calR_{\VREx}$ in terms of the sample mean and variance:
\begin{align}\label{eq:vrex-rewritten}
    \argmin_{f\in\calF}\ \calR_{\VREx}(f) = \argmin_{f\in\calF}\ m \cdot \hat{\mu}_f + \beta \cdot \hat{\sigma}^2_f.
\end{align}
Note that as $\beta \to \infty$, $\calR_{\VREx}$ learns a minimal invariant-risk predictor
under the same assumptions, and by the same argument, as~\cref{prop:Gaussian_QRM_invariant}. Dividing this objective by the positive constant $m>0$, we can rewrite it in a form that allows a direct comparison of our $\alpha$ parameter and this $\beta$ parameter:
\begin{align}\label{eq:vrex-rewritten-comparison}
    \argmin_{f\in\calF}\ \hat{\mu}_f + \left(\frac{\beta \cdot \hat{\sigma}_f}{m}\right) \cdot \hat{\sigma}_f.
\end{align}
Comparing \eqref{eq:vrex-rewritten-comparison} and \eqref{eq:qrm-gaussian}, we note the relation $\beta = m \cdot \Phi^{-1}(\cutoff) / \hat{\sigma}_f$ for a fixed $f$. For different $f$s, a particular setting of our parameter $\alpha$ corresponds to different settings of \citeauthor{krueger21rex}'s $\beta$ parameter, depending on the sample standard deviation over training risks $\hat{\sigma}_f$.

\subsubsection{Kernel density estimator}
\label{app:causality:discovery:kde}
We now consider the case of using a kernel density estimate, in particular,
\begin{equation}
    \hat F_{\text{KDE},f}(x)
      = \frac{1}{m} \sum_{i = 1}^m \Phi \left( \frac{x - R^{e_i}(f)}{h_{f}} \right)
    \label{eq:KDE}
\end{equation}
to estimate the cumulative risk distribution.
\begin{proposition}[Kernel EQRM learns a minimal risk-invariant predictor as $\alpha \to 1$]
    Let
    \[\hat f_{\alpha} := \argmin_{f \in \calF} \hat F^{-1}_{\text{KDE},f}(\alpha),\]
    be the kernel EQRM predictor, where $\hat F^{-1}_{\text{KDE},f}$ denotes the quantile function computed from the kernel density estimate over (empirical) risks of $f$ with a standard 
    Gaussian kernel. Suppose we use a data-dependent bandwidth $h_f$ such that $h_f \to 0$ implies $\hat\sigma_f \to 0$ (e.g., the ``Gaussian-optimal'' rule $h_f = (4/3m)^{0.2} \cdot \hat{\sigma}_f$~\citep{silverman1986density}). As in Proposition~\ref{prop:Gaussian_QRM_invariant}, suppose also that
    \begin{enumerate}
        \item $\calF$ contains an invariant-risk predictor $f_0 \in \calF$ with finite training risks (i.e., $\hat\sigma_{f_0} = 0$ and each $R^{e_i}(f_0) < \infty$), and
        \item there are no arbitrarily negative training risks (i.e., $R_* := \inf_{f \in \calF, i \in [m]} R^{e_i}(f) > -\infty$).
    \end{enumerate}
    For any $f \in \calF$, let $R_f^* := \min_{i \in [m]} R^{e_i}(f)$ denote the smallest of the (empirical) risks of $f$ across domains.
    Then,
    \[\lim_{\alpha \to 1} \hat\sigma_{\hat f_\alpha} = 0 \quad \text{ and } \quad \limsup_{\alpha \to 1} R_{\hat f_\alpha}^* \leq R_{f_0}^*.\]
    \label{prop:KDE_QRM_invariant}
\end{proposition}
As in \cref{prop:Gaussian_QRM_invariant}, Assumption 1 depends on invariance of the label-generating mechanism across domains (as discussed further in \cref{app:causal_recovery} below), while Assumption 2 automatically holds for most loss functions used in supervised learning. We now prove \cref{prop:KDE_QRM_invariant}.
\begin{proof}
    By our assumption on the choice of bandwidth, it suffices to show that, as $\alpha \to 1$, $h_{\hat f_{\alpha}} \to 0$.
    
    Let $\Phi$ denote the standard Gaussian CDF. Since $\Phi$ is non-decreasing, for all $x \in \bbR$,
    \[\hat F_{\text{KDE},\hat f_{\alpha}}(x)
      = \frac{1}{m} \sum_{i = 1}^m \Phi \left( \frac{x - R^{e_i}(\hat f_{\alpha})}{h_{\hat f_{\alpha}}} \right)
      \leq \Phi \left( \frac{x - R_{\hat f_{\alpha}}^*}{h_{\hat f_{\alpha}}} \right).\]
    In particular, for $x = \hat F^{-1}_{\text{KDE},\hat f_{\alpha}}(\alpha)$, we have
    \[\alpha = \hat F_{\text{KDE},\hat f_{\alpha}}(\hat F^{-1}_{\text{KDE},\hat f_{\alpha}}(\alpha))
      \leq \Phi \left( \frac{\hat F^{-1}_{\text{KDE},\hat f_{\alpha}}(\alpha) - R_f^*}{h_{\hat f_{\alpha}}} \right).\]
    Inverting $\Phi$ and rearranging gives
    \[R_f^* + h_{\hat f_{\alpha}} \cdot \Phi^{-1}(\alpha)
      \leq \hat F^{-1}_{\text{KDE},\hat f_{\alpha}}(\alpha).\]
    Hence, by definitions of $\hat f_{\alpha}$ and $f_0$,
    \begin{equation}
        R_f^* + h_{\hat f_{\alpha}} \cdot \Phi^{-1}(\alpha)
        \leq \hat F^{-1}_{\text{KDE},\hat f_{\alpha}}(\alpha)
        \leq \hat F^{-1}_{\text{KDE},f_0}(\alpha)
        = R_{f_0}^*.
        \label{ineq:KDE_basic_ineq}
    \end{equation}
    Since, for $\alpha \geq 0.5$ we have that $h_{\hat f_{\alpha}} \cdot \Phi^{-1}(\alpha) \geq 0$, it follows that $R_{\hat f_\alpha}^* \leq R_{f_0}^*$. Moreover, rearranging Inequality~\eqref{ineq:KDE_basic_ineq} and using the definition of $R_*$, we obtain
    \[h_{\hat f_{\alpha}}
     \leq \frac{R_{f_0}^* - R^*_{\hat f_{\alpha}}}{\Phi^{-1}(\alpha)}
     \leq \frac{R_{f_0}^* - R_*}{\Phi^{-1}(\alpha)}
     \to 0\]
     as $\alpha \to 1$.
\end{proof}

\subsubsection{Causal recovery}
\label{app:causal_recovery}
We now discuss and prove our main result, \cref{thm:causal_predictor}, regarding the conditions under which the causal predictor is the only minimal invariant-risk predictor. Together with \Cref{prop:Gaussian_QRM_invariant,prop:KDE_QRM_invariant}, this provides the conditions under which EQRM successfully performs ``causal recovery'', i.e., correctly recovers the true causal coefficients in a linear causal model of the data. As discussed in~\cref{sec:additional_exps:linear_regr:risks_vs_functions}, EQRM recovers the causal predictor by seeking \textit{invariant risks} across domains, which differs from seeking \textit{invariant functions} or coefficients (as in IRM~\citep{arjovsky2020invariant}). As we discuss below, \Cref{thm:causal_predictor} generalizes related results in the literature regarding causal recovery based on \textit{invariant risks}~\citep{krueger21rex,peters2016causal}.

\textbf{Assumption (v).} In contrast to both \citet{peters2016causal} and \citet{krueger21rex}, we do not require specific types of interventions on the covariates. In particular, our main assumption on the distributions of the covariates across domains, namely that the system of $d$-variate quadratic equations in~\eqref{eq:causal_recovery_equations_main} has a unique solution,
is more general than these comparable results. For example, whereas both \citet{peters2016causal} and \citet{krueger21rex} require one or more separate interventions for \emph{every} covariate $X_j$, Example 4 below shows that we only require interventions on the subset of covariates that are effects of $Y$, while weaker conditions suffice for other covariates. Although this generality comes at the cost of abstraction, we now provide some concrete examples with different types of interventions to aid understanding. Note that, to simplify calculations and provide a more intuitive form, \eqref{eq:causal_recovery_equations_main} of \cref{thm:causal_predictor} assumes, without loss of generality, that all covariates are standardized to have mean $0$ and variance $1$, except where interventions change these. We can, however, rewrite \eqref{eq:causal_recovery_equations_main} of \cref{thm:causal_predictor} in a slightly more general form which does not require this assumption of standardized covariates:
\begin{align}
    \notag
    0 \geq
    & x^\intercal \bbE_{X \sim e_1}[X X^\intercal] x
        + 2 x^\intercal \bbE_{N,X \sim e_1} \left[ N X \right] \\
    \notag
    = & \cdots \\
    \label{eq:causal_recovery_equations_app}
    = & x^\intercal \bbE_{X \sim e_m}[X X^\intercal] x
        + 2 x^\intercal \bbE_{N,X \sim e_m} \left[ N X \right].
\end{align}
We now present a number of concrete examples or special cases in which Assumption~(v) of \cref{thm:causal_predictor} would be satisfied, using this slightly more general form. In each example, we assume that variables are generated according to an SCM with an acyclic causal graph, as described in~\cref{app:causality:defs}.
\begin{enumerate}[itemsep=1em]
    \item \textit{No effects of $Y$.} In the case that there are no effects of $Y$ (i.e., no $X_j$ is a causal descendant of $Y$, and hence each $X_j$ is uncorrelated with $N$), it suffices for there to exists at least one environment $e_i$ in which the covariance $\operatorname{Cov}_{X \sim e}[X]$ has full rank. These are standard conditions for identifiability in linear regression. More generally, it suffices for $\sum_{i = 1}^m \operatorname{Cov}_{X \sim e_i}[X]$ to have full rank; this is the same condition one would require if simply performing linear regression on the pooled data from all $m$ environments. Intuitively, this full-rank condition guarantees that the observed covariate values are sufficiently uncorrelated to distinguish the effect of each covariate on the response $Y$. 
    However, it does not necessitate interventions on the covariates, which are necessary to identify the \emph{direction of causation} in a linear model; hence, this full-rank condition fails to imply causal recovery in the presence of effects of $Y$. See \cref{sec:additional_exps:linear_regr:risks_vs_functions} for a concrete example of this failure.

    \item \textit{\emph{Hard} interventions.} For each covariate $X_j$, compared to some baseline environment $e_0$, there is some environment $e_{X_j}$ arising from a hard single-node intervention $do(X_j = z)$, with $z \neq 0$. If $X_j$ is any leaf node in the causal DAG, then in $e_{X_j}$, $X_j$ is uncorrelated with $N$ and with each $X_k$ ($k \neq j$), so the inequality in~\eqref{eq:causal_recovery_equations_app} gives
    \begin{align*}
        0 \geq x^\intercal \bbE_{X \sim e_{X_j}}[X X^\intercal] x
          = x_j^2 z^2 + x_{-j}^\intercal \bbE_{X \sim e_0}[X X^\intercal] x_{-j}.
    \end{align*}
    Since the matrix $\bbE_{X \sim e}[X X^\intercal]$ is positive semidefinite (and $z \neq 0$ implies $z^2 > 0$), it follows that $x_j = 0$. The terms in~\eqref{eq:causal_recovery_equations_app} containing $x_j$ thus vanish, and iterating this argument for parents of leaf nodes in the causal DAG, and so on, gives $x = 0$. This condition is equivalent to that in Theorem 2(a) of \citet{peters2016causal} and is a strict improvement over Corollary~2 of \citet{yin2021optimization} and Theorem~1 of \citet{krueger21rex}, which respectively require two and three distinct hard interventions on each variable.
    
    \item \textit{Shift interventions.} For each covariate $X_j$, compared to some baseline environment $e_0$, there is some environment $e_{X_j}$ consisting of the shift intervention $X_j \leftarrow g_j(\PA(X_j), N_j) + z$, for some $z \neq 0$. Recalling that we assumed each covariate was centered (i.e., $\E_{X \sim e_0}[X_k] = 0$) in $e_0$, if $X_j$ is any leaf node in the causal DAG, then every other covariate remains centered in $e_{X_j}$ (i.e., $\E_{X \sim e_{X_j}}[X_k] = 0$ for each $k \neq j$). Hence, the excess risk is
    \[x^\intercal \bbE_{X \sim e_{X_j}}[X X^\intercal] x
            + 2 x^\intercal \bbE_{N,X \sim e_{X_j}} \left[ N X \right]
      = x_j^2 z^2 + x^\intercal \bbE_{X \sim e_0}[X X^\intercal] x + 2 x^\intercal \bbE_{N,X \sim e_0} \left[ N X \right].\]
    Since, by \eqref{eq:causal_recovery_equations_app},
    \[x^\intercal \bbE_{X \sim e_0}[X X^\intercal] x
            + 2 x^\intercal \bbE_{N,X \sim e_0} \left[ N X \right]
      = x^\intercal \bbE_{X \sim e_{X_j}}[X X^\intercal] x
            + 2 x^\intercal \bbE_{N,X \sim e_{X_j}} \left[ N X \right],\]
    it follows that $x_j^2z^2 = 0$, and so, since $z \neq 0$, $x_j = 0$. As above, the terms in~\eqref{eq:causal_recovery_equations_app} containing $x_j$ thus vanish, and iterating this argument for parents of leaf nodes in the causal DAG, and so on, gives $x = 0$.
    \looseness-1 This condition is equivalent to the additive setting of Theorem 2(b) of \citet{peters2016causal}.
    
    \item \textit{Noise interventions.} Suppose that each covariate is related to its causal parents through an additive noise model; i.e.,
    \[X_j = g_j(\PA(X_j)) + N_j,\]
    where
    $\E[N_j] = 0$ and $0 < \E[N^2] < \infty$. Theorem 2(b) of \citet{peters2016causal} considers ``noise'' interventions, of the form
    \[X_j \leftarrow g_j(\PA(X_j)) + \sigma N_j,\]
    where $\sigma^2 \neq 1$. Suppose that, for each covariate $X_j$, compared to some baseline environment $e_0$, there exists an environment $e_{X_j}$ consisting of the above noise intervention. If $X_j$ is any leaf node in the causal DAG, then, since we assumed $\bbE_{X \sim e_0}[X_j^2] = 1$,
    \begin{align*}
        & x^\intercal \bbE_{X \sim e_{X_j}}[X X^\intercal] x
            + 2 x^\intercal \bbE_{N,X \sim e_{X_j}} \left[ N X \right] \\
        & = (\sigma^2 - 1) x_j^2 \bbE[N_j^2]
          + x^\intercal \bbE_{X \sim e_0}[X X^\intercal] x
            + 2 x^\intercal \bbE_{N,X \sim e_0} \left[ N X \right].
    \end{align*}
    Hence, the system \eqref{eq:causal_recovery_equations_app} implies $0 = (\sigma^2 - 1) x_j^2 \bbE[N_j^2]$.
    Since $\sigma^2 \neq 1$ and $\bbE[N_j^2] > 0$, it follows that $x_j = 0$.

    \item \textit{Scale interventions.} For each covariate $X_j$, compared to some baseline environment $e_0$, there exist two environments $e_{X_j,i}$ ($i \in \{1,2\}$) consisting of scale interventions $X_j \leftarrow \sigma_i g_j(\PA(X_j), N_j)$, for some $\sigma_i \neq \pm 1$, with $\sigma_1 \neq \sigma_2$. If $X_j$ is any leaf node in the causal DAG, then, since we assumed $\bbE_{X \sim e_0}[X_j^2] = 1$,
    \begin{align*}
        & x^\intercal \bbE_{X \sim e_{X_j}}[X X^\intercal] x
            + 2 x^\intercal \bbE_{N,X \sim e_{X_j}} \left[ N X \right] \\
        & = (\sigma_i^2 - 1) x_j^2
          + 2 (\sigma_i - 1) x_j \bbE_{X \sim e_0}[X_j X_{-j}^\intercal] x_{-j}^\intercal
          + x^\intercal \bbE_{X \sim e_0}[X X^\intercal] x \\
        & + 2 (\sigma_i - 1) x_j \bbE_{N,X \sim e_0} \left[ X_j N \right]
          + 2 x^\intercal \bbE_{N,X \sim e_0} \left[ N X \right].
    \end{align*}
    Hence, the system \eqref{eq:causal_recovery_equations_app} implies
    \begin{align*}
        0 & = (\sigma_i^2 - 1) x_j^2
          + 2 (\sigma_i - 1) x_j \left( \bbE_{X \sim e_0}[X_j X_{-j}^\intercal] x_{-j}^\intercal
          + \bbE_{N,X \sim e_0} \left[ X_j N \right] \right).
    \end{align*}
    Since $\sigma_i^2 \neq 1$, if $x_j \neq 0$, then solving for $x_j$ gives
    \[x_j = -2 \frac{\bbE_{X \sim e_0}[X_j X_{-j}^\intercal] x_{-j}^\intercal
        + \bbE_{N,X \sim e_0} \left[ X_j N \right]}{\sigma_i + 1}.\]
    Since $\sigma_1 \neq \sigma_2$, this is possible only if $x_j = 0$.
    This provides an example where a single intervention per covariate would be insufficient to guarantee causal recovery, but two distinct interventions per covariate suffice.
    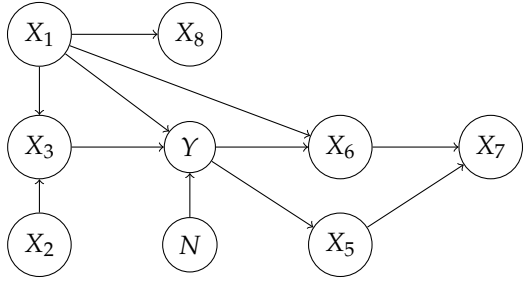
\begin{figure}[tb]
    \centering
    \begin{tikzpicture}
        \node[shape=circle,draw=black] (X1) at (0,1.5) {$X_1$};
        \node[shape=circle,draw=black] (X8) at (2,1.5) {$X_8$};
        \node[shape=circle,draw=black] (X2) at (0,-1.3) {$X_2$};
        \node[shape=circle,draw=black] (X3) at (0,0) {$X_3$};
        \node[shape=circle,draw=black] (N) at (2,-1.3) {$N$};
        \node[shape=circle,draw=black] (Y) at (2,0) {$Y$};
        \node[shape=circle,draw=black] (X5) at (4,-1.3) {$X_5$};
        \node[shape=circle,draw=black] (X6) at (4,0) {$X_6$};
        \node[shape=circle,draw=black] (X7) at (6,0) {$X_7$};
        \path [->](X1) edge node[below] {} (X8);
        \path [->](X1) edge node[below] {} (X3);
        \path [->](X1) edge node[below] {} (Y);
        \path [->](X1) edge node[below] {} (X6);
        \path [->](X2) edge node[left] {} (X3);
        \path [->](X3) edge node[below] {} (Y);
        \path [->](N) edge node[left] {} (Y);
        \path [->](Y) edge node[left] {} (X5);
        \path [->](Y) edge node[left] {} (X6);
        \path [->](X5) edge node[left] {} (X7);
        \path [->](X6) edge node[left] {} (X7);
    \end{tikzpicture}
    \caption{Example causal DAG with various types of covariates. $X_1$ and $X_3$ are the parents of $Y$, and so the true causal coefficient $\beta$ has only two non-zero coordinates $\beta_1$ and $\beta_3$. $X_1$, $X_2$, and $X_3$ are ancestors of $Y$. $X_5$, $X_6$, and $X_7$ are effects, or descendants, of $Y$ and are the only covariates for which $\E[X_j N]$ can be nonzero; hence, $X_5$, $X_6$, and $X_7$ are the only covariates on which interventions are generally necessary.}
    \label{fig:causal_recovery_example}
\end{figure}

    \item \textit{Sufficiently uncorrelated causes and intervened-upon effects.} Suppose that, within the true causal DAG, $\DE(Y) \subseteq [d]$ indexes the \emph{descendants}, or \emph{effects} of $Y$ (e.g., in Figure~\ref{fig:causal_recovery_example}, $\DE(Y) = \{5,6,7\}$). Suppose that for every $j \in \DE(Y)$, compared to a single baseline environment $e_0$, there is a environment $e_{X_j}$ consisting of either a $do(X_j = z)$ intervention or a shift intervention $X_j \leftarrow g_j(\PA(X_j), N_j) + z$, with $z \neq 0$ and that the matrix
    \begin{equation}
        \sum_{i = 1}^m \operatorname{Cov}_{X \sim e_i} \left[ X_{[d]\backslash \DE(Y)} \right]
        \label{exp:ancestor_full_rank_condition}
    \end{equation}
    has full rank. Then, as argued in the previous two cases, for each $j \in \DE(Y)$, $x_j = 0$. Moreover, for any $j \in [d] \backslash \DE(Y)$, $\E[X_j N] = 0$, and so the system of equations~\eqref{eq:causal_recovery_equations_app} reduces to
    \begin{align*}
        0
        & \geq x_{[d]\backslash \DE(Y)}^\intercal \E_{X \sim e_1} \left[ X_{[d]\backslash \DE(Y)} X_{[d]\backslash \DE(Y)}^\intercal \right] x_{[d]\backslash \DE(Y)}^\intercal \\
        & = \cdots \\
        & = x_{[d]\backslash \DE(Y)}^\intercal \E_{X \sim e_m} \left[ X_{[d]\backslash \DE(Y)} X_{[d]\backslash \DE(Y)}^\intercal \right] x_{[d]\backslash \DE(Y)}^\intercal.
    \end{align*}
    Since each $\E_{X \sim e_m} \left[ X_{[d]\backslash \DE(Y)} X_{[d]\backslash \DE(Y)}^\intercal \right]$ is positive semidefinite, the solution $x = 0$ to this reduced system of equations is unique if (and only if) the matrix~\eqref{exp:ancestor_full_rank_condition} has full rank.
    This example demonstrates that interventions are only needed for effect covariates, while a weaker full-rank condition suffices for the remaining ones. In many practical settings, it may be possible to determine \textit{a priori} that a particular covariate $X_j$ is not a descendant of $Y$; in this case, the practitioner need not intervene on $X_j$, as long as sufficiently diverse observational data on $X_j$ is available.
    To the best of our knowledge, this does not follow from any existing results in the literature, such as Theorem 2 of \citet{peters2016causal} or Corollary~2 of \cite{yin2021optimization}. 

\end{enumerate}

We conclude this section with the proof of \Cref{thm:causal_predictor}:
\begin{proof}
    Under the linear SEM setting with squared-error loss, for any estimator $\hat\beta$,
    \begin{align*}
        \calR^e(\hat\beta)
        & = \bbE_{N, X \sim e} \left[ \left( (\beta - \hat \beta)^\intercal X + N \right)^2 \right] \\
        & = \bbE_{X \sim e} \left[ \left( (\beta - \hat \beta)^\intercal X \right)^2 \right] + 2\bbE_{N,X \sim e} \left[ (\beta - \hat \beta)^\intercal N X \right] + \bbE_N \left[ N^2 \right]. \\
    \end{align*}
    Since the second moment of the noise term $\bbE_N[N^2]$ is equal to the risk $\E_{(X,Y) \sim e}[(Y\! -\! \beta^T X)^2]$ of the causal predictor $\beta$, by the definition of $Y = \beta^T X + N$, we have that $\bbE_N[N^2]$ is invariant across environments.
    Thus, minimizing the squared error risk $\calR^e(\hat\beta)$ is equivalent to minimizing the excess risk
    \begin{align*}
        & \bbE_{X \sim e} \left[ \left( (\beta - \hat \beta)^\intercal X \right)^2 \right]
            + 2\bbE_{N,X \sim e} \left[ (\beta - \hat \beta)^\intercal N X \right] \\
        & = (\beta - \hat \beta)^\intercal \bbE_{X \sim e}[X X^\intercal] (\beta - \hat \beta)
            + 2(\beta - \hat \beta)^\intercal \bbE_{N,X \sim e} \left[ N X \right]
    \end{align*}
    over estimators $\hat\beta$.
    Since the true coefficient $\beta$ is an invariant-risk predictor with $0$ excess risk, if $\hat\beta$ is a minimal invariant-risk predictor, it has at most $0$ invariant excess risk; i.e.,
    \begin{align}
        \notag
        0 \geq
        & (\beta - \hat \beta)^\intercal \bbE_{X \sim e_1}[X X^\intercal] (\beta - \hat \beta)
            + 2(\beta - \hat \beta)^\intercal \bbE_{N,X \sim e_1} \left[ N X \right] \\
        \notag
        = & \cdots \\
        \label{eq:equal_squared_error_risks}
        = & (\beta - \hat \beta)^\intercal \bbE_{X \sim e_m}[X X^\intercal] (\beta - \hat \beta)
            + 2(\beta - \hat \beta)^\intercal \bbE_{N,X \sim e_m} \left[ N X \right].
    \end{align}
    By Assumption (v), the unique solution to this is $\beta - \hat\beta = 0$; i.e., $\hat\beta = \beta$.
\end{proof}

\section{On the equivalence of different DG formulations} 
\label{app:sup-and-esssup}

In Section~\ref{sec:qrm}, we claimed that under mild conditions, the minimax domain generalization problem in~\eqref{eq:domain-gen} is equivalent to the essential supremum problem in~\eqref{eq:domain-gen-rewritten}.
In this subsection, we formally describe the conditions under which these two problems are equivalent.  We also highlight several examples in which the assumptions needed to prove this equivalence hold.  

Specifically, this appendix is organized as follows.  First, in \S~\ref{sec:formal-connections} we offer a more formal analysis of the equivalence between the probable domain general problems in~\eqref{eq:prob_gen} and~\eqref{eq:qrm}.  Next, in \S~\ref{sec:sup-and-esssup-connections}, we connect the domain generalization problem in~\eqref{eq:domain-gen} to the essential supremum problem in~\eqref{eq:domain-gen-rewritten}.

\subsection{Connecting formulations for QRM via a push-forward measure} \label{sec:formal-connections}

To begin, we consider the abstract measure space $(\Eall, \calA, \bbQ)$, where $\calA$ is a $\sigma$-algebra defined on the subsets of $\Eall$.  Recall that in our setting, the domains $e\in\Eall$ are assumed to be drawn from the distribution $\bbQ$.  Given this setting, in \S~\ref{sec:qrm} we introduced the probable domain generalization problem in~\eqref{eq:prob_gen}, which we rewrite below for convenience:
\begin{alignat}{2}
    &\min_{f\in\calF,\, \thres \in \bbR} &&\thres  \qquad \subjto \Pr_{e\sim\bbQ} \left\{\calR^e(f) \leq t \right\} \geq \alpha.
\end{alignat}
Our objective is to formally show that this problem is equivalent to~\eqref{eq:qrm}.  To do so, for each $f\in\calF$, let consider a second measurable space $(\R_+, \calB)$, where $\R_+$ denotes the set of non-negative real numbers and $\calB$ denotes the Borel $\sigma$-algebra over this space.  For each $f\in\calF$, we can now define the $(\R_+, \calB)$-valued random variable\footnote{For brevity, we will assume that $G_f$ is always measurable with respect to the underlying $\sigma$-algebra $\calA$.} $G_f:\Eall\to\R_+$ via
\begin{align}
    G_f : e \mapsto \calR^e(f) = \E_{\Prob(X^e,Y^e)}[\ell(f(X^e),Y^e)].
\end{align}
Concretely, $G_f$ maps an domain $e$ to the corresponding risk $\calR^e(f)$ of $f$ in that domain.  In this way, $G_f$ effectively summarizes $e$ by its effect on our predictor's risk, thus projecting from the often-unknown and potentially high-dimensional space of possible distribution shifts or interventions to the one-dimensional space of observed, real-valued risks.  However, note that $G_f$ is not necessarily injective, meaning that two domains $e_1$ and $e_2$ may be mapped to the same risk value under $G_f$.

The utility of defining $G_f$ is that it allows us to formally connect~\eqref{eq:prob_gen} with~\eqref{eq:qrm} via a push-forward measure through $G_f$.  That is, given any $f\in\calF$, we can define the measure\footnote{Here $\bbT_f$ is defined over the induced measurable space $(\R_+, \calB)$.}
\begin{align}
    \bbT_f =^d G_f\:\#\: \bbQ \label{eq:def-of-risk-dist}
\end{align}
where $\#$ denotes the \emph{push-forward} operation and $=^d$ denotes equality in distribution.  Observe that the relationship in~\eqref{eq:def-of-risk-dist} allows us to explicitly connect $\bbQ$---the often unknown distribution over (potentially high-dimensional and/or non-Euclidean) domain shifts in \cref{fig:fig1:q-dist}---to $\bbT_f$---the distribution over real-valued risks in \cref{fig:fig1:risk}, from which we can directly observe samples.  In this way, we find that for each $f\in\calF$,
\begin{align}
    \Pr_{e\sim\bbQ} \{\calR^e(f) \leq t \} = \Pr_{R\sim\bbT_f} \{R \leq t\}.
\end{align}
This relationship lays bare the connection between~\eqref{eq:prob_gen} and~\eqref{eq:qrm}, in that the domain or environment distribution $\bbQ$ can be replaced by a distribution over risks $\bbT_f$.

\subsection{Connecting~\texorpdfstring{\eqref{eq:domain-gen}}{DG} to the essential supremum problem~\texorpdfstring{\eqref{eq:domain-gen-rewritten}}{3.1}} \label{sec:sup-and-esssup-connections}

We now study the relationship between~\eqref{eq:domain-gen} and~\eqref{eq:domain-gen-rewritten}. In particular, in \S~\ref{sec:cont-domain-sets} and \S~\ref{sec:disc-domain-set}, we consider the distinct settings wherein $\Eall$ comprises continuous and discrete spaces respectively. 

\subsubsection{Continuous domain sets \texorpdfstring{$\Eall$}{E}} \label{sec:cont-domain-sets}

When $\Eall$ is a continuous space, it can be shown that~\eqref{eq:domain-gen} and~\eqref{eq:domain-gen-rewritten} are \emph{equivalent} whenever: (a) the map $G_f$ defined in Section~\ref{sec:formal-connections} is continuous; and (b) the measure $\bbQ$ satisfies very mild regularity conditions.  

\paragraph{The case when $\bbQ$ is the Lebesgue measure.}  Our first result concerns the setting in which $\Eall$ is a subset of Euclidean space and where $\bbQ$ is chosen to be the Lebesgue measure on $\Eall$.  We summarize this result in the following proposition.

\begin{proposition} \label{prop:cont-equivalent-of-dg}
Let us assume that the map $G_f$ is continuous for each $f\in\calF$.  Further, let $\bbQ$ denote the Lebesgue measure over $\Eall$; that is, we assume that domains are drawn uniformly at random from $\Eall$. Then~\eqref{eq:domain-gen} and~\eqref{eq:domain-gen-rewritten} are equivalent.
\end{proposition}

\begin{proof}
To prove this claim, it suffices to show that under the assumptions in the statement of the proposition, it holds for any $f\in\calF$ that 
\begin{align}
    \sup_{e\in\Eall} R^e(f) = \esssup_{e\sim\bbQ} R^e(f). \label{eq:equality-of-sup-and-esssup}
\end{align} 
To do so, let us fix an arbitrary $f\in\calF$ and write
\begin{align}
    A := \sup_{e\in\Eall} R^e(f) \quad\text{and}\quad B := \esssup_{e\sim\bbQ} R^e(f).
\end{align}
At a high-level, our approach is to show that $B \leq A$, and then that $A\leq B$, which together will imply the result in~\eqref{eq:equality-of-sup-and-esssup}.  To prove the first inequality, observe that by the definition of the supremum, it holds that $R^e(f) \leq A$ $\forall e\in\Eall$.  Therefore, $\bbQ\{e\in\Eall : R^e(f) > A\} = 0$, which directly implies that $B\leq A$.  Now for the second inequality, let $\epsilon>0$ be arbitrarily chosen.  Consider that due to the continuity of $G_f$, there exists an $e_0\in\Eall$ such that
\begin{align}
    R^{e_0}(f) +\epsilon > A. \label{eq:apply-def-of-sup}
\end{align}
Now again due to the continuity of $G_f$, we can choose a ball $\calB_\epsilon\subset\Eall$ centered at $e_0$ such that $|R^e(f)-R^{e_0}(f)| \leq \epsilon$ $\forall e\in\calB_\epsilon$.  Given such a ball, observe that $\forall e\in\calB_\epsilon$, it holds that
\begin{align}
    R^e(f) \geq R^{e_0}(f) - \epsilon > A - 2\epsilon 
\end{align}
where the first inequality follows from the reverse triangle inequality and the second inequality follows from~\eqref{eq:apply-def-of-sup}.  Because $\bbQ\{e\in\calB_\epsilon : R^e(f) > A-2\epsilon\} > 0$, it directly follows that $A-2\epsilon \leq B$.  As $\epsilon>0$ was chosen arbitrarily, this inequality holds for any $\epsilon>0$, and thus we can conclude that $A\leq B$, completing the proof.
\end{proof}

\paragraph{Generalizing Prop.\  \ref{prop:cont-equivalent-of-dg} to other measure $\bbQ$.} We note that this proof can be generalized to measures $\bbQ$ other than the Lebesgue measure.  Indeed, the result holds for any measure $\bbQ$ taking support on $\Eall$ for which it holds that $\bbQ$ places non-zero probability mass on any closed subset of $\Eall$.  This would be the case, for instance, if $\bbQ$ was a truncated Gaussian distribution with support on $\Eall$.  Furthermore, if we let $\bbL$ denote the Lebesgue measure on $\Eall$, then another more general instance of this property occurs whenever $\bbL$ is absolutely continuous with respect to $\bbQ$, i.e., whenever $\bbL \ll \bbQ$.

\begin{corollary}
Let us assume that $\bbQ$ places nonzero mass on every open ball with radius strictly larger than zero.  Then under the continuity assumptions of Prop.~\ref{prop:cont-equivalent-of-dg}, it holds that~\eqref{eq:domain-gen} and~\eqref{eq:domain-gen-rewritten} are equivalent.
\end{corollary}

\begin{proof}
The proof of this fact follows along the same lines as that of Prop.~\ref{prop:cont-equivalent-of-dg}.  In particular, the same argument shows that $B\leq A$.  Similarly, to show that $A\leq B$, we can use the same argument, noting that $\bbQ\{e\in\calB_\epsilon : R^e(f) > A-2\epsilon\}$ continues to hold, due to our assumption that $\bbQ$ places nonzero mass on $\calB_\epsilon$.
\end{proof}

\paragraph{Examples.}  We close this subsection by considering several real-world examples in which the conditions of Prop.~\ref{prop:cont-equivalent-of-dg} hold.  In particular, we focus on examples in the spirit of ``Model-Based Domain Generalization''~\cite{robey2021modelbased}.  In this setting, it is assumed that the variation from domain to domain is parameterized by a \emph{domain transformation model} $x^e\mapsto D(x^e,e') =: x^{e'}$, which maps the covariates $x^e$ from a given domain $e\in\Eall$ to another domain $e'\in\Eall$.  As discussed in~\cite{robey2021modelbased}, domain transformation models cover settings in which inter-domain variation is due to \emph{domain shift}~\cite[\S 1.8]{quinonero2008dataset}.  Indeed, under this model (formally captured by Assumptions 4.1 and 4.2 in~\cite{robey2021modelbased}), the domain generalization problem in~\eqref{eq:domain-gen} can be equivalently rewritten as
\begin{align}
    \min_{f\in\calF} \max_{e\in\Eall} \E_{(X,Y)} [\ell(f(D(X,e)),Y)]. \label{eq:model-based-dg}
\end{align}
For details, see Prop.\ 4.3 in~\cite{robey2021modelbased}.  In this problem, $(X,Y)$ denote an underlying pair of random variables such that 
\begin{align}
    \Prob(X^e) =^d D \:\#\: (\Prob(X), \delta(e)) \quad\text{and}\quad \Prob(Y^e) =^d \Prob(Y) 
\end{align}

for each $e\in\Eall$ where $\delta(e)$ is a Dirac measure placed at $e\in\Eall$.  Now, turning our attention back to Prop.\ ~\ref{prop:cont-equivalent-of-dg}, we can show the following result for~\eqref{eq:model-based-dg}.

\begin{remark} \label{rmk:model-based}
Let us assume that the map $e\mapsto D(\cdot, e)$ is continuous with respect to a metric $d_{\Eall}(e,e')$ on $\Eall$ and that $x\mapsto \ell(x,\cdot)$ is continuous with respect to the absolute value.  Further, assume that each predictor $f\in\calF$ is continuous in the standard Euclidean metric on $\R^d$.  Then~\eqref{eq:domain-gen} and~\eqref{eq:domain-gen-rewritten} are equivalent.
\end{remark}

\begin{proof}
By Prop.~\ref{prop:cont-equivalent-of-dg}, it suffices to show that the map
\begin{align}
    G_f : e \mapsto \E_{(X,Y)} [\ell(f(D(X,e)),Y)]
\end{align}
is a continuous function.  To do so, recall that the composition of continuous functions is continuous, and therefore we have, by the assumptions listed in the above remark, that the map $e\mapsto \ell(f(D(x,e)),y)$ is continuous for each $(x,y)\sim(X,Y)$.  To this end, let us define the function $h_f(x,y,e) := \ell(f(D(x,e)),y)$ and let $\epsilon>0$.  By the continuity of $h_f$ in $e$, there exists a $\delta=\delta(\epsilon) > 0$ such that $| h_f(x,y,e) - h_f(x,y,e')|<\epsilon$ whenever $d_{\Eall}(e,e') < \delta$.  Now observe that
\begin{align}
    &\left|\E_{(X,Y)} [h_f(X,Y,e)] - \E_{(X,Y)} [h_f(X,Y,e')]  \right| \\
    &\qquad = \left| \int_{\Eall} h_f(X,Y,e) \text{d}\Prob(X,Y) - \int_{\Eall} h_f(X,Y,e') \text{d}\Prob(X,Y) \right| \\
    &\qquad= \left| \int_{\Eall} (h_f(X,Y,e) - h_f(X,Y,e')) \text{d}\Prob(X,Y) \right| \\
    &\qquad\leq \int_{\Eall} \left| h_f(X,Y,e) - h_f(X,Y,e') \right| \text{d}\Prob(X,Y). \label{eq:last-inequal-continuity}
\end{align}
Therefore, whenever $d_{\Eall}(e,e')<\delta$ it holds that
\begin{align}
    \left|\E_{(X,Y)} [h_f(X,Y,e)] - \E_{(X,Y)} [h_f(X,Y,e')]  \right| \leq \int_{\Eall} \epsilon\text{d}\Prob(X,Y) = \epsilon
\end{align} 
by the monotonicity of expectation.  This completes the proof that $G_f$ is continuous.
\end{proof}

In this way, provided that the risks in each domain vary in a continuous way through $e$,~\eqref{eq:domain-gen} and~\eqref{eq:domain-gen-rewritten} are equivalent.  As a concrete example, consider an image classification setting in which the variation from domain to domain corresponds to different rotations of the images.  This is the case, for instance, in the \texttt{RotatedMNIST} dataset~\cite{ilse2020diva,gulrajani2020search}, wherein the training domains correspond to different rotations of the \texttt{MNIST} digits.  Here, a domain transformation model $D$ can be defined by
\begin{align}
    D(x, e) = R(e)x \quad\text{where}\quad e\in\Eall \subseteq [0, 2\pi),
\end{align}
and where $R(e)$ is a rotation matrix.  In this case, it is clear that $D$ is a continuous function of $e$ (in fact, the map is \emph{linear}), and therefore the result in~\eqref{rmk:model-based} holds.

\subsubsection{Discrete domain sets \texorpdfstring{$\Eall$}{E}} \label{sec:disc-domain-set}

When $\Eall$ is a discrete set, the conditions we require for~\eqref{eq:domain-gen} and~\eqref{eq:domain-gen-rewritten} to be equivalent are even milder.  In particular, the only restriction we place on the problems is that $\bbQ$ must place non-zero mass on each element of $\Eall$; that is, $\bbQ(e) > 0$ $\forall e\in\Eall$.  We state this more formally below.

\begin{proposition} 
Let us assume that $\Eall$ is discrete, and that $\bbQ$ is such that $\forall e\in\Eall$, it holds that $\bbQ(e) > 0$.  Then it holds that~\eqref{eq:domain-gen} and~\eqref{eq:domain-gen-rewritten} are equivalent.
\end{proposition}

\section{Notes on KDE bandwidth selection}
\label{app:kde}
In our setting, we are interested in bandwidth-selection methods which: (i) work well for 1D distributions and small sample sizes $m$; and (ii) guarantee recovery of the causal predictor as $\alpha \to 1$ by satisfying $h_f \to 0 \implies \hat{\sigma}_f \to 0$, where $h_f$ is the data-dependent bandwidth and $\hat{\sigma}_f$ is the sample standard deviation (see \cref{app:causality:discovery:kde,app:causal_recovery}). We thus investigated three popular bandwidth-selection methods: (1) the Gaussian-optimal rule~\citep{silverman1986density}, $h_f = (4/3m)^{0.2} \cdot \hat{\sigma}_f$; (2) Silverman's rule-of-thumb~\citep{silverman1986density}, $h_f = m^{-0.2} \cdot \min (\hat{\sigma}_f, \frac{\text{IQR}}{1.34})$, with IQR the interquartile range; and (3) the median-heuristic~\citep{scholkopf1997kernel, takeuchi2006nonparametric, sriperumbudur2009kernel}, which sets the bandwidth to be the median pairwise-distance between data points. Note that many sensible methods exist, as do more complete studies on bandwidth selection---see e.g.~\cite{silverman1986density}.

For (i), we found Silverman's rule-of-thumb~\citep{silverman1986density} to perform very well, the Gaussian-optimal rule~\citep{silverman1986density} to perform well, and the median-heuristic~\citep{scholkopf1997kernel, takeuchi2006nonparametric, sriperumbudur2009kernel} to perform poorly. For (ii), only the Gaussian-optimal rule satisfies $h_f \to 0 \implies \hat{\sigma}_f \to 0$. Thus, in practice, we use either the Gaussian-optimal rule (particularly when causal predictor's are sought as $\alpha \to 1$), or Silverman's rule-of-thumb.

\section{Generalization bounds}
\label{app:gen_bounds}

This appendix states and proves our main generalization bound, Theorem~\ref{thm:generalization}. Theorem~\ref{thm:generalization} applies for many possible estimates $\widehat{\bbT}_f$, and we further show how to apply Theorem~\ref{thm:generalization} to the specific case of using a kernel density estimate.

\subsection{Main generalization bound and proof}

Suppose that, from each of $N$ IID environments $e_1,...,e_N \sim \bbP(e)$, we observe $n$ IID labeled samples $(X_{i,1},Y_{i,1}),...,(X_{n,1},Y_{n,1}) \sim \bbP(X^e, Y^e)$. Fix a hypothesis class $\calF$ and confidence level $\alpha \in [0, 1]$. For any hypothesis $f : \calX \to \calY$, define the
\emph{empirical risk on environment $e_i$} by
\[\widehat \calR^{e_i}(f)
  := \frac{1}{n} \sum_{j = 1}^n \ell \left( Y_{i,j}, f(X_{i,j}) \right),
  \quad \text{ for each } \quad
  i \in [N].\]
Throughout this section, we will abbreviate the distribution $F_{\bbT_f}(t) = \Pr_e[\calR^e(f) \leq t]$ of $f$'s risk by $F_f(t)$ and its estimate
$F_{\widehat{\bbT}_f}$, computed from the observed empirical risks $\widehat \calR^{e_1}(f),...,\widehat \calR^{e_N}(f)$, by $\widehat F_f$.

We propose to select a hypothesis by minimizing this over our hypothesis class:
\begin{equation}
    \widehat f := \argmin_{f \in \calF} F^{-1}_{\widehat\bbT_f}(\alpha).
    \label{estimator:empirical_VaR_min}
\end{equation}
In this section, we prove a uniform generalization bound, which in particular, provides conditions under which the estimator \eqref{estimator:empirical_VaR_min} generalizes uniformly over $\calF$. Because the novel aspect of the present paper is the notion of generalizing \emph{across} environments, we will take for granted that the hypothesis class $\calF$ generalizes uniformly \emph{within} each environments (i.e., that each $\sup_{f \in \calF} \calR^{e_i}(f) - \widehat{\calR}^{e_i}(f)$ can be bounded with high probability); myriad generalization bounds from learning theory can be used to show this.

\begin{theorem}
    Let $\calG := \{ \widehat F(\calR^{e_1}(f),\calR^{e_2}(f),...,\calR^{e_N}(f)) : f \in \calF, e_1,...,e_n \in \Eall \}$ denote the class of possible estimated risk distributions over $N$ environments, and, for any $\epsilon > 0$, let $\calN_\epsilon(\calG)$ denote the $\epsilon$-covering number of $\calG$ under $\calL_\infty(\bbR)$.
    Suppose the class $\calF$ generalizes uniformly within environments; i.e., for any $\delta > 0$, there exists $t_{n,\delta,\calF}$ such that
    \[\esssup_e \Pr_{\{(X_j,Y_j)\}_{j = 1}^n \sim \bbP(X^e, Y^e)} \left[ \sup_{f \in \calF} \rfe - \widehat \calR^e(f) > t_{n,\delta,\calF} \right] \leq \delta.\]
    Let
    \[\operatorname{Bias}(\calF, \widehat F)
      := \sup_{f \in \calF, t \in \bbR} F_f(t) - \E_{e_1,...,e_N}[\widehat F_f(t)]\]
    denote the worst-case bias of the estimator $\widehat F$ over the class $f$.
    Noting that $\widehat F_f$ is a function of the empirical risk CDF
    \[\widehat Q_f(t) := \frac{1}{N} \sum_{i = 1}^N 1\{\calR^{e_i}(f) \leq t\},\]
    suppose that the function $\widehat Q_f \mapsto \widehat F_f$ is $L$-Lipschitz under $\calL_\infty(\bbR)$.
    Then, for any $\epsilon,\delta > 0$,
    \begin{equation}
        \Pr_{\substack{e_1,...,e_N\\\{(X_j,Y_j)\}_{j = 1}^n \sim \bbP(X^{e_i}, Y^{e_i})}} \left[ \sup_{f \in \calF} F^{-1}_f\left(\alpha - B(\calF, \widehat F) - \epsilon \right) - \widehat F^{-1}_f(\alpha) > t_{n,\frac{\delta}{N},\calF}
        \right]
        \leq \delta + 8 \mathcal{N}_{\epsilon/16}(\calG) e^{-\frac{N\epsilon^2}{64L}}.
        \label{ineq:main_generalization_bound}
    \end{equation}
    \label{thm:generalization}
\end{theorem}
The key technical observation of Theorem~\ref{thm:generalization} is that we can pull the supremum over $\calF$ outside the probability by incurring a $\calN_{\epsilon/16}(\calG)$ factor increase in the probability of failure. To ensure $\calN_{\epsilon/16}(\calG) < \infty$, we need to limit the space of possible empirical risk profiles $\calG$ (e.g., by kernel smoothing), incurring an additional bias term $B(\calF, \widehat F)$. As we demonstrate later, for common distribution estimators, such as kernel density estimators, one can bound the covering number $\calN_{\epsilon/16}(\calG)$ in Inequality~\eqref{ineq:main_generalization_bound} by standard methods, and the Lipschitz constant $L$ is typically $1$. Under mild (e.g., smoothness) assumptions on the family of possible true risk profiles, one can additionally bound the Bias Term, again by standard arguments.

Before proving Theorem~\ref{thm:generalization}, we state two standard lemmas used in the proof:
\begin{lemma}[Symmetrization; Lemma 2 of \citep{bousquet2003introduction}]
    Let $X$ and $X'$ be independent realizations of a random variable with respect to which $\mathcal{F}$ is a family of integrable functions. Then, for any $\epsilon > 0$,
    \[\Pr \left[ \sup_{f \in \calF} f(X) - \E f(X) > \epsilon \right]
        \leq 2\Pr \left[ \sup_{f \in \calF} f(X) - f(X') > \frac{\epsilon}{2} \right].\]
    \label{lemma:symmetrization}
\end{lemma}

\begin{lemma}[Dvoretzky–Kiefer–Wolfowitz (DKW) Inequality; Corollary 1 of \citep{massart1990tight}]
    Let $X_1,...,X_n$ be IID $\bbR$-valued random variables with CDF $P$. Then, for any $\epsilon > 0$,
    \[\Pr \left[ \sup_{t \in \bbR} \left| F_f(t) - \frac{1}{n} \sum_{i = 1}^n 1\{X_i \leq t\} \right| > \epsilon \right] \leq 2e^{-2n\epsilon^2}.\]
    \label{lemma:DKW}
\end{lemma}

We now prove our main result, Theorem~\ref{thm:generalization}.
\begin{proof}[Proof of Theorem~\ref{thm:generalization}]
    For convenience, let $F_f(t) := \bbP_{e \sim \bbP(e)}[\rfe \leq t]$.
    In preparation for Symmetrization, for any $f \in \calF$, let $\widehat F_f'$ denote $\widehat F_f$ computed on an independent ``ghost'' sample $e_1',...,e_N' \sim \bbP(e)$. Let $P_{\epsilon/16} \subseteq \mathcal{G}$ denote an $(\epsilon/16)$-cover of $\mathcal{G}$ with $|P_{\epsilon/16}| = \mathcal{N}_{\epsilon/16}$. For any $F \in \mathcal{G}$, let $DF \in \argmin_{G \in P_{\epsilon/16}} \|G - F\|_\infty$ denote any projection of $F$ onto $P_{\epsilon/16}$. Let $\hat Q_f$ denote the empirical CDF, as defined in Theorem~\ref{thm:generalization}. Then,
    \begin{align}
        & \Pr_{e_1,...,e_N} \left[ \sup_{f \in \calF, t \in \bbR} \E_{e_1,...,e_N} \left[ \widehat F_f(t) \right] - \widehat F_f(t) > \epsilon \right] \\
        \label{l1}
        & \leq 2 \Pr_{\substack{e_1,...,e_N\\e_1',...,e_N'}} \left[ \sup_{f \in \calF, t \in \R} \widehat F_f'(t) - \widehat F_f(t) > \epsilon/2 \right] \\
        \label{l2}
        & \leq 2 \Pr_{\substack{e_1,...,e_N\\e_1',...,e_N'}} \left[ \sup_{f \in \calF} \left\| \widehat F_f' - \widehat F_f \right\|_\infty > \epsilon/2 \right] \\
        \label{l3}
        & \leq 2 \Pr_{\substack{e_1,...,e_N\\e_1',...,e_N'}} \left[ \sup_{f \in \mathcal{F}} \epsilon/8 + \left\| D \widehat F_f' - D \widehat F_f \right\|_\infty > \epsilon/2 \right] \\
        \label{l4}
        & \leq 2 \mathcal{N}_{\epsilon/16} \sup_{f \in \mathcal{F}} \Pr_{\substack{e_1,...,e_N\\e_1',...,e_N'}} \left[ \epsilon/8 + \left\| D \widehat F_f' - D \widehat F_f \right\|_\infty > \epsilon/2 \right] \\
        \label{l5}
        & \leq 2 \mathcal{N}_{\epsilon/16} \sup_{f \in \mathcal{F}} \Pr_{\substack{e_1,...,e_N\\e_1',...,e_N'}} \left[ \epsilon/4 + \left\| \widehat F_f' - \widehat F_f \right\|_\infty > \epsilon/2 \right] \\
        \label{l6}
        & = 2 \mathcal{N}_{\epsilon/16} \sup_{f \in \mathcal{F}} \Pr_{\substack{e_1,...,e_N\\e_1',...,e_N'}} \left[ \left\| \widehat F_f' - \widehat F_f \right\|_\infty > \epsilon/4 \right] \\
        \label{l7}
        & \leq 2 \mathcal{N}_{\epsilon/16} \sup_{f \in \mathcal{F}} \Pr_{\substack{e_1,...,e_N\\e_1',...,e_N'}} \left[ \left\| \widehat Q_f' - \widehat Q_f \right\|_\infty > \frac{\epsilon}{4L} \right] \\
        \label{l8}
        & \leq 4 \mathcal{N}_{\epsilon/16} \sup_{f \in \mathcal{F}} \Pr_{\substack{e_1,...,e_N\\e_1',...,e_N'}} \left[ \left\| \E \left[ \widehat Q_f \right] - \widehat Q_f \right\|_\infty > \frac{\epsilon}{8L} \right] \\
        \label{l9}
        & = 4 \mathcal{N}_{\epsilon/16} \sup_{f \in \mathcal{F}} \Pr_{e_1,...,e_N} \left[ \sup_{t \in \bbR} \left| F_f(t) - \frac{1}{N} \sum_{i = 1}^N 1\{\calR^e(f) \leq t\} \right| > \frac{\epsilon}{8L} \right] \\
        \label{l10}
        & \leq 8 \mathcal{N}_{\epsilon/16} \exp \left( - \frac{N\epsilon^2}{64L} \right).
    \end{align}
    Here, line~\eqref{l1} follows from the Symmetrization Lemma (Lemma~\ref{lemma:symmetrization}), lines~\eqref{l3} and \eqref{l5} follow from the definition of $D$, line~\eqref{l4} is a union bound over $\widehat{\mathcal{P}}_{\epsilon/16}$, line~\eqref{l7} follows from the Lipschitz assumption, line~\eqref{l8} follows from the triangle inequality, line~\eqref{l9} follows from the fact that the empirical CDF is an unbiased estimate of the true CDF, and line~\eqref{l10} follows from the DKW Inequality (Lemma~\ref{lemma:DKW}).
    
    Since $\sup_x f(x) - \sup_x g(x) \leq \sup_x f(x) - g(x)$,
    \begin{align}
        \notag
        & \Pr_{e_1,...,e_N} \left[ \sup_{f \in \calF, t \in \bbR} F_f(t) - \widehat F_f(t) > \epsilon + \operatorname{Bias}(\calF, \widehat F) \right] \\
        \notag
        & = \Pr_{e_1,...,e_N} \left[ \sup_{f \in \calF, t \in \bbR} F_f(t) - \widehat F_f(t) > \epsilon + \sup_{f \in \calF, t \in \bbR} F_f(t) - \E_{e_1,...,e_N} \left[ \widehat F_f(t) \right] \right] \\
        \notag
        & \leq \Pr_{e_1,...,e_N} \left[ \sup_{f \in \calF, t \in \bbR} \E_{e_1,...,e_N} \left[ \widehat F_f(t) \right] - \widehat F_f(t) > \epsilon \right] \\
        \label{l11}
        & \leq 8 \mathcal{N}_{\epsilon/16} \exp \left( - \frac{N\epsilon^2}{64L} \right),
    \end{align}
    by~\eqref{l10}.
    Meanwhile, applying the presumed uniform bound on within-environment generalization error together with a union bound over the $N$ environments, gives us a high-probability bound on the maximum generalization error of $f$ within any of the $N$ environments:
    \[\Pr_{\substack{\{e_i\}_{i = 1}^N \sim \bbP(e)\\\{(X_{i,j},Y_{i,j})\}_{j = 1}^n \sim \bbP(X^{e_i},Y^{e_i})}} \left[ \max_{i \in [N]} \sup_{f \in \calF} \calR^{e_i}(f) - \widehat{R}^{e_i}(f) \leq t_{n,\frac{\delta}{2N},\calF} \right] \leq \delta/2,\]
    It follows that, with probability at least $1 - \delta/2$, for all $f \in \calF$ and $t \in \bbR$,
    \[\widehat F_f \left( t + t_{n,\frac{\delta}{2N},\calF} \right)
      \leq \widehat F_{\widehat \calR^{e_1}(f),...,\widehat \calR^{e_1}(f)}(t),\]
    where $\widehat F_{\widehat \calR^{e_1}(f),...,\widehat \calR^{e_1}(f)}(t)$ is the actually empirical estimate $\widehat F_f(t)$ of computed using the $N$ empirical risks $\widehat \calR^{e_1}(f),...,\widehat \calR^{e_N}(f)$.
    Plugging this into the left-hand side of Inequality~\eqref{l11},
    \[\Pr_{e_1,...,e_N} \left[ \sup_{f \in \calF, t \in \bbR} F_f \left( t + t_{n,\frac{\delta}{2N},\calF} \right) - \widehat F_{\widehat \calR^{e_1}(f),...,\widehat \calR^{e_1}(f)}(t) > \epsilon  + \operatorname{Bias}(\calF, \widehat F) \right]
    \leq 8 \mathcal{N}_{\epsilon/16} \exp \left( - \frac{N\epsilon}{64L} \right).\]
    Setting $t = \widehat F^{-1}_{\widehat \calR^{e_1}(f),...,\widehat \calR^{e_1}(f)}(\alpha)$ and applying the non-decreasing function $F_f^{-1}$ gives the desired result:
    \[\Pr_{e_1,...,e_N} \left[ \sup_{f \in \calF, t \in \bbR} F_f^{-1} \left( \alpha - \epsilon - \operatorname{Bias}(\calF, \widehat F) \right) - \widehat F^{-1}_{\widehat \calR^{e_1}(f),...,\widehat \calR^{e_1}(f)}(\alpha) \geq t_{n,\frac{\delta}{2N},\calF} \right]
    \leq 8 \mathcal{N}_{\epsilon/16} \exp \left( - \frac{N\epsilon}{64L} \right).\]
\end{proof}

\subsection{Kernel density estimator}

In this section, we apply our generalization bound Theorem~\eqref{thm:generalization} to the kernel density estimator (KDE)
\[\widehat F_h(t)
  = \int_{-\infty}^t \frac{1}{nh} \sum_{i = 1}^n K \left( \frac{\tau - X_i}{h} \right) \, d\tau\]
of the cumulative risk distribution under the assumptions that:
\begin{enumerate}
    \item the loss $\ell$ takes values in a bounded interval $[a, b] \subseteq \bbR$, and
    \item for all $f \in \mathcal{F}$, the true risk profile $F_f$ is $\beta$-H\"older continuous with constant $L$, for any $\beta > 0$. 
\end{enumerate}
We also make standard integrability and symmetry assumptions on the kernel $K : \bbR \to \bbR$ (see Section 1.2.2 \citep{tsybakov2004introduction} for discussion of these assumptions):
\[\int_\bbR |K(u)| \, du < \infty, \quad
  \int_\bbR K(u) \, du = 1, \quad
  \int_\bbR |u|^\beta |K(u) \, du < \infty,\]
and, for each positive integer $j < \beta$,
\begin{equation}
    \int_\bbR u^j K(u) \, du = 0.
    \label{eq:kernel_order_assumption}
\end{equation}

We will use Theorem~\ref{thm:generalization} to show that, for an appropriately chosen bandwidth $h$,
\[\sup_{f \in \calF, t \in \bbR} F_f(t) - \widehat F_f(t)
  \in O_P \left( \left( \frac{\log N}{N} \right)^{\frac{\beta}{2\beta + 1}} \right).\]

We start by bounding the bias term $B(\calF, \widehat F)$. Since
\begin{align*}
    \E_{X_1,...,X_n} \left[ \int_{-\infty}^t \left| \frac{1}{nh} \sum_{i = 1}^n K \left( \frac{\tau - X_i}{h} \right) \right| \right] \, d\tau
      & \leq \frac{1}{h} \E_X \left[ \int_{-\infty}^\infty \left| K \left( \frac{\tau - X_i}{h} \right) \right| \right] \, d\tau \\
      & \leq \|K\|_1 < \infty,
\end{align*}
applying Fubini's theorem, linearity of expectation, the change of variables $x \mapsto \tau + xh$, Fubini's theorem again, and the fact that $\int_\bbR K(u) \, dx = 1$,
\begin{align*}
    F_f(t) - \E_{X_1,...,X_n} \left[ \widehat F_h(t) \right]
    & = F_f(t) - \E_{e_1,...,e_N} \left[ \int_{-\infty}^t \frac{1}{nh} \sum_{i = 1}^n K \left( \frac{\tau - X_i}{h} \right) \right] \\
    & = F_f(t) - \int_{-\infty}^t \E_{X_1,...,X_n} \left[ \frac{1}{nh} \sum_{i = 1}^n K \left( \frac{\tau - X_i}{h} \right) \right] \\
    & = F_f(t) - \int_{-\infty}^t \int_{\bbR} \frac{1}{h} K \left( \frac{\tau - x}{h} \right) p(x) \, dx \, d\tau \\
    & = F_f(t) - \int_{-\infty}^t \int_{\bbR} K(x) p(\tau + xh) \, dx \, d\tau \\
    & = F_f(t) - \int_{\bbR} K(x) \int_{-\infty}^t p(\tau + xh) \, d\tau \, dx \\
    & = \int_{\bbR} K(x) \left( F_f(t) - F(t + xh) \right) \, dx.
\end{align*}
By Taylor's theorem for some $\pi \in [0, 1]$,
\begin{align*}
    F(t + xh)
    & = \sum_{j = 0}^{\lfloor \beta \rfloor - 1} \frac{(xh)^j}{j!} \frac{d^j}{dt^j} F_f(t)
      + \frac{(xh)^{\lfloor \beta \rfloor}}{\lfloor \beta \rfloor!} \frac{d^{\lfloor \beta \rfloor}}{dt^{\lfloor \beta \rfloor}} F(t + \pi xh).
\end{align*}
Hence, by the assumption~\eqref{eq:kernel_order_assumption},
\begin{align*}
    F_f(t) - \E_{X_1,...,X_n} \left[ \widehat F_h(t) \right]
    & = \int_{\bbR} K(x) \left( F_f(t) - \sum_{j = 0}^{\lfloor \beta \rfloor - 1} \frac{(xh)^j}{j!} \frac{d^j}{dt^j} F_f(t)
      + \frac{(xh)^{\lfloor \beta \rfloor}}{\lfloor \beta \rfloor!} \frac{d^{\lfloor \beta \rfloor}}{dt^{\lfloor \beta \rfloor}} F(t + \pi xh) \right) \, dx \\
    & = \int_{\bbR} K(x) \left( \frac{(xh)^{\lfloor \beta \rfloor}}{\lfloor \beta \rfloor!} \frac{d^{\lfloor \beta \rfloor}}{dt^{\lfloor \beta \rfloor}} F(t + \pi xh) \right) \, dx \\
    & = \int_{\bbR} K(x) \frac{(xh)^{\lfloor \beta \rfloor}}{\lfloor \beta \rfloor!} \left( \frac{d^{\lfloor \beta \rfloor}}{dt^{\lfloor \beta \rfloor}} F(t + \pi xh) - \frac{d^{\lfloor \beta \rfloor}}{dt^{\lfloor \beta \rfloor}} F_f(t) \right) \, dx.
\end{align*}
Thus, by the H\"older continuity assumption,
\begin{align}
    \notag
    \left| F_f(t) - \E_{X_1,...,X_n} \left[ \widehat F_h(t) \right] \right|
    & \leq \int_{\bbR} K(x) \frac{(xh)^{\lfloor \beta \rfloor}}{\lfloor \beta \rfloor!} \left| \frac{d^{\lfloor \beta \rfloor}}{dt^{\lfloor \beta \rfloor}} F(t + \pi xh) - \frac{d^{\lfloor \beta \rfloor}}{dt^{\lfloor \beta \rfloor}} F_f(t) \right| \, dx \\
    \label{ineq:KDE_bias_bound}
    & \leq \int_{\bbR} K(x) \frac{(xh)^{\lfloor \beta \rfloor}}{\lfloor \beta \rfloor!} L (\pi xh)^{\beta - \lfloor \beta \rfloor} \, dx
    \leq C h^\beta,
\end{align}
where $C := \frac{L}{\lfloor \beta \rfloor!} \int_{\bbR} |x|^\beta |K(x)| \, dx$ is a constant.

Next, since, by the Fundamental Theorem of Calculus,
\begin{align*}
    \frac{d^{\lfloor \beta+1 \rfloor}}{dt^{\lfloor \beta+1 \rfloor}} \widehat F_f(t)
      = \frac{d^{\lfloor \beta+1 \rfloor}}{dt^{\lfloor \beta+1 \rfloor}} \int_{-\infty}^t \frac{1}{nh} \sum_{i = 1}^N K \left( \frac{\tau - X_i}{h} \right) \, d\tau
      = \frac{1}{nh} \sum_{i = 1}^N \frac{d^{\lfloor \beta \rfloor}}{dt^{\lfloor \beta \rfloor}} K \left( \frac{t - X_i}{h} \right),
\end{align*}
$\|F_f\|_{\calC^{\beta+1}} \leq \|K_h\|_{\calC^\beta} = h^{-(\beta+1)} \|K\|_{\calC^\beta}$.
Hence, by standard bounds on the covering number of H\"older continuous functions~\citep{devore1993constructive}, there exists a constant $c > 0$ depending only on $\beta$ such that
\begin{equation}
    \calN_{\epsilon/16}(\calN)
      \leq \exp \left( c (b - a) \left( \frac{\|K\|_{\calC^\beta}}{h^{\beta+1} \epsilon} \right)^{\frac{1}{\beta+1}} \right)
      = \exp \left( c \frac{(b - a)}{h} \left( \frac{\|K\|_{\calC^\beta}}{\epsilon} \right)^{\frac{1}{\beta+1}} \right).
    \label{ineq:KDE_covering_number_bound}
\end{equation}

Finally, since $\widehat F_h = \widehat Q * K_h$ (where $*$ denotes convolution), by linearity of the convolution and Young's convolution inequality~\citep[p.34]{ball1997elementary},
\begin{align*}
    \left\| \widehat F_h - \widehat F_h' \right\|_\infty
    & \leq \left\| \widehat Q - \widehat Q' \right\|_\infty \|K_h\|_1.
\end{align*}
Since, by a change of variables, $\|K_h\|_1 = \|K\|_1 = 1$, the KDE is a $1$-Lipschitz function of the empirical CDF, under $\calL_\infty(\bbR)$.

Thus, plugging Inequality~\eqref{ineq:KDE_bias_bound}, Inequality~\eqref{ineq:KDE_covering_number_bound}, and $L = 1$ into Theorem~\ref{thm:generalization} and taking $n \to \infty$ gives, for any $\epsilon > 0$,
\[\Pr_{e_1,...,e_N} \left[ \sup_{f \in \calF} F^{-1}_f\left(\alpha - C h^\beta - \epsilon \right) - \widehat F^{-1}_f(\alpha) > 0
        \right]
        \leq 8 \exp \left( c \frac{b - a}{h} \left( \frac{\|K\|_{\calC^\beta}}{\epsilon} \right)^{\frac{1}{\beta+1}} \right) e^{-\frac{N\epsilon^2}{64}}.\]
Plugging in $\epsilon = \sqrt{\frac{\log \frac{1}{\delta} + c\frac{b - a}{h}}{N}}$ gives
\[\Pr_{e_1,...,e_N} \left[ \sup_{f \in \calF} F^{-1}_f\left(\alpha - C h^\beta - \sqrt{\frac{\log \frac{1}{\delta} + c\frac{b - a}{h}}{N}} \right) - \widehat F^{-1}_f(\alpha) > 0
        \right]
        \leq \delta.\]
This bound is optimized by $h \asymp \left( (b - a) \frac{\log N}{N} \right)^{\frac{1}{2\beta + 1}}$, giving an overall bound of
\[\Pr_{e_1,...,e_N} \left[ \sup_{f \in \calF, t \in \bbR} F_f(t) - \widehat F_f(t) > c h^\frac{\beta}{2\beta + 1} \right] \leq \delta\]
\[\Pr_{e_1,...,e_N} \left[ \sup_{f \in \calF} F^{-1}_f\left(\alpha - c h^\frac{\beta}{2\beta + 1} + \sqrt{\frac{\log \frac{1}{\delta}}{N}} \right) - \widehat F^{-1}_f(\alpha) > 0
        \right]
        \leq \delta.\]
for some $c > 0$.
In particular, as $N, n \to \infty$, the EQRM estimate $\widehat f$ satisfies
\[F^{-1}_{\widehat f}(\alpha) \to \inf_{f \in \calF} F^{-1}_f(\alpha).\]

\newpage
\section{Further implementation details}%
\label{sec:impl_details}%

\subsection{Algorithm}%
\label{sec:impl_details:algs}%
Below we detail the EQRM algorithm. Note that: (i) any distribution estimator may be used in place of $\textsc{dist}$ so long as the functions $\textsc{dist}.\textsc{estimate\_params}$ and $\textsc{dist}.\textsc{icdf}$ are differentiable; (ii) other bandwidth-selection methods may be used on line 14, with the Gaussian-optimal rule serving as the default; and (iii) the bisection method \textsc{bisect} on line 20 requires an additional parameter, the maximum number of steps, which we always set to 32.

  \begin{algorithm}[H]
    \caption{Empirical Quantile Risk Minimization (EQRM).}
    \label{alg:eqrm}
      \kwInit{Predictor $f_{\theta}$, loss function $\ell$, desired probability of generalization $\alpha$, learning rate $\eta$, distribution estimator \textsc{dist}, $M$ datasets with $D^m = \{(x^m_i, y^m_i)\}_{i=1}^{n_m}$.}
      \vspace{1mm}
      Initialize $f_{\theta}$\;
      \vspace{1mm}
      \While{not converged}{
        \vspace{1mm}
        \tcc{Get per-domain risks (i.e.\ average losses)}
        $L^m \gets \frac{1}{n_m} \sum_{i=1}^{n_m} \ell(f_{\theta}(x^m_i), y^m_i)$, for $m = 1, \dots, M$ \;
        \vspace{1mm}
        \tcc{Estimate the parameters of $\widehat{\bbT}_f$}
        \textsc{dist.estimate\_params}($\bL$) \;
        \vspace{1mm}
        \tcc{Compute the $\alpha$-quantile of $\widehat{\bbT}_f$}
        $q \gets$ \textsc{dist.icdf}($\alpha$) \;
        \vspace{1mm}
        \tcc{Update $f_\theta$}
        $\theta \gets \theta - \eta \cdot \nabla_{\theta} q$ \;
      }
      \vspace{1mm}
      \kwOutput{$f_{\theta}$}
      \vspace{4mm}
      \myproc{\textsc{gauss.estimate\_params}($\bL$)}{
        \tcc{Compute the sample mean and variance}
        $\hat{\mu} \gets \frac{1}{M} \sum_{m=1}^M L^m$ \;
        $\hat{\sigma}^2 \gets \frac{1}{M-1} \sum_{m=1}^M (L^m - \hat{\mu})^2$ \;
      }
      \vspace{1mm}
      \myproc{\textsc{gauss.icdf}($\alpha$)}{
        \KwRet $\hat{\mu} + \hat{\sigma} \cdot \Phi^{-1}(\alpha)$\;
      }
      \vspace{1mm}
      \myproc{\textsc{kde.estimate\_params}($\bL$)}{
        \tcc{Set bandwidth $h$ (Gaussian-optimal rule used as default)}
        $\hat{\sigma}^2 \gets \frac{1}{M-1} \sum_{m=1}^M (L^m - \frac{1}{M} \sum_{j=1}^M L^j)^2$\;
        \vspace{0.5mm}
        $h\! \gets\! (\frac{4}{3M})^{0.2} \cdot \hat{\sigma}$
      }
      \vspace{1mm}
      \myproc{\textsc{kde.icdf}($\alpha$)}{
        \tcc{Define the CDF when using $M$ Gaussian kernels}
        $F_m(x')$ $\gets L^m + h \cdot \Phi(x')$ \;
        $F(x') \gets \frac{1}{M} \sum_{m=1}^M F_m(x')$ \;
        \vspace{2mm}
        \tcc{Invert the CDF via bisection}\vspace{-0.5mm}
        mn $\gets \min_m F^{-1}_m(\alpha)$ \;
        mx $\gets \max_m F^{-1}_m(\alpha)$ \;\vspace{1mm}
        \KwRet  \textsc{bisect}($F, \alpha$, mn, mx) \;
      }

  \end{algorithm}

\subsection{ColoredMNIST}%
\label{sec:impl_details:cmnist}%
For the \texttt{CMNIST} results of \cref{sec:exps:synthetic}, we used full batches (size $25000$), $400$ steps for ERM pretraining, $600$ total steps for IRM, VREx, EQRM, and $1000$ total steps for GroupDRO, SD, and IGA. We used the original \texttt{MNIST} training set to create training and validation sets for each domain, and the original \texttt{MNIST} test set for the test sets of each domain. We also decayed the learning rate using cosine annealing/scheduling. We swept over penalty weights in $\{50, 100, 500, 1000, 5000\}$ for IRM, VREx and IGA, penalty weights in $\{0.001, 0.01, 0.1, 1\}$ for SD, $\eta$'s in $\{0.001, 0.01, 0.1, 0.5, 1.0\}$ for GroupDRO, and $\alpha$'s in $1 - \{e^{-100}, e^{-250}, e^{-500}, e^{-750}, e^{-1000}\}$ for EQRM. To allow these values of $\alpha$, which are \emph{very} close to 1, we used an asymptotic expression for the Normal inverse CDF, namely $\Phi^{-1}(\alpha) \approx \sqrt{-2 \ln (1 - \alpha)}$ as $\alpha \to 1$~\citep{blair1976rational}. This allowed us to parameterize $\alpha = 1 - e^{-1000}$ as $\ln (1 - \alpha) = \ln (e^{-1000})= -1000$, avoiding issues with floating-point precision. As is the standard for \texttt{CMNIST}, we used a test-domain validation set to select the best settings (after the total number of steps), then reported the mean and standard deviation over 10 random seeds on a test-domain test set. As in previous works, the hyperparameter ranges of all methods were selected by peeking at test-domain performance. While not ideal, this is quite difficult to avoid with \texttt{CMNIST} and highlights the problem of model selection more generally in DG---as discussed by many previous works~\citep{arjovsky2020invariant, krueger21rex, gulrajani2020search, zhang2022rich}. Finally, we note several observations from our \texttt{CMNIST}, WILDS and DomainBed experiments which, despite not being thoroughly investigated with their own set of experiments (yet), may prove useful for future work: (i) ERM pretraining seems an effective strategy for DG methods, and can likely replace the more delicate penalty-annealing strategies (as also observed in \citep{zhang2022rich}); (ii) lowering the learning rate after ERM pretraining seems to stabilize DG methods; and (iii) EQRM often requires a lower learning rate than other DG methods after ERM pretraining, with its loss and gradients tending to be significantly larger.

\subsection{DomainBed}
\label{sec:impl_details:domainbed}%

For EQRM, we used the default algorithm setup: a kernel-density estimator of the risk distribution with the ``Gaussian-optimal'' rule~\citep{silverman1986density} for bandwidth selection. We used the standard hyperparameter-sampling procedure of Domainbed, running over 3 trials for 20 randomly-sampled hyperparameters per trial. For EQRM, this involved:

\begin{center}
\begin{tabular}{@{}lcc@{}}
\toprule
\textbf{Hparam}         & \textbf{Default}             & \textbf{Sampling} \\
\midrule
$\alpha$                & 0.75          & $U(0.5, 0.99)$ \\
Burn-in/anneal iters    & 2500          & $10^k$, with $k \sim U(2.5, 3.5)$ \\
EQRM learning rate (post burn-in)  & $10^{-6}$     & $10^k$, with $k \sim U(-7, -5)$ \\               
\bottomrule
\end{tabular}
\end{center}

All other all hyperparameters remained as their DomainBed-defaults, while the baseline results were taken directly from the most up-to-date DomainBed tables\footnote{\url{https://github.com/facebookresearch/DomainBed/tree/main/domainbed/results/2020_10_06_7df6f06}}. See our code for further details.

\subsection{WILDS}
\label{sec:impl_details:wilds}%

We considered two WILDS datasets: \texttt{iWildCam} and \texttt{OGB-MolPCBA} (henceforth \texttt{OGB}).  For both of these datasets, we used the architectures use in the original WILDS paper~\cite{wilds2021}; that is, for \texttt{iWildCam} we used a ResNet-50 architecture~\citep{he2016deep} 
pretrained on ImageNet~\citep{deng2009imagenet},
and for \texttt{OGB}, we used a Graph Isomorphism Network~\citep{xu2018powerful}
combined with virtual nodes~\cite{gilmer2017neural}. 
To perform model-selection, we followed the guidelines provided in the original WILDS paper~\cite{wilds2021}.  In particular, for each of the baselines we consider, we performed grid searches over the hyperparameter ranges listed in~\cite{wilds2021} with respect to the given validation sets; see \cite[Appendices E.1.2 and E.4.2]{wilds2021} for a full list of these hyperparameter ranges.

\paragraph{EQRM.}  For both datasets, we ran EQRM with KDE using the Gaussian-optimal bandwidth-selection method.  All EQRM models were initialized with the same ERM checkpoint, which is obtained by training ERM using the code provided by~\cite{wilds2021}.  Following~\cite{wilds2021}, for \texttt{iWildCam}, we trained ERM for 12 epochs, and for OGB, we trained ERM for 100 epochs.  We again followed~\cite{wilds2021} by using a batch size of 32 for \texttt{iWildCam} and 8 groups per batch.  For \texttt{OGB}, we performed grid searches over the batch size in the range $B\in\{32, 64, 128, 256, 512, 1024, 2048\}$, and we used $0.25B$ groups per batch.  We selected the learning rate for EQRM from $\eta\in\{10^{-2}, 10^{-3}, 10^{-4}, 10^{-5}, 10^{-6}, 10^{-7}, 10^{-8}\}$.

\paragraph{Computational resources.} 
All experiments on the WILDS datasets were run across two four-GPU workstations, comprising a total of eight Quadro RTX 5000 GPUs.

\section{Connections between QRM and DRO}\label{app:dro}
In this appendix we draw connections between quantile risk minimization~(QRM) and distributionally robust optimization (DRO) by considering an alternative optimization problem which we call \textit{superquantile risk minimization}~\footnote{This definition assumes that $\bbT_f$ is continuous; for a more general treatment, see~\cite{rockafellar2000CVaR}.}:
\begin{align}\label{eq:sqrm_def}
    \min_{f\in\calF} \: \SQ_\alpha(R; \bbT_f)
    \qquad \text{where} \qquad 
    \SQ_\alpha(R; \bbT_f) := \E_{R\sim\bbT_f} \left[ R \:\: \big| \:\: R \geq F^{-1}_{\bbT_f}(\alpha) \right].
\end{align}

Here, $\SQ_{\alpha}$ represents the \textit{superquantile}---also known as the \textit{conditional value-at-risk} (CVaR) or \textit{expected tail loss}---at level $\alpha$, which
can be seen as the conditional expectation of a random variable $R$ subject to $R$ being larger than the $\alpha$-quantile
$F^{-1}(\alpha)$. In our case, where $R$ represents the statistical risk on a randomly-sampled environment, $\SQ_\alpha$ can be seen as the expected risk in the worst $100 \cdot (1-\alpha)$\% of cases/domains. Below, we exploit the well-known duality properties of CVaR to formally connect~\eqref{eq:qrm} and GroupDRO~\cite{sagawa2019distributionally}; see Prop.~\ref{prop:dual-rep-cvar} for details. 

\subsection{Notation for this appendix}
Throughout this appendix, for each $f\in\calF$, we will let the risk random variable $R$ be a defined on the probability space $(\R_+, \calB, \bbT_f)$, where $\R_+$ denotes the nonnegative real numbers and $\calB$ denotes the Borel $\sigma$-algebra on $\R_+$.  We will also consider the Lebesgue spaces $L^p := L^p(\R_+, \calB, \bbT_f)$ of functions $h$ for which $\E_{r\sim\bbT_f}[|h(r)|^p]$ is finite.  For conciseness, we will use the notation
\begin{align}
    \left\langle g(r), h(r)\right\rangle := \int_{r\geq 0} g(r)h(r) \text{d}r
\end{align}
to denote the standard inner product on $\R_+$.  Furthermore, we will use the notation $\bbU\ll\bbV$ to signify that $\bbU$ is \emph{absolutely continuous} with respect to $\bbV$, meaning that if $\bbU(A)=0$ for every set $A$ for which $\bbV(A)=0$.  We also use the abbreviation ``a.e.`` to mean ``almost everywhere.''  Finally, the notation $\Pi_{[a,b]}(c)$ denotes the projection of a number $c$ into the real interval $[a,b]$.

\subsection{(Strong) Duality of the superquantile}
We begin by proving that strong duality holds for the superquantile function $\SQ_\alpha$.  We note that this duality result is well-known in the literature (see, e.g.,~\cite{shapiro2021lectures}), and has been exploited in the context of adaptive sampling~\cite{curi2020adaptive} and offline reinforcement learning~\cite{urpi2021risk}.  We state this result and proof for the sake of exposition.
\begin{proposition}[Dual representation of $\SQ_\alpha$]\label{prop:dual-rep-cvar}
If $R\in L^P$ for some $p\in(1,\infty)$, then 
\begin{align}
    \SQ_\alpha(R; \bbT_f) = \max_{\bbU\in\calU_f(\alpha)} \E_{\bbU}[R] \label{eq:cvar-prob}
\end{align}
where the uncertainty set $\calU_f(\alpha)$ is defined as
\begin{align}
    \calU_f(\alpha) := \left\{ \bbU \in L^q : \bbU \ll \bbT_f, \: \bbU\in[0, \nicefrac{1}{1-\alpha}] \text{ a.e. }, \norm{U}_{L^1}=1 \right\}. \label{eq:uncertainty-set-cvar}
\end{align}
\end{proposition}

\begin{proof}
Note that the primal objective can be equivalently written as
\begin{align}
    \SQ_\alpha(R; \bbT_f) = \min_{t\in\R} \:  \left\{ t + \frac{1}{1-\alpha} \langle  (R - t)_+, \bbT_f \rangle \right\}
\end{align}
where $(z)_+ = \max\{0,z\}$~\cite{rockafellar2000CVaR}, which in turn has the following epigraph form:
\begin{alignat}{2}
    &\min_{t\in\R, \:\: s\in L^p_+} && t + \frac{1}{1-\alpha} \langle s, \bbT_f \rangle \\
    &\subjto && R(r) - t \leq s(r) \:\text{ a.e. } r\in\R_+.
\end{alignat}
When written in Lagrangian form, we can express this problem as
\begin{align}
    \min_{t\in\R, \:\: s\in L^p_+} \max_{\lambda\in L^q_+}  \:\: \left\{t(1 - \langle 1, \lambda\rangle) +  \left\langle s, \frac{1}{1-\alpha}\bbT_f - \lambda \right\rangle + \langle R, \lambda\rangle \right\}.
\end{align}
Note that this objective is \emph{linear} in $t$, $s$, and $\lambda$, and therefore due to the strong duality of linear programs, we can optimize over $s$, $t$, and $\lambda$ in any order~\cite{boyd2004convex}.  Minimizing over $t$ reveals that the problem is unbounded unless $\int_{r\geq 0} \lambda(r)\text{d}r = 1$, meaning that $\lambda$ is a probability distribution since $\lambda(r)\geq 0$ almost everywhere.  Thus, the problem can be written as
\begin{align}
    \min_{s\in L^p_+} \max_{\lambda\in\calP(\R_+)} \:\: \left\{ \left\langle s, \frac{1}{1-\alpha}\bbT_f - \lambda \right\rangle + \langle R, \lambda\rangle \right\}
\end{align}
where $\calP^q(\R_+)$ denotes the subspace of $L^q$ of probability distributions on $\R_+$.

Now consider the maximization over $s$.  Note that if there is a set $A\subset\Eall$ of nonzero Lebesgue measure on which $\lambda(A) \geq (\nicefrac{1}{1-\alpha})\bbT_f(A)$, then the problem is unbounded below because $s(A)$ can be made arbitrarily large.  Therefore, it must be the case that $\lambda \leq (\nicefrac{1}{1-\alpha})\bbT_f$ almost everywhere.  On the other hand, if $\lambda(A) \leq (\nicefrac{1}{1-\alpha})\bbT_f(A)$, then $s(A) = 0$ minimizes the first term in the objective.  Therefore, $s$ can be eliminated provided that $\lambda\leq (\nicefrac{1}{1-\alpha})\bbT_f$ almost everywhere.  Thus, we can write the problem as
\begin{alignat}{2}
    &\max_{\lambda\in\calP^q(\bbR_+)} &&\langle R, \lambda\rangle = \E_{\lambda}[R] \\
    &\subjto && \lambda(r) \leq \frac{1}{1-\alpha}\bbT_f(r) \:\text{ a.e. } r\geq 0.
\end{alignat}
Now observe that the constraint in the above problem is equivalent to $\lambda\ll\bbQ$.  Thus, by defining $\bbU = \text{d}\lambda/\text{d}\bbT_f$ to be the Radon-Nikodym derivative of $\lambda$ with respect to $\bbQ$, we can write the problem in the form of~\eqref{eq:cvar-prob}, completing the proof.
\end{proof}

Succinctly, this proposition shows that provided that $R$ is sufficiently smooth (i.e., an element of $L^p$), it holds that minimizing the superquantile function is equivalent to solving
\begin{align}
    \min_{f\in\calF} \: \max_{\bbU\in\calU_f(\alpha)} \E_{\bbU}[R] \label{eq:dro-cvar}
\end{align}
which is a distributionally robust optimization (DRO) problem with uncertainty set $\calU_f(\alpha)$ as defined in~\eqref{eq:uncertainty-set-cvar}.  In plain terms, for any $\alpha\in(0,1)$, this uncertainty set contains probability distributions on $\R_+$ which can place no larger than $\nicefrac{1}{1-\alpha}$ on any risk value. 

At an intuitive level, this shows that by varying $\alpha$ in~\cref{eq:sqrm_def},
one can interpolate between a range DRO problems.  In particular, at level $\alpha=1$, we recover the problem in~\eqref{eq:domain-gen-rewritten}, which can be viewed as a DRO problem which selects a Dirac distribution which places solely on the essential supremum of $R\sim\bbT_f$.  On the other hand, at level $\alpha=0$, we recover a problem which selects a distribution that equally weights each of the risks in different domains equally.  A special case of this is the GroupDRO formulation in~\cite{sagawa2019distributionally}, wherein under the assumption that the data is partitioned into $m$ groups, the inner maximum in~\eqref{eq:dro-cvar} is taken over the $(m-1)$-dimensional simplex $\Delta_m$ (see, e.g., equation (7) in~\cite{sagawa2019distributionally}).

\section{Additional analyses and experiments}%
\label{sec:additional_exps}%
\subsection{Linear regression}%
\label{sec:additional_exps:linear_regr}%
In this section we extend \cref{sec:exps:synthetic} to provide further analyses and discussion of EQRM using linear regression datasets based on~\cref{ex:example}. In particular, we: (i) extend \cref{fig:exps:linear-regr} to include plots of the predictors' risk CDFs (\ref{sec:additional_exps:linear_reg:cdf-curves}); and (ii) discuss the ability of EQRM to recover the causal predictor when $\sigma_1^2$, $\sigma_2^2$ and/or $\sigma_Y^2$ change over environments, compared to IRM~\citep{arjovsky2020invariant} and VREx~\citep{krueger21rex} (\ref{sec:additional_exps:linear_regr:risks_vs_functions}).

\subsubsection{Risk CDFs as risk-robustness curves}
\label{sec:additional_exps:linear_reg:cdf-curves}
As an extension of~\cref{fig:exps:linear-regr}, in particular the PDFs in \cref{fig:exps:linear-regr}~\textbf{B}, \cref{fig:exps:linear-regr:cdfs} depicts the risk CDFs for different predictors. Here we see that a predictor's risk CDF depicts its risk-robustness curve, and also that each $\alpha$ results in a predictor $f_{\alpha}$ with minimial $\alpha$-quantile risk. That is, for each desired level of robustness (i.e.\ probability of the upper-bound on risk holding, y-axis), the corresponding $\alpha$ has minimal risk (x-axis).

\begin{figure}[h]
    \centering
        \centering
        \includegraphics[width=0.4\linewidth]{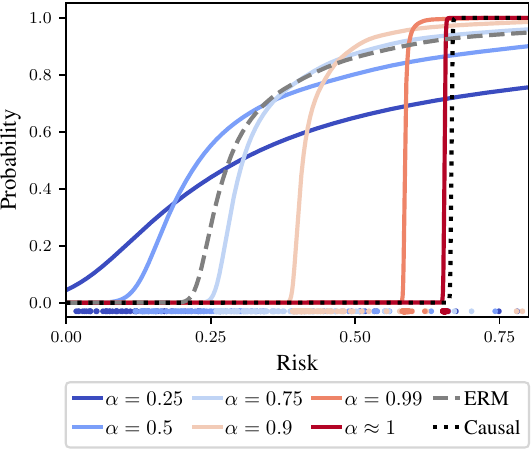}
    \caption{\small \textbf{Extension of \cref{fig:exps:linear-regr} showing the risk CDFs (i.e.\ risk-robustness curves) for different predictors.} For each risk upper-bound ($x$), we see the corresponding probability of it holding under the training domains ($y$). Note that, for each level of robustness ($y$, i.e.\ probability that the risk upper-bound holds), the corresponding $\alpha$ has the lowest upper-bound on risk ($x$). Also note that these CDFs correspond to the PDFs of \cref{fig:exps:linear-regr}~(\textbf{B}).}
    \label{fig:exps:linear-regr:cdfs}
\end{figure}

\subsubsection{Invariant risks vs.\ invariant functions}
\label{sec:additional_exps:linear_regr:risks_vs_functions}
We now compare seeking invariant \emph{risks} to seeking invariant \emph{functions} by analyzing linear regression datasets, based on \cref{ex:example}, in which $\sigma_1^2$, $\sigma_2^2$ and/or $\sigma_Y^2$ change over domains. This is turn allows us to compare EQRM (invariant risks), VREx~\citep{krueger21rex} (invariant risks), and IRM~\citep{arjovsky2020invariant} (invariant functions).

\paragraph{Domain-skedasticity.} For recovering the causal predictor, the key difference between using invariant \emph{risks} and invariant \emph{functions} lies in the assumption about \textit{domain-skedasticity}, i.e.\ the ``predicatability'' of $Y$ across domains. In particular, the causal predictor only has invariant risks in \textit{domain-homoskedastic} cases and not in \textit{domain-heteroskedastic} cases, the latter describing scenarios in which the predictability of $Y$ (i.e.\ the amount of irreducible error or intrinsic noise) varies across domains, meaning that the risk of the causal predictor will be smaller on some domains than others. Thus, methods seeking the causal predictor through invariant risks must assume domain homoskedasticity~\citep{peters2016causal, krueger21rex}. In contrast, methods seeking the causal predictor through invariant \emph{functions} need not make such a domain-homoskedasticity assumption, but instead the slightly weaker assumption of the conditional mean $\bbE[Y |\PA(Y)]$ being invariant across domains. As explained in the next paragraph and summarized in~\cref{tab:qrm-comparisons}, this translates into the coefficient $\beta_{\text{cause}}$ being invariant across domains for the linear SEM of \cref{ex:example}.

\begin{table}[tb]
    \centering
    \caption{\small Recovering the causal predictor for linear regression tasks based on \cref{ex:example}. A tick means that it is \emph{possible} to recover the causal predictor, under further assumptions.}\label{tab:qrm-comparisons}
    \resizebox{0.8\textwidth}{!}{
    \begin{tabular}{@{}cccccccc@{}} \toprule
         \multirow{2}{*}{Changing} & \multirow{2}{*}{\makecell{Domain \\ Scedasticity}} & \multicolumn{2}{c}{Invariant} & \multirow{2}{*}{IRM} & \multirow{2}{*}{VREx} & \multirow{2}{*}{EQRM}  \\ \cmidrule(lr){3-4}
         & & Risk & Function ($\beta_{\text{cause}}$) \\ \midrule
         $\sigma_1$ & \textit{Homo}scedastic & \cmark & \cmark & \cmark & \cmark & \cmark   \\
         $\sigma_2$ & \textit{Homo}scedastic & \cmark & \cmark & \cmark & \cmark & \cmark \\
         $\sigma_Y$ & \textit{Hetero}scedastic & \xmark & \cmark & \cmark & \xmark & \xmark \\ \bottomrule
    \end{tabular}}
\end{table}

\paragraph{Mathematical analysis.} We now analyze the risk-invariant solutions of \cref{ex:example}. We start by expanding the structural equations of \cref{ex:example} as:
\begin{align*}
    X_1 &= N_1, \\
    Y &= N_1 + N_Y, \\
    X_2 &= N_1 + N_Y + N_2.
\end{align*}
We then note that the goal is to learn a model $\widehat{Y} = \beta_1 \cdot X_1 + \beta_2 \cdot X_2$, which has residual error
\begin{align*}
    \widehat{Y} - Y &= \beta_1 \cdot N_1 + \beta_2 \cdot (N_1 + N_Y + N_2) - N_1 - N_Y \\
    &= (\beta_1 + \beta_2 - 1)\cdot N_1 + (\beta_2 - 1)\cdot N_Y + \beta_2 \cdot N_2.
\end{align*}
Then, since all variables have zero mean and the noise terms are independent, the risk (i.e.\ the MSE loss) is simply the variance of the residuals, which can be written as
\begin{align*}
    \bbE[(\widehat{Y} - Y)^2] &= (\beta_1 + \beta_2 - 1)^2\cdot \sigma^2_1 + (\beta_2 - 1)^2\cdot \sigma^2_Y + \beta_2^2 \cdot \sigma^2_2.
\end{align*}
Here, we have that, when:
\begin{itemize}\vspace{-2mm}
    \item \textbf{Only $\sigma_1$ changes:} the only way to keep the risk invariant across domains is to set $\beta_1 + \beta_2 = 1$. The minimal invariant-risk solution then depends on $\sigma_y$ and  $\sigma_2$:
    \begin{itemize}
        \item if $\sigma_y < \sigma_2$, the minimal invariant-risk solution sets $\beta_1 = 1$ and $\beta_2 = 0$  (causal predictor);
        \item if $\sigma_y > \sigma_2$, the minimal invariant-risk solution sets $\beta_1 = 0$ and $\beta_2 = 1$~(anti-causal predictor);
        \item if $\sigma_y = \sigma_2$, then any solution $(\beta_1, \beta_2)\! =\! (c, 1\! -\! c)$ with  $c\in[0,1]$ is a minimal invariant-risk solution, including the causal predictor $c\! =\! 1$, anti-causal predictor $c\! =\! 0$, and everything in-between. 
    \end{itemize}
    \item \textbf{Only $\sigma_2$ changes:} the invariant-risk solutions set $\beta_2 = 0$, with the minimal invariant-risk solution also setting $\beta_1 = 1$ (causal predictor).
    \item \textbf{$\sigma_1$ and $\sigma_2$ change:} \textit{the} invariant-risk solution sets $\beta_1=1, \beta_2=0$ (causal predictor).
    \item \textbf{Only $\sigma_Y$ changes:} the invariant-risk solutions set $\beta_2 = 1$, with the minimal invariant-risk solution also setting $\beta_1=0$ (anti-causal predictor).
    \item \textbf{$\sigma_1$ and $\sigma_Y$ change:} \textit{the} invariant-risk solution sets $\beta_1\! =\! 0, \beta_2\! =\! 1$ (anti-causal predictor).
    \item \textbf{$\sigma_2$ and $\sigma_Y$ change:} there is no invariant-risk solution.
    \item \textbf{$\sigma_1$, $\sigma_2$ and $\sigma_Y$ change:} there is no invariant-risk solution.
\end{itemize}

\paragraph{Empirical analysis.} To see this empirically, we refer the reader to Table 5 of \citet[App.~G.2]{krueger21rex}, which compares the invariant-risk solution of VREx to the invariant-function solution of IRM on the synthetic linear-SEM tasks of \citet[Sec.~5.1]{arjovsky2020invariant}, which calculate the MSE between the estimated coefficients $(\hat{\beta}_1, \hat{\beta}_2)$ and those of the causal predictor $(1, 0)$.

\paragraph{Different goals, solutions, and advantages.} We end by emphasizing the fact that 
the invariant-risk and invariant-function solutions have different pros and cons depending both on the goal and the assumptions made. If the goal is the recover the causal predictor or causes of $Y$, then the invariant-function solution has the advantage due to weaker assumptions on domain skedasticity. However, if the goal is learn predictors with stable or invariant performance, such that they perform well on new domains with high probability, then the invariant-risk solution has the advantage. For example, in the domain-heteroskedastic cases above where $\sigma_Y$ changes or $\sigma_Y$ and $\sigma_1$ change, the invariant-function solution recovers the causal predictor $\beta_1=1, \beta_2=0$ and thus has arbitrarily-large risk as $\sigma_Y \to \infty$ (i.e.\ in the worst-case). In contrast, the invariant-risk solution recovers the anti-causal predictor $\beta_1=0, \beta_2=1$ and thus has fixed risk $\sigma^2_2$ in all domains.

\subsection{DomainBed}%
\label{sec:additional_exps:domainbed}%

In this section, we include the full per-dataset DomainBed results. We consider the two most common model-selection methods of the DomainBed package---training-domain validation set and test-domain validation set (oracle)---and compare EQRM to a range of baselines. Implementation details for these experiments are provided in \S~\ref{sec:impl_details:domainbed} and our open-source code.

\subsubsection{Model selection: training-domain validation set} \label{sec:additional_exps:domainbed:train-dom}

\paragraph{VLCS}
\begin{center}
\adjustbox{max width=0.75\textwidth}{%
\begin{tabular}{lccccc}
\toprule
\textbf{Algorithm}   & \textbf{C}           & \textbf{L}           & \textbf{S}           & \textbf{V}           & \textbf{Avg}         \\
\midrule
ERM                  & 97.7 $\pm$ 0.4       & 64.3 $\pm$ 0.9       & 73.4 $\pm$ 0.5       & 74.6 $\pm$ 1.3       & 77.5                 \\
IRM                  & 98.6 $\pm$ 0.1       & 64.9 $\pm$ 0.9       & 73.4 $\pm$ 0.6       & 77.3 $\pm$ 0.9       & 78.5                 \\
GroupDRO             & 97.3 $\pm$ 0.3       & 63.4 $\pm$ 0.9       & 69.5 $\pm$ 0.8       & 76.7 $\pm$ 0.7       & 76.7                 \\
Mixup                & 98.3 $\pm$ 0.6       & 64.8 $\pm$ 1.0       & 72.1 $\pm$ 0.5       & 74.3 $\pm$ 0.8       & 77.4                 \\
MLDG                 & 97.4 $\pm$ 0.2       & 65.2 $\pm$ 0.7       & 71.0 $\pm$ 1.4       & 75.3 $\pm$ 1.0       & 77.2                 \\
CORAL                & 98.3 $\pm$ 0.1       & 66.1 $\pm$ 1.2       & 73.4 $\pm$ 0.3       & 77.5 $\pm$ 1.2       & 78.8                 \\
MMD                  & 97.7 $\pm$ 0.1       & 64.0 $\pm$ 1.1       & 72.8 $\pm$ 0.2       & 75.3 $\pm$ 3.3       & 77.5                 \\
DANN                 & 99.0 $\pm$ 0.3       & 65.1 $\pm$ 1.4       & 73.1 $\pm$ 0.3       & 77.2 $\pm$ 0.6       & 78.6                 \\
CDANN                & 97.1 $\pm$ 0.3       & 65.1 $\pm$ 1.2       & 70.7 $\pm$ 0.8       & 77.1 $\pm$ 1.5       & 77.5                 \\
MTL                  & 97.8 $\pm$ 0.4       & 64.3 $\pm$ 0.3       & 71.5 $\pm$ 0.7       & 75.3 $\pm$ 1.7       & 77.2                 \\
SagNet               & 97.9 $\pm$ 0.4       & 64.5 $\pm$ 0.5       & 71.4 $\pm$ 1.3       & 77.5 $\pm$ 0.5       & 77.8                 \\
ARM                  & 98.7 $\pm$ 0.2       & 63.6 $\pm$ 0.7       & 71.3 $\pm$ 1.2       & 76.7 $\pm$ 0.6       & 77.6                 \\
VREx                 & 98.4 $\pm$ 0.3       & 64.4 $\pm$ 1.4       & 74.1 $\pm$ 0.4       & 76.2 $\pm$ 1.3       & 78.3                 \\
RSC                  & 97.9 $\pm$ 0.1       & 62.5 $\pm$ 0.7       & 72.3 $\pm$ 1.2       & 75.6 $\pm$ 0.8       & 77.1                 \\ \midrule
EQRM                 & 98.3 $\pm$ 0.0       & 63.7 $\pm$ 0.8       & 72.6 $\pm$ 1.0       & 76.7 $\pm$ 1.1       & 77.8                 \\
\bottomrule
\end{tabular}}
\end{center}

\paragraph{PACS}
\begin{center}
\adjustbox{max width=0.75\textwidth}{%
\begin{tabular}{lccccc}
\toprule
\textbf{Algorithm}   & \textbf{A}           & \textbf{C}           & \textbf{P}           & \textbf{S}           & \textbf{Avg}         \\
\midrule
ERM                  & 84.7 $\pm$ 0.4       & 80.8 $\pm$ 0.6       & 97.2 $\pm$ 0.3       & 79.3 $\pm$ 1.0       & 85.5                 \\
IRM                  & 84.8 $\pm$ 1.3       & 76.4 $\pm$ 1.1       & 96.7 $\pm$ 0.6       & 76.1 $\pm$ 1.0       & 83.5                 \\
GroupDRO             & 83.5 $\pm$ 0.9       & 79.1 $\pm$ 0.6       & 96.7 $\pm$ 0.3       & 78.3 $\pm$ 2.0       & 84.4                 \\
Mixup                & 86.1 $\pm$ 0.5       & 78.9 $\pm$ 0.8       & 97.6 $\pm$ 0.1       & 75.8 $\pm$ 1.8       & 84.6                 \\
MLDG                 & 85.5 $\pm$ 1.4       & 80.1 $\pm$ 1.7       & 97.4 $\pm$ 0.3       & 76.6 $\pm$ 1.1       & 84.9                 \\
CORAL                & 88.3 $\pm$ 0.2       & 80.0 $\pm$ 0.5       & 97.5 $\pm$ 0.3       & 78.8 $\pm$ 1.3       & 86.2                 \\
MMD                  & 86.1 $\pm$ 1.4       & 79.4 $\pm$ 0.9       & 96.6 $\pm$ 0.2       & 76.5 $\pm$ 0.5       & 84.6                 \\
DANN                 & 86.4 $\pm$ 0.8       & 77.4 $\pm$ 0.8       & 97.3 $\pm$ 0.4       & 73.5 $\pm$ 2.3       & 83.6                 \\
CDANN                & 84.6 $\pm$ 1.8       & 75.5 $\pm$ 0.9       & 96.8 $\pm$ 0.3       & 73.5 $\pm$ 0.6       & 82.6                 \\
MTL                  & 87.5 $\pm$ 0.8       & 77.1 $\pm$ 0.5       & 96.4 $\pm$ 0.8       & 77.3 $\pm$ 1.8       & 84.6                 \\
SagNet               & 87.4 $\pm$ 1.0       & 80.7 $\pm$ 0.6       & 97.1 $\pm$ 0.1       & 80.0 $\pm$ 0.4       & 86.3                 \\
ARM                  & 86.8 $\pm$ 0.6       & 76.8 $\pm$ 0.5       & 97.4 $\pm$ 0.3       & 79.3 $\pm$ 1.2       & 85.1                 \\
VREx                 & 86.0 $\pm$ 1.6       & 79.1 $\pm$ 0.6       & 96.9 $\pm$ 0.5       & 77.7 $\pm$ 1.7       & 84.9                 \\
RSC                  & 85.4 $\pm$ 0.8       & 79.7 $\pm$ 1.8       & 97.6 $\pm$ 0.3       & 78.2 $\pm$ 1.2       & 85.2                 \\ \midrule
EQRM                 & 86.5 $\pm$ 0.4       & 82.1 $\pm$ 0.7       & 96.6 $\pm$ 0.2       & 80.8 $\pm$ 0.2       & 86.5                 \\
\bottomrule
\end{tabular}}
\end{center}

\paragraph{OfficeHome}
\begin{center}
\adjustbox{max width=0.75\textwidth}{%
\begin{tabular}{lccccc}
\toprule
\textbf{Algorithm}   & \textbf{A}           & \textbf{C}           & \textbf{P}           & \textbf{R}           & \textbf{Avg}         \\
\midrule
ERM                  & 61.3 $\pm$ 0.7       & 52.4 $\pm$ 0.3       & 75.8 $\pm$ 0.1       & 76.6 $\pm$ 0.3       & 66.5                 \\
IRM                  & 58.9 $\pm$ 2.3       & 52.2 $\pm$ 1.6       & 72.1 $\pm$ 2.9       & 74.0 $\pm$ 2.5       & 64.3                 \\
GroupDRO             & 60.4 $\pm$ 0.7       & 52.7 $\pm$ 1.0       & 75.0 $\pm$ 0.7       & 76.0 $\pm$ 0.7       & 66.0                 \\
Mixup                & 62.4 $\pm$ 0.8       & 54.8 $\pm$ 0.6       & 76.9 $\pm$ 0.3       & 78.3 $\pm$ 0.2       & 68.1                 \\
MLDG                 & 61.5 $\pm$ 0.9       & 53.2 $\pm$ 0.6       & 75.0 $\pm$ 1.2       & 77.5 $\pm$ 0.4       & 66.8                 \\
CORAL                & 65.3 $\pm$ 0.4       & 54.4 $\pm$ 0.5       & 76.5 $\pm$ 0.1       & 78.4 $\pm$ 0.5       & 68.7                 \\
MMD                  & 60.4 $\pm$ 0.2       & 53.3 $\pm$ 0.3       & 74.3 $\pm$ 0.1       & 77.4 $\pm$ 0.6       & 66.3                 \\
DANN                 & 59.9 $\pm$ 1.3       & 53.0 $\pm$ 0.3       & 73.6 $\pm$ 0.7       & 76.9 $\pm$ 0.5       & 65.9                 \\
CDANN                & 61.5 $\pm$ 1.4       & 50.4 $\pm$ 2.4       & 74.4 $\pm$ 0.9       & 76.6 $\pm$ 0.8       & 65.8                 \\
MTL                  & 61.5 $\pm$ 0.7       & 52.4 $\pm$ 0.6       & 74.9 $\pm$ 0.4       & 76.8 $\pm$ 0.4       & 66.4                 \\
SagNet               & 63.4 $\pm$ 0.2       & 54.8 $\pm$ 0.4       & 75.8 $\pm$ 0.4       & 78.3 $\pm$ 0.3       & 68.1                 \\
ARM                  & 58.9 $\pm$ 0.8       & 51.0 $\pm$ 0.5       & 74.1 $\pm$ 0.1       & 75.2 $\pm$ 0.3       & 64.8                 \\
VREx                 & 60.7 $\pm$ 0.9       & 53.0 $\pm$ 0.9       & 75.3 $\pm$ 0.1       & 76.6 $\pm$ 0.5       & 66.4                 \\
RSC                  & 60.7 $\pm$ 1.4       & 51.4 $\pm$ 0.3       & 74.8 $\pm$ 1.1       & 75.1 $\pm$ 1.3       & 65.5                 \\ \midrule
EQRM                 & 60.5 $\pm$ 0.1       & 56.0 $\pm$ 0.2       & 76.1 $\pm$ 0.4       & 77.4 $\pm$ 0.3       & 67.5                 \\
\bottomrule
\end{tabular}}
\end{center}

\paragraph{TerraIncognita}
\begin{center}
\adjustbox{max width=\textwidth}{%
\begin{tabular}{lccccc}
\toprule
\textbf{Algorithm}   & \textbf{L100}        & \textbf{L38}         & \textbf{L43}         & \textbf{L46}         & \textbf{Avg}         \\
\midrule
ERM                  & 49.8 $\pm$ 4.4       & 42.1 $\pm$ 1.4       & 56.9 $\pm$ 1.8       & 35.7 $\pm$ 3.9       & 46.1                 \\
IRM                  & 54.6 $\pm$ 1.3       & 39.8 $\pm$ 1.9       & 56.2 $\pm$ 1.8       & 39.6 $\pm$ 0.8       & 47.6                 \\
GroupDRO             & 41.2 $\pm$ 0.7       & 38.6 $\pm$ 2.1       & 56.7 $\pm$ 0.9       & 36.4 $\pm$ 2.1       & 43.2                 \\
Mixup                & 59.6 $\pm$ 2.0       & 42.2 $\pm$ 1.4       & 55.9 $\pm$ 0.8       & 33.9 $\pm$ 1.4       & 47.9                 \\
MLDG                 & 54.2 $\pm$ 3.0       & 44.3 $\pm$ 1.1       & 55.6 $\pm$ 0.3       & 36.9 $\pm$ 2.2       & 47.7                 \\
CORAL                & 51.6 $\pm$ 2.4       & 42.2 $\pm$ 1.0       & 57.0 $\pm$ 1.0       & 39.8 $\pm$ 2.9       & 47.6                 \\
MMD                  & 41.9 $\pm$ 3.0       & 34.8 $\pm$ 1.0       & 57.0 $\pm$ 1.9       & 35.2 $\pm$ 1.8       & 42.2                 \\
DANN                 & 51.1 $\pm$ 3.5       & 40.6 $\pm$ 0.6       & 57.4 $\pm$ 0.5       & 37.7 $\pm$ 1.8       & 46.7                 \\
CDANN                & 47.0 $\pm$ 1.9       & 41.3 $\pm$ 4.8       & 54.9 $\pm$ 1.7       & 39.8 $\pm$ 2.3       & 45.8                 \\
MTL                  & 49.3 $\pm$ 1.2       & 39.6 $\pm$ 6.3       & 55.6 $\pm$ 1.1       & 37.8 $\pm$ 0.8       & 45.6                 \\
SagNet               & 53.0 $\pm$ 2.9       & 43.0 $\pm$ 2.5       & 57.9 $\pm$ 0.6       & 40.4 $\pm$ 1.3       & 48.6                 \\
ARM                  & 49.3 $\pm$ 0.7       & 38.3 $\pm$ 2.4       & 55.8 $\pm$ 0.8       & 38.7 $\pm$ 1.3       & 45.5                 \\
VREx                 & 48.2 $\pm$ 4.3       & 41.7 $\pm$ 1.3       & 56.8 $\pm$ 0.8       & 38.7 $\pm$ 3.1       & 46.4                 \\
RSC                  & 50.2 $\pm$ 2.2       & 39.2 $\pm$ 1.4       & 56.3 $\pm$ 1.4       & 40.8 $\pm$ 0.6       & 46.6                 \\ \midrule
EQRM                 & 47.9 $\pm$ 1.9       & 45.2 $\pm$ 0.3       & 59.1 $\pm$ 0.3       & 38.8 $\pm$ 0.6       & 47.8                 \\
\bottomrule
\end{tabular}}
\end{center}

\paragraph{DomainNet}
\begin{center}
\adjustbox{max width=\textwidth}{%
\begin{tabular}{lccccccc}
\toprule
\textbf{Algorithm}   & \textbf{clip}        & \textbf{info}        & \textbf{paint}       & \textbf{quick}       & \textbf{real}        & \textbf{sketch}      & \textbf{Avg}         \\
\midrule
ERM                  & 58.1 $\pm$ 0.3       & 18.8 $\pm$ 0.3       & 46.7 $\pm$ 0.3       & 12.2 $\pm$ 0.4       & 59.6 $\pm$ 0.1       & 49.8 $\pm$ 0.4       & 40.9                 \\
IRM                  & 48.5 $\pm$ 2.8       & 15.0 $\pm$ 1.5       & 38.3 $\pm$ 4.3       & 10.9 $\pm$ 0.5       & 48.2 $\pm$ 5.2       & 42.3 $\pm$ 3.1       & 33.9                 \\
GroupDRO             & 47.2 $\pm$ 0.5       & 17.5 $\pm$ 0.4       & 33.8 $\pm$ 0.5       & 9.3 $\pm$ 0.3        & 51.6 $\pm$ 0.4       & 40.1 $\pm$ 0.6       & 33.3                 \\
Mixup                & 55.7 $\pm$ 0.3       & 18.5 $\pm$ 0.5       & 44.3 $\pm$ 0.5       & 12.5 $\pm$ 0.4       & 55.8 $\pm$ 0.3       & 48.2 $\pm$ 0.5       & 39.2                 \\
MLDG                 & 59.1 $\pm$ 0.2       & 19.1 $\pm$ 0.3       & 45.8 $\pm$ 0.7       & 13.4 $\pm$ 0.3       & 59.6 $\pm$ 0.2       & 50.2 $\pm$ 0.4       & 41.2                 \\
CORAL                & 59.2 $\pm$ 0.1       & 19.7 $\pm$ 0.2       & 46.6 $\pm$ 0.3       & 13.4 $\pm$ 0.4       & 59.8 $\pm$ 0.2       & 50.1 $\pm$ 0.6       & 41.5                 \\
MMD                  & 32.1 $\pm$ 13.3      & 11.0 $\pm$ 4.6       & 26.8 $\pm$ 11.3      & 8.7 $\pm$ 2.1        & 32.7 $\pm$ 13.8      & 28.9 $\pm$ 11.9      & 23.4                 \\
DANN                 & 53.1 $\pm$ 0.2       & 18.3 $\pm$ 0.1       & 44.2 $\pm$ 0.7       & 11.8 $\pm$ 0.1       & 55.5 $\pm$ 0.4       & 46.8 $\pm$ 0.6       & 38.3                 \\
CDANN                & 54.6 $\pm$ 0.4       & 17.3 $\pm$ 0.1       & 43.7 $\pm$ 0.9       & 12.1 $\pm$ 0.7       & 56.2 $\pm$ 0.4       & 45.9 $\pm$ 0.5       & 38.3                 \\
MTL                  & 57.9 $\pm$ 0.5       & 18.5 $\pm$ 0.4       & 46.0 $\pm$ 0.1       & 12.5 $\pm$ 0.1       & 59.5 $\pm$ 0.3       & 49.2 $\pm$ 0.1       & 40.6                 \\
SagNet               & 57.7 $\pm$ 0.3       & 19.0 $\pm$ 0.2       & 45.3 $\pm$ 0.3       & 12.7 $\pm$ 0.5       & 58.1 $\pm$ 0.5       & 48.8 $\pm$ 0.2       & 40.3                 \\
ARM                  & 49.7 $\pm$ 0.3       & 16.3 $\pm$ 0.5       & 40.9 $\pm$ 1.1       & 9.4 $\pm$ 0.1        & 53.4 $\pm$ 0.4       & 43.5 $\pm$ 0.4       & 35.5                 \\
VREx                 & 47.3 $\pm$ 3.5       & 16.0 $\pm$ 1.5       & 35.8 $\pm$ 4.6       & 10.9 $\pm$ 0.3       & 49.6 $\pm$ 4.9       & 42.0 $\pm$ 3.0       & 33.6                 \\
RSC                  & 55.0 $\pm$ 1.2       & 18.3 $\pm$ 0.5       & 44.4 $\pm$ 0.6       & 12.2 $\pm$ 0.2       & 55.7 $\pm$ 0.7       & 47.8 $\pm$ 0.9       & 38.9                 \\ \midrule
EQRM                 & 56.1 $\pm$ 1.3       & 19.6 $\pm$ 0.1       & 46.3 $\pm$ 1.5       & 12.9 $\pm$ 0.3       & 61.1 $\pm$ 0.0       & 50.3 $\pm$ 0.1       & 41.0                 \\
\bottomrule
\end{tabular}}
\end{center}

\paragraph{Averages}
    \begin{center}
        \adjustbox{max width=\linewidth}{%
        \begin{tabular}{lcccccc}
        \toprule
        \textbf{Algorithm}        & \textbf{VLCS}             & \textbf{PACS}             & \textbf{OfficeHome}       & \textbf{TerraIncognita}   & \textbf{DomainNet}        & \textbf{Avg}              \\
        \midrule
        ERM                      & 77.5 $\pm$ 0.4            & 85.5 $\pm$ 0.2            & 66.5 $\pm$ 0.3            & 46.1 $\pm$ 1.8            & 40.9 $\pm$ 0.1            & 63.3                      \\
        IRM                       & 78.5 $\pm$ 0.5            & 83.5 $\pm$ 0.8            & 64.3 $\pm$ 2.2            & 47.6 $\pm$ 0.8            & 33.9 $\pm$ 2.8            & 61.6                      \\
        GroupDRO                 & 76.7 $\pm$ 0.6            & 84.4 $\pm$ 0.8            & 66.0 $\pm$ 0.7            & 43.2 $\pm$ 1.1            & 33.3 $\pm$ 0.2            & 60.9                      \\
        Mixup                    & 77.4 $\pm$ 0.6            & 84.6 $\pm$ 0.6            & 68.1 $\pm$ 0.3            & 47.9 $\pm$ 0.8            & 39.2 $\pm$ 0.1            & 63.4                      \\
        MLDG                     & 77.2 $\pm$ 0.4            & 84.9 $\pm$ 1.0            & 66.8 $\pm$ 0.6            & 47.7 $\pm$ 0.9            & 41.2 $\pm$ 0.1            & 63.6                      \\
        CORAL                     & 78.8 $\pm$ 0.6            & 86.2 $\pm$ 0.3            & 68.7 $\pm$ 0.3            & 47.6 $\pm$ 1.0            & 41.5 $\pm$ 0.1            & 64.6                      \\
        MMD                       & 77.5 $\pm$ 0.9            & 84.6 $\pm$ 0.5            & 66.3 $\pm$ 0.1            & 42.2 $\pm$ 1.6            & 23.4 $\pm$ 9.5            & 63.3                      \\
        DANN                      & 78.6 $\pm$ 0.4            & 83.6 $\pm$ 0.4            & 65.9 $\pm$ 0.6            & 46.7 $\pm$ 0.5            & 38.3 $\pm$ 0.1            & 62.6                      \\
        CDANN                     & 77.5 $\pm$ 0.1            & 82.6 $\pm$ 0.9            & 65.8 $\pm$ 1.3            & 45.8 $\pm$ 1.6            & 38.3 $\pm$ 0.3            & 62.0                     \\
        MTL                       & 77.2 $\pm$ 0.4            & 84.6 $\pm$ 0.5            & 66.4 $\pm$ 0.5            & 45.6 $\pm$ 1.2            & 40.6 $\pm$ 0.1            & 62.9                      \\
        SagNet                    & 77.8 $\pm$ 0.5            & 86.3 $\pm$ 0.2            & 68.1 $\pm$ 0.1            & 48.6 $\pm$ 1.0            & 40.3 $\pm$ 0.1            & 64.2                      \\
        ARM                       & 77.6 $\pm$ 0.3            & 85.1 $\pm$ 0.4            & 64.8 $\pm$ 0.3            & 45.5 $\pm$ 0.3            & 35.5 $\pm$ 0.2            & 61.7                      \\
        VREx                      & 78.3 $\pm$ 0.2            & 84.9 $\pm$ 0.6            & 66.4 $\pm$ 0.6            & 46.4 $\pm$ 0.6            & 33.6 $\pm$ 2.9            & 61.9                      \\ \midrule
        EQRM                      & 77.8 $\pm$ 0.6            & 86.5 $\pm$ 0.2            & 67.5 $\pm$ 0.1            & 47.8 $\pm$ 0.6            & 41.0 $\pm$ 0.3            & 64.1                      \\
        \bottomrule
        \end{tabular}}
    \end{center}

\subsubsection{Model selection: test-domain validation set (oracle)}\label{sec:additional_exps:domainbed:test-dom}

\paragraph{VLCS}
\begin{center}
\adjustbox{max width=0.75\textwidth}{%
\begin{tabular}{lccccc}
\toprule
\textbf{Algorithm}   & \textbf{C}           & \textbf{L}           & \textbf{S}           & \textbf{V}           & \textbf{Avg}         \\
\midrule
ERM                  & 97.6 $\pm$ 0.3       & 67.9 $\pm$ 0.7       & 70.9 $\pm$ 0.2       & 74.0 $\pm$ 0.6       & 77.6                 \\
IRM                  & 97.3 $\pm$ 0.2       & 66.7 $\pm$ 0.1       & 71.0 $\pm$ 2.3       & 72.8 $\pm$ 0.4       & 76.9                 \\
GroupDRO             & 97.7 $\pm$ 0.2       & 65.9 $\pm$ 0.2       & 72.8 $\pm$ 0.8       & 73.4 $\pm$ 1.3       & 77.4                 \\
Mixup                & 97.8 $\pm$ 0.4       & 67.2 $\pm$ 0.4       & 71.5 $\pm$ 0.2       & 75.7 $\pm$ 0.6       & 78.1                 \\
MLDG                 & 97.1 $\pm$ 0.5       & 66.6 $\pm$ 0.5       & 71.5 $\pm$ 0.1       & 75.0 $\pm$ 0.9       & 77.5                 \\
CORAL                & 97.3 $\pm$ 0.2       & 67.5 $\pm$ 0.6       & 71.6 $\pm$ 0.6       & 74.5 $\pm$ 0.0       & 77.7                 \\
MMD                  & 98.8 $\pm$ 0.0       & 66.4 $\pm$ 0.4       & 70.8 $\pm$ 0.5       & 75.6 $\pm$ 0.4       & 77.9                 \\
DANN                 & 99.0 $\pm$ 0.2       & 66.3 $\pm$ 1.2       & 73.4 $\pm$ 1.4       & 80.1 $\pm$ 0.5       & 79.7                 \\
CDANN                & 98.2 $\pm$ 0.1       & 68.8 $\pm$ 0.5       & 74.3 $\pm$ 0.6       & 78.1 $\pm$ 0.5       & 79.9                 \\
MTL                  & 97.9 $\pm$ 0.7       & 66.1 $\pm$ 0.7       & 72.0 $\pm$ 0.4       & 74.9 $\pm$ 1.1       & 77.7                 \\
SagNet               & 97.4 $\pm$ 0.3       & 66.4 $\pm$ 0.4       & 71.6 $\pm$ 0.1       & 75.0 $\pm$ 0.8       & 77.6                 \\
ARM                  & 97.6 $\pm$ 0.6       & 66.5 $\pm$ 0.3       & 72.7 $\pm$ 0.6       & 74.4 $\pm$ 0.7       & 77.8                 \\
VREx                 & 98.4 $\pm$ 0.2       & 66.4 $\pm$ 0.7       & 72.8 $\pm$ 0.1       & 75.0 $\pm$ 1.4       & 78.1                 \\
RSC                  & 98.0 $\pm$ 0.4       & 67.2 $\pm$ 0.3       & 70.3 $\pm$ 1.3       & 75.6 $\pm$ 0.4       & 77.8                 \\ \midrule
EQRM                 & 98.2 $\pm$ 0.2       & 66.8 $\pm$ 0.8       & 71.7 $\pm$ 1.0       & 74.6 $\pm$ 0.3       & 77.8                 \\
\bottomrule
\end{tabular}}
\end{center}

\paragraph{PACS}
\begin{center}
\adjustbox{max width=0.75\textwidth}{%
\begin{tabular}{lccccc}
\toprule
\textbf{Algorithm}   & \textbf{A}           & \textbf{C}           & \textbf{P}           & \textbf{S}           & \textbf{Avg}         \\
\midrule
ERM                  & 86.5 $\pm$ 1.0       & 81.3 $\pm$ 0.6       & 96.2 $\pm$ 0.3       & 82.7 $\pm$ 1.1       & 86.7                 \\
IRM                  & 84.2 $\pm$ 0.9       & 79.7 $\pm$ 1.5       & 95.9 $\pm$ 0.4       & 78.3 $\pm$ 2.1       & 84.5                 \\
GroupDRO             & 87.5 $\pm$ 0.5       & 82.9 $\pm$ 0.6       & 97.1 $\pm$ 0.3       & 81.1 $\pm$ 1.2       & 87.1                 \\
Mixup                & 87.5 $\pm$ 0.4       & 81.6 $\pm$ 0.7       & 97.4 $\pm$ 0.2       & 80.8 $\pm$ 0.9       & 86.8                 \\
MLDG                 & 87.0 $\pm$ 1.2       & 82.5 $\pm$ 0.9       & 96.7 $\pm$ 0.3       & 81.2 $\pm$ 0.6       & 86.8                 \\
CORAL                & 86.6 $\pm$ 0.8       & 81.8 $\pm$ 0.9       & 97.1 $\pm$ 0.5       & 82.7 $\pm$ 0.6       & 87.1                 \\
MMD                  & 88.1 $\pm$ 0.8       & 82.6 $\pm$ 0.7       & 97.1 $\pm$ 0.5       & 81.2 $\pm$ 1.2       & 87.2                 \\
DANN                 & 87.0 $\pm$ 0.4       & 80.3 $\pm$ 0.6       & 96.8 $\pm$ 0.3       & 76.9 $\pm$ 1.1       & 85.2                 \\
CDANN                & 87.7 $\pm$ 0.6       & 80.7 $\pm$ 1.2       & 97.3 $\pm$ 0.4       & 77.6 $\pm$ 1.5       & 85.8                 \\
MTL                  & 87.0 $\pm$ 0.2       & 82.7 $\pm$ 0.8       & 96.5 $\pm$ 0.7       & 80.5 $\pm$ 0.8       & 86.7                 \\
SagNet               & 87.4 $\pm$ 0.5       & 81.2 $\pm$ 1.2       & 96.3 $\pm$ 0.8       & 80.7 $\pm$ 1.1       & 86.4                 \\
ARM                  & 85.0 $\pm$ 1.2       & 81.4 $\pm$ 0.2       & 95.9 $\pm$ 0.3       & 80.9 $\pm$ 0.5       & 85.8                 \\
VREx                 & 87.8 $\pm$ 1.2       & 81.8 $\pm$ 0.7       & 97.4 $\pm$ 0.2       & 82.1 $\pm$ 0.7       & 87.2                 \\
RSC                  & 86.0 $\pm$ 0.7       & 81.8 $\pm$ 0.9       & 96.8 $\pm$ 0.7       & 80.4 $\pm$ 0.5       & 86.2                 \\ \midrule
EQRM                 & 88.3 $\pm$ 0.6       & 82.1 $\pm$ 0.5       & 97.2 $\pm$ 0.4       & 81.6 $\pm$ 0.5       & 87.3                 \\
\bottomrule
\end{tabular}}
\end{center}

\paragraph{OfficeHome}
\begin{center}
\adjustbox{max width=0.75\textwidth}{%
\begin{tabular}{lccccc}
\toprule
\textbf{Algorithm}   & \textbf{A}           & \textbf{C}           & \textbf{P}           & \textbf{R}           & \textbf{Avg}         \\
\midrule
ERM                  & 61.7 $\pm$ 0.7       & 53.4 $\pm$ 0.3       & 74.1 $\pm$ 0.4       & 76.2 $\pm$ 0.6       & 66.4                 \\
IRM                  & 56.4 $\pm$ 3.2       & 51.2 $\pm$ 2.3       & 71.7 $\pm$ 2.7       & 72.7 $\pm$ 2.7       & 63.0                 \\
GroupDRO             & 60.5 $\pm$ 1.6       & 53.1 $\pm$ 0.3       & 75.5 $\pm$ 0.3       & 75.9 $\pm$ 0.7       & 66.2                 \\
Mixup                & 63.5 $\pm$ 0.2       & 54.6 $\pm$ 0.4       & 76.0 $\pm$ 0.3       & 78.0 $\pm$ 0.7       & 68.0                 \\
MLDG                 & 60.5 $\pm$ 0.7       & 54.2 $\pm$ 0.5       & 75.0 $\pm$ 0.2       & 76.7 $\pm$ 0.5       & 66.6                 \\
CORAL                & 64.8 $\pm$ 0.8       & 54.1 $\pm$ 0.9       & 76.5 $\pm$ 0.4       & 78.2 $\pm$ 0.4       & 68.4                 \\
MMD                  & 60.4 $\pm$ 1.0       & 53.4 $\pm$ 0.5       & 74.9 $\pm$ 0.1       & 76.1 $\pm$ 0.7       & 66.2                 \\
DANN                 & 60.6 $\pm$ 1.4       & 51.8 $\pm$ 0.7       & 73.4 $\pm$ 0.5       & 75.5 $\pm$ 0.9       & 65.3                 \\
CDANN                & 57.9 $\pm$ 0.2       & 52.1 $\pm$ 1.2       & 74.9 $\pm$ 0.7       & 76.2 $\pm$ 0.2       & 65.3                 \\
MTL                  & 60.7 $\pm$ 0.8       & 53.5 $\pm$ 1.3       & 75.2 $\pm$ 0.6       & 76.6 $\pm$ 0.6       & 66.5                 \\
SagNet               & 62.7 $\pm$ 0.5       & 53.6 $\pm$ 0.5       & 76.0 $\pm$ 0.3       & 77.8 $\pm$ 0.1       & 67.5                 \\
ARM                  & 58.8 $\pm$ 0.5       & 51.8 $\pm$ 0.7       & 74.0 $\pm$ 0.1       & 74.4 $\pm$ 0.2       & 64.8                 \\
VREx                 & 59.6 $\pm$ 1.0       & 53.3 $\pm$ 0.3       & 73.2 $\pm$ 0.5       & 76.6 $\pm$ 0.4       & 65.7                 \\
RSC                  & 61.7 $\pm$ 0.8       & 53.0 $\pm$ 0.9       & 74.8 $\pm$ 0.8       & 76.3 $\pm$ 0.5       & 66.5                 \\ \midrule
EQRM                 & 60.0 $\pm$ 0.8       & 54.4 $\pm$ 0.7       & 76.5 $\pm$ 0.4       & 77.2 $\pm$ 0.5       & 67.0                 \\
\bottomrule
\end{tabular}}
\end{center}

\paragraph{TerraIncognita}
\begin{center}
\adjustbox{max width=\textwidth}{%
\begin{tabular}{lccccc}
\toprule
\textbf{Algorithm}   & \textbf{L100}        & \textbf{L38}         & \textbf{L43}         & \textbf{L46}         & \textbf{Avg}         \\
\midrule
ERM                  & 59.4 $\pm$ 0.9       & 49.3 $\pm$ 0.6       & 60.1 $\pm$ 1.1       & 43.2 $\pm$ 0.5       & 53.0                 \\
IRM                  & 56.5 $\pm$ 2.5       & 49.8 $\pm$ 1.5       & 57.1 $\pm$ 2.2       & 38.6 $\pm$ 1.0       & 50.5                 \\
GroupDRO             & 60.4 $\pm$ 1.5       & 48.3 $\pm$ 0.4       & 58.6 $\pm$ 0.8       & 42.2 $\pm$ 0.8       & 52.4                 \\
Mixup                & 67.6 $\pm$ 1.8       & 51.0 $\pm$ 1.3       & 59.0 $\pm$ 0.0       & 40.0 $\pm$ 1.1       & 54.4                 \\
MLDG                 & 59.2 $\pm$ 0.1       & 49.0 $\pm$ 0.9       & 58.4 $\pm$ 0.9       & 41.4 $\pm$ 1.0       & 52.0                 \\
CORAL                & 60.4 $\pm$ 0.9       & 47.2 $\pm$ 0.5       & 59.3 $\pm$ 0.4       & 44.4 $\pm$ 0.4       & 52.8                 \\
MMD                  & 60.6 $\pm$ 1.1       & 45.9 $\pm$ 0.3       & 57.8 $\pm$ 0.5       & 43.8 $\pm$ 1.2       & 52.0                 \\
DANN                 & 55.2 $\pm$ 1.9       & 47.0 $\pm$ 0.7       & 57.2 $\pm$ 0.9       & 42.9 $\pm$ 0.9       & 50.6                 \\
CDANN                & 56.3 $\pm$ 2.0       & 47.1 $\pm$ 0.9       & 57.2 $\pm$ 1.1       & 42.4 $\pm$ 0.8       & 50.8                 \\
MTL                  & 58.4 $\pm$ 2.1       & 48.4 $\pm$ 0.8       & 58.9 $\pm$ 0.6       & 43.0 $\pm$ 1.3       & 52.2                 \\
SagNet               & 56.4 $\pm$ 1.9       & 50.5 $\pm$ 2.3       & 59.1 $\pm$ 0.5       & 44.1 $\pm$ 0.6       & 52.5                 \\
ARM                  & 60.1 $\pm$ 1.5       & 48.3 $\pm$ 1.6       & 55.3 $\pm$ 0.6       & 40.9 $\pm$ 1.1       & 51.2                 \\
VREx                 & 56.8 $\pm$ 1.7       & 46.5 $\pm$ 0.5       & 58.4 $\pm$ 0.3       & 43.8 $\pm$ 0.3       & 51.4                 \\
RSC                  & 59.9 $\pm$ 1.4       & 46.7 $\pm$ 0.4       & 57.8 $\pm$ 0.5       & 44.3 $\pm$ 0.6       & 52.1                 \\ \midrule
EQRM                 & 57.0 $\pm$ 1.5       & 49.5 $\pm$ 1.2       & 59.0 $\pm$ 0.3       & 43.4 $\pm$ 0.6       & 52.2                 \\
\bottomrule
\end{tabular}}
\end{center}

\paragraph{DomainNet}
\begin{center}
\adjustbox{max width=\textwidth}{%
\begin{tabular}{lccccccc}
\toprule
\textbf{Algorithm}   & \textbf{clip}        & \textbf{info}        & \textbf{paint}       & \textbf{quick}       & \textbf{real}        & \textbf{sketch}      & \textbf{Avg}         \\
\midrule
ERM                  & 58.6 $\pm$ 0.3       & 19.2 $\pm$ 0.2       & 47.0 $\pm$ 0.3       & 13.2 $\pm$ 0.2       & 59.9 $\pm$ 0.3       & 49.8 $\pm$ 0.4       & 41.3                 \\
IRM                  & 40.4 $\pm$ 6.6       & 12.1 $\pm$ 2.7       & 31.4 $\pm$ 5.7       & 9.8 $\pm$ 1.2        & 37.7 $\pm$ 9.0       & 36.7 $\pm$ 5.3       & 28.0                 \\
GroupDRO             & 47.2 $\pm$ 0.5       & 17.5 $\pm$ 0.4       & 34.2 $\pm$ 0.3       & 9.2 $\pm$ 0.4        & 51.9 $\pm$ 0.5       & 40.1 $\pm$ 0.6       & 33.4                 \\
Mixup                & 55.6 $\pm$ 0.1       & 18.7 $\pm$ 0.4       & 45.1 $\pm$ 0.5       & 12.8 $\pm$ 0.3       & 57.6 $\pm$ 0.5       & 48.2 $\pm$ 0.4       & 39.6                 \\
MLDG                 & 59.3 $\pm$ 0.1       & 19.6 $\pm$ 0.2       & 46.8 $\pm$ 0.2       & 13.4 $\pm$ 0.2       & 60.1 $\pm$ 0.4       & 50.4 $\pm$ 0.3       & 41.6                 \\
CORAL                & 59.2 $\pm$ 0.1       & 19.9 $\pm$ 0.2       & 47.4 $\pm$ 0.2       & 14.0 $\pm$ 0.4       & 59.8 $\pm$ 0.2       & 50.4 $\pm$ 0.4       & 41.8                 \\
MMD                  & 32.2 $\pm$ 13.3      & 11.2 $\pm$ 4.5       & 26.8 $\pm$ 11.3      & 8.8 $\pm$ 2.2        & 32.7 $\pm$ 13.8      & 29.0 $\pm$ 11.8      & 23.5                 \\
DANN                 & 53.1 $\pm$ 0.2       & 18.3 $\pm$ 0.1       & 44.2 $\pm$ 0.7       & 11.9 $\pm$ 0.1       & 55.5 $\pm$ 0.4       & 46.8 $\pm$ 0.6       & 38.3                 \\
CDANN                & 54.6 $\pm$ 0.4       & 17.3 $\pm$ 0.1       & 44.2 $\pm$ 0.7       & 12.8 $\pm$ 0.2       & 56.2 $\pm$ 0.4       & 45.9 $\pm$ 0.5       & 38.5                 \\
MTL                  & 58.0 $\pm$ 0.4       & 19.2 $\pm$ 0.2       & 46.2 $\pm$ 0.1       & 12.7 $\pm$ 0.2       & 59.9 $\pm$ 0.1       & 49.0 $\pm$ 0.0       & 40.8                 \\
SagNet               & 57.7 $\pm$ 0.3       & 19.1 $\pm$ 0.1       & 46.3 $\pm$ 0.5       & 13.5 $\pm$ 0.4       & 58.9 $\pm$ 0.4       & 49.5 $\pm$ 0.2       & 40.8                 \\
ARM                  & 49.6 $\pm$ 0.4       & 16.5 $\pm$ 0.3       & 41.5 $\pm$ 0.8       & 10.8 $\pm$ 0.1       & 53.5 $\pm$ 0.3       & 43.9 $\pm$ 0.4       & 36.0                 \\
VREx                 & 43.3 $\pm$ 4.5       & 14.1 $\pm$ 1.8       & 32.5 $\pm$ 5.0       & 9.8 $\pm$ 1.1        & 43.5 $\pm$ 5.6       & 37.7 $\pm$ 4.5       & 30.1                 \\
RSC                  & 55.0 $\pm$ 1.2       & 18.3 $\pm$ 0.5       & 44.4 $\pm$ 0.6       & 12.5 $\pm$ 0.1       & 55.7 $\pm$ 0.7       & 47.8 $\pm$ 0.9       & 38.9                 \\ \midrule
EQRM                 & 55.5 $\pm$ 1.8       & 19.6 $\pm$ 0.1       & 45.9 $\pm$ 1.9       & 12.9 $\pm$ 0.3       & 61.1 $\pm$ 0.0       & 50.3 $\pm$ 0.1       & 40.9                 \\
\bottomrule
\end{tabular}}
\end{center}

\paragraph{Averages}
\begin{center}
\adjustbox{max width=\textwidth}{%
\begin{tabular}{lcccccc}
\toprule
\textbf{Algorithm}        & \textbf{VLCS}             & \textbf{PACS}             & \textbf{OfficeHome}       & \textbf{TerraIncognita}   & \textbf{DomainNet}        & \textbf{Avg}              \\
\midrule
ERM                       & 77.6 $\pm$ 0.3            & 86.7 $\pm$ 0.3            & 66.4 $\pm$ 0.5            & 53.0 $\pm$ 0.3            & 41.3 $\pm$ 0.1            & 65.0                      \\
IRM                       & 76.9 $\pm$ 0.6            & 84.5 $\pm$ 1.1            & 63.0 $\pm$ 2.7            & 50.5 $\pm$ 0.7            & 28.0 $\pm$ 5.1            & 60.6                      \\
GroupDRO                  & 77.4 $\pm$ 0.5            & 87.1 $\pm$ 0.1            & 66.2 $\pm$ 0.6            & 52.4 $\pm$ 0.1            & 33.4 $\pm$ 0.3            & 63.3                      \\
Mixup                     & 78.1 $\pm$ 0.3            & 86.8 $\pm$ 0.3            & 68.0 $\pm$ 0.2            & 54.4 $\pm$ 0.3            & 39.6 $\pm$ 0.1            & 65.4                      \\
MLDG                      & 77.5 $\pm$ 0.1            & 86.8 $\pm$ 0.4            & 66.6 $\pm$ 0.3            & 52.0 $\pm$ 0.1            & 41.6 $\pm$ 0.1            & 64.9                      \\
CORAL                     & 77.7 $\pm$ 0.2            & 87.1 $\pm$ 0.5            & 68.4 $\pm$ 0.2            & 52.8 $\pm$ 0.2            & 41.8 $\pm$ 0.1            & 65.6                      \\
MMD                       & 77.9 $\pm$ 0.1            & 87.2 $\pm$ 0.1            & 66.2 $\pm$ 0.3            & 52.0 $\pm$ 0.4            & 23.5 $\pm$ 9.4            & 61.4                      \\
DANN                      & 79.7 $\pm$ 0.5            & 85.2 $\pm$ 0.2            & 65.3 $\pm$ 0.8            & 50.6 $\pm$ 0.4            & 38.3 $\pm$ 0.1            & 63.8                      \\
CDANN                     & 79.9 $\pm$ 0.2            & 85.8 $\pm$ 0.8            & 65.3 $\pm$ 0.5            & 50.8 $\pm$ 0.6            & 38.5 $\pm$ 0.2            & 64.1                      \\
MTL                       & 77.7 $\pm$ 0.5            & 86.7 $\pm$ 0.2            & 66.5 $\pm$ 0.4            & 52.2 $\pm$ 0.4            & 40.8 $\pm$ 0.1            & 64.8                      \\
SagNet                    & 77.6 $\pm$ 0.1            & 86.4 $\pm$ 0.4            & 67.5 $\pm$ 0.2            & 52.5 $\pm$ 0.4            & 40.8 $\pm$ 0.2            & 65.0                      \\
ARM                       & 77.8 $\pm$ 0.3            & 85.8 $\pm$ 0.2            & 64.8 $\pm$ 0.4            & 51.2 $\pm$ 0.5            & 36.0 $\pm$ 0.2            & 63.1                      \\
VREx                      & 78.1 $\pm$ 0.2            & 87.2 $\pm$ 0.6            & 65.7 $\pm$ 0.3            & 51.4 $\pm$ 0.5            & 30.1 $\pm$ 3.7            & 62.5                      \\
RSC                       & 77.8 $\pm$ 0.6            & 86.2 $\pm$ 0.5            & 66.5 $\pm$ 0.6            & 52.1 $\pm$ 0.2            & 38.9 $\pm$ 0.6            & 64.3                      \\ \midrule
EQRM                      & 77.8 $\pm$ 0.2            & 87.3 $\pm$ 0.2            & 67.0 $\pm$ 0.4            & 52.2 $\pm$ 0.7            & 40.9 $\pm$ 0.3            & 65.1                      \\
\bottomrule
\end{tabular}}
\end{center}

\subsection{WILDS}%
\label{sec:additional_exps:wilds}%

In Figure~\ref{fig:methods-and-baselines}, we visualize the test-time risk distributions of IRM and GroupDRO relative to ERM, as well as EQRM$_\alpha$ for select values\footnote{We display results for fewer values of $\alpha$ in Figure~\ref{fig:methods-and-baselines} to keep the plots uncluttered.} of $\alpha$.  In each of these figures, we see that IRM and GroupDRO tend to have heavier tails than any of the other algorithms.

\begin{figure}[h]
    \centering
    \includegraphics[width=0.6\textwidth]{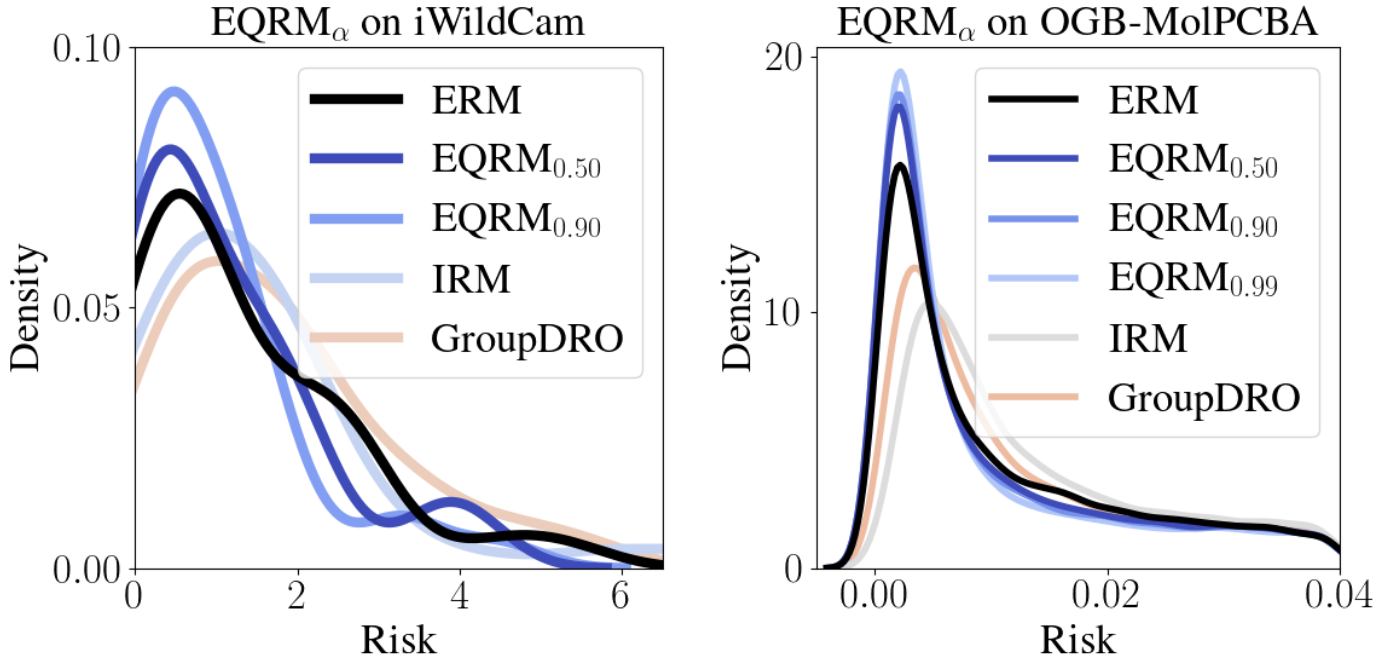}
    \caption{\small \textbf{Baseline test risk distributions on iWildCam and OGB-MolPCBA.}  We supplement Figure~\ref{fig:real-world-pdfs} by providing comparisons to two baseline algorithms: IRM and GroupDRO.  In each case, EQRM$_\alpha$ tends to display superior tail performance relative to ERM, IRM, and GroupDRO.} 
    \label{fig:methods-and-baselines}
\end{figure}

\paragraph{Other performance metrics.}  In the main text, we studied the tails of the \textit{risk} distributions of predictors trained on \texttt{iWildCam} and \texttt{OGB}. However, in the broader DG literature, there are a number of other metrics that are used to assess performance or OOD-generalization. In particular, for \texttt{iWildCam}, past work has used the macro $F_1$ score as well as the average accuracy across domains to assess OOD generalization; for \texttt{OGB}, the standard metric is a predictor's average precision over test domains~\cite{wilds2021}. In Tables~\ref{tab:accuracies-iwildcam} and~\ref{tab:accuracies-ogb}, we report these metrics and compare the performance of our algorithms to ERM, IRM, and GroupDRO. Below, we discuss the results in each of these tables.

To begin, consider Table~\ref{tab:accuracies-iwildcam}.  Observe that ERM
achieves the best \emph{in-distribution} (ID) scores relative to any of the other algorithms.  However, when we consider the \emph{out-of-distribution} columns, we see that EQRM offers better performance with respect to both the macro $F_1$ score and the mean accuracy.  Thus, although our algorithms are not explicitly trained to optimize these metrics, their strong performance on the tails of the risk distribution appears to be correlated with strong OOD performance with these alternative metrics.  We also observe that relative to ERM, EQRM suffers smaller accuracy drops between ID and OOD mean accuracy. Specifically, ERM dropped 5.50 points, whereas EQRM dropped by an average of 2.38 points.

Next, consider Table~\ref{tab:accuracies-ogb}.  Observe again that ERM is the strongest-performing \textit{baseline} (first band of the table).  Also observe that EQRM performs similarly to ERM, with validation and test precision tending to cluster around 28 and 27 respectively. However, we stress that these metrics are \emph{averaged} over their respective domains, whereas in Tables~\ref{tab:quantiles-iwildcam} and \ref{tab:quantiles-ogb}, we showed that EQRM performed well on the more difficult domains, i.e.\ when using \emph{tail} metrics.
 
\begin{table}[ht]
\begin{minipage}{0.48\textwidth}
    \centering
    \caption{WILDS metrics on \texttt{iWildCam}.}
    \label{tab:accuracies-iwildcam}
    \resizebox{\columnwidth}{!}{
    \begin{tabular}{ccccc} \toprule
         \multirow{2}{*}{Algorithm} & \multicolumn{2}{c}{Macro $F_1$ $(\uparrow)$} & \multicolumn{2}{c}{Mean accuracy $(\uparrow)$} \\ \cmidrule(lr){2-3} \cmidrule(lr){4-5}
         & ID & OOD & ID & OOD \\ \midrule
         ERM & \textbf{49.8} & 30.6 & \textbf{77.0} & 71.5\\
         IRM & 23.4 & 15.2 & 59.6 & 64.1 \\
         GroupDRO & 34.3 & 22.1 & 66.7 & 67.7 \\ \midrule
         QRM$_{0.25}$ & 18.3 & 11.4 & 54.3 & 58.3 \\
         QRM$_{0.50}$ & 48.1 & 33.8 & 76.2 & 73.5 \\
         QRM$_{0.75}$ & 49.5 & 31.8 & 76.1 & 72.0 \\
         QRM$_{0.90}$ & 48.6 & 32.9 & 77.1 & 73.3 \\
         QRM$_{0.99}$ & 45.9 & 30.8 & 76.6 & 71.3 \\ 
         \bottomrule
    \end{tabular}}

\end{minipage}
\quad
\begin{minipage}{0.48\textwidth}
    \centering
    \caption{WILDS metrics on \texttt{OGB-MolPCBA}.}
    \label{tab:accuracies-ogb}
    \resizebox{0.7\columnwidth}{!}{
    \begin{tabular}{ccc} \toprule
         \multirow{2}{*}{Algorithm} & \multicolumn{2}{c}{Mean precision $(\uparrow)$}  \\ \cmidrule(lr){2-3}
         & Validation & Test \\ \midrule
         ERM & 28.1 & 27.3 \\
         IRM & 15.4 & 15.5 \\
         GroupDRO & 23.5 & 22.3 \\ \midrule
         QRM$_{0.25}$ & 28.1 & 27.3 \\
         QRM$_{0.50}$ & \textbf{28.3} & \textbf{27.4} \\
         QRM$_{0.75}$ & 28.1 & 27.1 \\
         QRM$_{0.90}$ & 27.9 & 27.2 \\
         QRM$_{0.99}$ & 28.1 & 27.4 \\ %
        \bottomrule
    \end{tabular}
    }

\end{minipage}
\end{table}

\newpage
\section{Limitations of our work}\label{app:limitations}
As discussed in the first paragraph of \cref{sec:discussion}, the main limitation of our work is that, for $\alpha$ to \textit{precisely} approximate the probability of generalizing with risk below the associated $\alpha$-quantile value, we must have a large number of i.i.d.-sampled domains. Currently, this is rarely satisfied in practice, although \cref{sec:discussion} describes how new data-collection procedures could help to better-satisfy this assumption. We believe that our work, and its promise of machine learning systems that generalize with high probability, provides sufficient motivation for collecting real-world datasets with a large number of i.i.d.-sampled domains. In addition, we hope that future work can explore ways to relax this assumption, e.g., by leveraging knowledge of domain dependencies like time.
\end{document}